\documentclass[twoside,11pt]{article}

\usepackage{jmlr2e}

\usepackage[utf8]{inputenc} 
\usepackage[T1]{fontenc}    
\usepackage{hyperref}       
\usepackage{url}            
\usepackage{booktabs}       
\usepackage{amsfonts}       
\usepackage{nicefrac}       
\usepackage{microtype}      

\usepackage{graphicx}
\usepackage{booktabs} 

\usepackage[ruled,vlined, noend]{algorithm2e}
\usepackage{algorithmic}

\usepackage{xcolor}
\usepackage{amssymb}
\usepackage{bbm}
\usepackage{bm}

\usepackage{amsmath}
\usepackage{amsthm}

\DeclareMathOperator*{\argmin}{arg\,min}

\usepackage{xcolor}
\usepackage{caption}
\usepackage{subcaption}
\usepackage{float}
\usepackage{wrapfig,lipsum,booktabs}
\usepackage{mathrsfs}
\usepackage{enumitem}

\newtheorem{lemma}{Lemma}
\newtheorem{theorem}{Theorem}

\newtheorem{innercustomgeneric}{\customgenericname}
\providecommand{\customgenericname}{}
\newcommand{\newcustomtheorem}[2]{%
  \newenvironment{#1}[1]
  {%
   \renewcommand\customgenericname{#2}%
   \renewcommand\theinnercustomgeneric{##1}%
   \innercustomgeneric
  }
  {\endinnercustomgeneric}
}

\newcustomtheorem{customthm}{Theorem}
\newcustomtheorem{customlemma}{Lemma}

\newcommand{\Bspace}{\ensuremath{\mathscr{B}_{HS}}}
\newcommand{\BKspace}{\ensuremath{B_K}}

\ShortHeadings{Functional optimal transport}{Zhu, Guha, Do, Xu, Nguyen and Zhao}
\firstpageno{1}
\begin{document}

\title{Functional optimal transport: map estimation and domain adaptation for functional data}

\author{\name Jiacheng Zhu$^*$ \email jzhu4@andrew.cmu.edu \\
\addr Department of Mechanical Engineering, Carnegie Mellon University\\
Pittsburgh, PA 15213, USA
\AND
\name Aritra Guha$^*$ \email ag997x@att.com  \\
\addr Data Science \& AI Research, AT\&T Chief Data Office\\
Bedminster, NJ 07921, USA
\AND
\name Dat Do$^*$ \email dodat@umich.edu \\
\addr Department of Statistics, University of Michigan\\
Ann Arbor, MI 48105, USA
\AND
\name Mengdi Xu \email mengdixu@andrew.cmu.edu \\
\addr Department of Mechanical Engineering, Carnegie Mellon University\\
Pittsburgh, PA 15213, USA
\AND
\name XuanLong Nguyen \email xuanlong@umich.edu \\
\addr Department of Statistics, University of Michigan\\
Ann Arbor, MI 48105, USA
\AND
\name Ding Zhao \email dingzhao@cmu.edu \\
\addr Department of Mechanical Engineering, Carnegie Mellon University\\
Pittsburgh, PA 15213, USA
}

\editor{ } 
 
\maketitle

\begin{abstract}
We introduce a formulation of optimal transport problem for distributions on function spaces, where the stochastic map between functional domains can be partially represented in terms of an (infinite-dimensional) Hilbert-Schmidt operator mapping a Hilbert space of functions to another. 
For numerous machine learning tasks, data can be naturally viewed as samples drawn from spaces of functions, such as curves and surfaces, in high dimensions. 
Optimal transport for functional data analysis provides a useful framework of treatment for such domains. 
{ Since probability measures in infinite dimensional spaces generally lack absolute continuity (that is, with respect to non-degenerate Gaussian measures), the Monge map in the standard optimal transport theory for finite dimensional spaces may not exist. Our approach to the optimal transport problem in infinite dimensions is by a suitable regularization technique --- we restrict the class of transport maps to be a Hilbert-Schmidt space of operators.}
To this end, we develop an efficient algorithm for finding the stochastic transport map between functional domains and provide theoretical guarantees on the existence, uniqueness, and consistency of our estimate for the Hilbert-Schmidt operator. 
We validate our method on synthetic datasets and examine the functional properties of the transport map. Experiments on real-world datasets of robot arm trajectories further demonstrate the effectiveness of our method on applications in domain adaptation. 
\end{abstract}

\begin{keywords}
  Optimal transport, Optimal transport map estimation, Functional data analysis, Hilbert Schmidt operator, Domain adaptation
\end{keywords}

\section{Introduction}

Optimal transport (OT) is a formalism for finding and quantifying the movement of mass from one probability distribution to another \citep{villani2008optimal}. In recent years, it has been instrumental in the development of various new machine learning methods, including deep generative modeling \citep{arjovsky2017wgan,salimans2018imp_GANOT}, unsupervised learning \citep{ho2017multilevel,mallasto2017learning_wGP} and domain adaptation \citep{ganin2015unsupervisedDA,bhushan2018deepjdot}. As statistical machine learning algorithms are applied to increasingly complex domains, it is of interest to develop optimal transport-based methods for complex data structures. A particularly common form of such structures arises from functional data --- data that may be viewed as random samples of smooth functions, curves, or surfaces in high dimension spaces \citep{Ferraty-Vieu-06,ramsay2005_fda,hsing2015theoretical_fda,mirshani_ICML2019_functional,dupont2021generative_dist_function}. Examples of real-world applications involving functional data are numerous, ranging from robotics \citep{deisenroth2013gaussianprocess_robotics} and natural language processing \citep{rodrigues2014gaussianprocess_nlp} to economics \citep{horvath2012inference_func_data_economics} and healthcare \citep{cheng2020sparse_multi_GP_time_series}. 
Therefore, it is of interest to extend and develop a suitable optimal transport formulation to the \emph{functional} data domains. 

The goal of this paper is to provide a novel formulation of the optimal transport problem in function spaces,
to develop an efficient learning algorithm for estimating a suitable notion of optimal stochastic map that transports samples from one functional domain to another, to provide theoretical guarantees regarding the \textit{existence}, \textit{uniqueness}, and \textit{consistency} of our estimates, and to demonstrate the effectiveness of our approach to several application domains where the functional optimal transport (FOT) viewpoint proves natural and useful. 
There are several formidable challenges: both the source and the target function spaces can be quite complex and in general of \textit{infinite dimensions}. 
One needs to deal with probability distributions over such spaces, which is difficult if one is to model them with data. Moreover, optimal coupling or optimal transport map between the two distributions on infinite dimensional spaces is generally hard to characterize and compute efficiently, except for very few specific cases. Yet, to be useful in practice, one must find an explicit transport map that can approximate the optimal coupling well, i.e., to find an approximate solution to the original Monge problem \citep{villani2008optimal}. 

{ The primary technical challenge here, more specifically, is that the optimal Monge map may not exist in general. By a Brenier-type theorem, a sufficient condition for the existence of the Monge map under a suitable convex cost function is that the source distribution be absolutely continuous with respect to non-degenerate Gaussian measures in the source domain (see \cite{ambrosio2005gradient}, Sec. 6.2). In finite dimensional (Euclidean) domains, absolute continuity with respect to non-degenerate Gaussian measures is equivalent to absolute continuity with respect to the Lebesgue measure. However, in infinite dimensional domains, probability measures tend to lack the requisite absolute continuity. In fact, distributions in infinite-dimensional space can be discrete and tend to be singular to each other \citep{kakutani1948equivalence}. Historically, the lack of existence of the Monge map was a key motivation for Kantorovich's optimal coupling formulation, which is well-posed for probability distributions in Polish spaces, and the discovery of Brenier-type theorems linking the optimal coupling to that of Monge map helped to ignite renewed interest and fresh new developments in the field of optimal transport in the past several decades. 
Since we work in a setting where the Monge map is not expected to exist in general, and moreover, a direct application of the Kantorovich's optimal coupling formulation seems difficult, we shall take a rather natural approach based on regularization: we seek to find the best deterministic map within a class of operators acting on the space of functions in the source domain. As we shall see shortly, this approach can be viewed as a regularized form of the optimal coupling problem as well.
}



We note that there is currently a growing interest, especially in the machine learning literature, in finding an explicit optimal transport map linked to the Monge problem, although such attempts were mainly confined to finite-dimensional domains. For discrete distributions, map estimation can be tackled by jointly learning the coupling and a transformation map \citep{perrot2016mapping_est_discre_OT}. This basic idea and extensions were shown to be useful for the alignment of multimodal distributions \citep{lee2019hierarchical_hiwa} and word embedding \citep{zhang2017earth_mover,grave2019unsupervised}; such joint optimization objective was shown \citep{alvarez2019towards_joint} to be related to the softassign Procrustes method \citep{rangarajan1997softassign}.
 Meanwhile, a different strand of work focused on scaling up the computation of the transport map \citep{genevay2016stoch_LSOT, meng2019large_map_est}, including approximating transport maps with neural networks \citep{seguy2017LSOT,makkuva2019ICNN}, deep generative models \citep{xie2019scalable_LSOT}, and flow models \citep{huang2020convex_ot}. 
 It is emphasized that all these methods are not quite suitable for capturing the distributions on the space of functions.  
 Recent developments on Gromov-Wasserstein distance enable the comparisons of distributions on space of different dimensions~\citep{memoli2011gromov}. 
 An alternative approach to constructing distances between probability measures on (Euclidean) spaces of different dimensions was recently proposed by \cite{cai2022_w_inequal}. 
These techniques are relevant but not immediately applicable to the functional domains, which are of infinite dimensions. In addition, a common feature of functional data analysis is that the function samples are typically observed at the different and possibly varying number of design points. A naive approach to functional data is to treat a function as a vector of components sampled at a number of design points in its domain. Such an approach fails to exploit the fine structures (e.g., continuity, regularity) present naturally in many functional domains. 
Moreover, non-functional approaches to functional data may be highly sensitive to the choice of design points as one moves from one domain to another.



{ Most known results and techniques on optimal transport between distributions on function spaces are related to Gaussian measures and Gaussian processes 
~\citep{mallasto2017learning_wGP,masarotto2019procrustes_GP,knott1984optimal,pigoli2014distances}. These results are natural generalization from those of the multivariate Gaussian distributions~\citep{dowson1982frechet,givens1984class_W_prob}. Specifically, the 2-Wasserstein distance between Gaussian processes with certain covariance coincides with the Procrustes distance between the two covariance operators (cf. Section 2 of ~\cite{masarotto2019procrustes_GP}). Furthermore, there exists a linear subspace in Hilbert spaces where the optimal map between two centered Gaussian processes is well-defined as a linear operator. In practice, the Gaussian distribution assumption is clearly too restrictive in many domains. Our work may be viewed as a first step at addressing optimal transport in the domains of functions that go beyond the Gaussian assumption, and with a particular focus on learning the explicit transport map for sampled functional data.}



In our approach the mathematical machinery of functional data analysis (FDA) 
\citep{hsing2015theoretical_fda,ramsay2005_fda}, along with recent advances in computational optimal transport via regularization techniques will be brought to bear on the aforementioned problems. There are several ingredients in our work. First, we take a probabilistic model-free approach, by avoiding making assumptions on the source and target distributions of functional data. Instead, we aim to learn the (stochastic) transport map directly. Second, we follow the FDA perspective by assuming that both the source and target distributions be supported on suitable Hilbert spaces of functions $H_1$ and $H_2$, respectively. A map $T: H_1 \rightarrow H_2$ sending elements of $H_1$ to that of $H_2$ will be represented by a class of linear operators, namely the integral operators. 
In fact, we shall restrict ourselves to Hilbert-Schmidt operators, which are compact, computationally convenient to regularize and amenable to theoretical analysis. Finally, the optimal \emph{deterministic} transport map between two probability measures on function spaces may not exist, due to the general lack of absolute continuity discussed earlier. 
To overcome this difficulty, we enlarge the space of transport maps by allowing for stochastic coupling $\Pi$ between the two domains $T(H_1) \subseteq H_2$ and $H_2$, while the complexity of such coupling can be controlled via the entropic regularization technique~\citep{Cuturi-Sinkhorn-13}. 

It is quite interesting to note that our formulation for optimal transport in the functional domains has two complementary interpretations: it can be viewed as learning an integral operator $T$ regularized by a transport plan (a coupling distribution $\Pi$) or it can also be seen as an optimal coupling problem for $\Pi$ (the Kantorovich problem), which is associated with a cost matrix parameterized by the integral operator $T$. In any case, we take a joint optimization approach for the transport map $T$ and the coupling distribution $\Pi$ in functional domains. Subject to suitable regularization on the space of transport maps, the existence of the optimal $(T,\Pi)$ and the uniqueness of $T$ can be established, which leads to a consistency result of our estimation procedure. 

{

There are several advantages in our choice of bounded linear operators for modeling the transport map. First, in functional analysis and functional data analysis in particular, bounded linear operators (including the rich class of integral operators) are the main workhorse for representing transformation among spaces of functions \citep{yosida1995functional,Ramsay-Silverman-05,hsing2015theoretical_fda}. Using this class of operators allows us to draw from the principled machinery and tools from functional analysis to establish a solid theoretical foundation for optimal transport in infinite-dimensional function spaces.
Second, compact linear operators are easily regularizable, which result in fast computational procedures for learning from the functional data.
The third advantage is the interpretability of the representation of the transport map $T$, which can be helpful in applications. In fact, linearity in the representation space is the desired evaluation protocol for even the most sophisticated large-scale pre-trained models~\citep{radford2021learning_CLIP}, which are powering applications across both natural language and image generation. 
}

{
There are several limitations in our approach. First, bounded (or compact) linear operators may be too restrictive in some domains. Learning nonlinear operators in infinite dimensional spaces is an interesting direction, but is beyond the scope of this paper. Moreover, as discussed earlier, for discrete probability measures in the source domain, even nonlinear operators are not enough since the Monge map generally does not exist anyway. When this is the case, a reasonable approach, in our opinion, is to revert to the Kantorovich's coupling formulation, where linear operators can still be efficiently utilized as a building block from a regularization viewpoint for the optimal coupling problem. Indeed, our modeling choice is sufficiently rich when coupled with the stochastic coupling to obtain the optimal $(T,\Pi)$, which ensures both the existence and uniqueness of the solution to our formulation of the optimal coupling problem. 
The second limitation is more practical, as it is related to the fact that the learned linear operator $T$ is represented in terms of its action on the eigenfunction basis of the function spaces of the source and target domains. This also highlights a key distinction for the functional optimal transport problem and our corresponding approach from the linear transformation techniques for fixed-dimensional vector spaces, which are relatively simpler and do not have to grapple with this issue (see, e.g.~\cite{perrot2016mapping_est_discre_OT,alvarez2019towards_joint}).
In many practical domains, the eigenfunction basis may not be available for certain types of datasets. In that case, we may rely on available methods to construct data-driven basis functions, such as functional principal component analysis (FPCA)~\citep{shang2014survey_fpca,liu2017functional_fpca}.
}


In summary, the contributions presented in this paper are the following.
\begin{itemize}
     \item  We propose functional optimal transport (FOT), a novel formulation of optimal transport in the infinite-dimensional functional spaces. We take a probabilistic model-free approach, by avoiding making assumptions on the source and target distributions of functional data.
    \item An approximate optimization algorithm is developed to estimate the transport map from data. We follow the FDA perspective by assuming that both the source and target distributions be supported on suitable Hilbert spaces of functions $H_1$ and $H_2$, respectively. A map $T: H_1 \rightarrow H_2$ sending elements of $H_1$ to that of $H_2$ will be represented by a class of linear operators, namely the integral operators. Despite the infinite-dimensional nature of this problem, we propose an approximate estimator and solve it with an alternative minimization algorithm. 
    \item We establish the existence and uniqueness for such a transport map on empirical samples of functional data from both source and target domains, as well as consistency theorems providing the support for our estimator. Specifically, we show the asymptotic convergence in terms of the number of basis functions $K$, the number of observed sample functions $n$, and the number of design points $d$. To the best of our knowledge, this is the first work in which a rigorous statistical theory for optimal transport in the domains of functions is established. 
    \item Simulation studies and experiments are conducted to validate our method and the associated theory. First, the convergence properties of our map estimation when the samples are synthesized with known basis functions are verified via simulations. In the task of estimating an explicit transport map for two sets of functional data, our proposed method displayed superior performance from both qualitative and quantitative perspectives in comparison to non-functional techniques.
    Next, our method is applied to several real-world datasets. We conduct the optimal transport domain adaptation for predicting robot-arm motion from two different datasets. 
    
\end{itemize}

\subsection{Organization}
The remainder of the paper is organized as follows.
Section~\ref{section:prelim_knowledge} contains some preliminary background of optimal transport and functional data analysis.
Section 3 presents the formulation of functional optimal transport based on Hilbert-Schmidt operators, followed by theoretical results on the existence, uniqueness, and consistency of our map estimator. 
In Section~\ref{section:methodology}, we describe an implementation of our estimation procedure by solving a block coordinate-wise convex optimization problem. The result is an efficient algorithm for finding explicit transport maps that can be applied on sampled functions.
Then, in Section~\ref{section:Experiments}, the effectiveness of our approach is validated first on synthetic datasets of smooth functional data and then applied in a suite of experiments for mapping real-world 3D trajectories between robotic arms with different configurations.
All proofs are given in Section~\ref{section:proof}. 
Finally, Section~\ref{section:conclusion} provides further discussions of related work, as well as several directions for future work.




\section{Preliminaries}
This section provides some basic background of optimal transport and functional data analysis.
\label{section:prelim_knowledge}
\subsection{Optimal transport}
The basic problem in optimal transport, also known as the Kantorovich problem~\citep{villani2008optimal,kantorovitch1958}, is to find an optimal coupling $\pi$ of given measures $\mu$ on space $\mathcal{X}$ and $\nu$ on space $\mathcal{Y}$ to minimize 
\begin{equation}\label{eq:Kantorovich}
    \inf_{\pi \in \Pi} \int_{ \mathcal{X} \times \mathcal{Y}} c(x, y) d \pi(x, y), \text{   subject to } \Pi=\{\pi: \gamma_\#^{\mathcal{X}} \pi = \mu, \gamma_\#^{\mathcal{Y}} \pi = \nu\}.
\end{equation}
In the above display, $c: \mathcal{X} \times \mathcal{Y} \mapsto \mathbb{R}^+$ is a cost function and  $\gamma^{\mathcal{X}}$, $\gamma^{\mathcal{Y}}$ 
denote projections from $\mathcal{X}\times \mathcal{Y}$ onto $\mathcal{X}$ and $\mathcal{Y}$ respectively, and hence the corresponding pushforward measures denoted by $\gamma_\#^{\mathcal{X}} \pi, \gamma_\#^{\mathcal{Y}} \pi$ of $\pi$ are its the marginal distributions on $\mathcal{X}$ and $\mathcal{Y}$. This optimization is well-defined and the optimal $\pi$ exists under mild conditions (in particular, $\mathcal{X},\mathcal{Y}$ are both separable and complete metric spaces, $c$ is lower semi-continuous \citep{villani2008optimal}). When $\mathcal{X}=\mathcal{Y}$ are metric spaces, $c(x,y)$ is the square of the distance between $x$ and $y$, then the square root of the optimal cost given by \eqref{eq:Kantorovich} defines the Wasserstein metric $W_2(\mu,\nu)$ on the space of square integrable probability measures on $\mathcal{X}$. A related problem is Monge problem, where one finds a Borel map $T: \mathcal{X} \to \mathcal{Y}$ that realizes the infimum
\begin{eqnarray}
\label{eq:Monge}
\inf_T  \int_{\mathcal{X}} c(x, T(x)) d \mu(x)  \text{ \smallskip {}{} subject to {}{} \smallskip  }   T_{\#}\mu = \nu.
\end{eqnarray}
Here $T_\#\pi$ generally denotes the pushforward measure of $\pi$ by $T$.

Note that the existence of the optimal deterministic map $T$ is not always guaranteed \citep{villani2008optimal}. However, in various applications, it is of interest to find a deterministic map that approximates the optimal coupling to the Kantorovich problem.
In many recent work, $T$ is typically restricted to a family of maps $\mathcal{F}$ followed by joint optimization of $T$ and $\pi$ \citep{perrot2016mapping_est_discre_OT, alvarez2019towards_joint,grave2019unsupervised,seguy2017LSOT,alvarez2020unsupervised_joint}:
\begin{eqnarray}
\label{eq:map}
\inf_{\pi \in \Pi ,T\in \mathcal{F}} \int_{ \mathcal{X} \times \mathcal{Y}} c(T(x), y) d \pi(x, y),
\end{eqnarray}
where $c: \mathcal{Y} \times \mathcal{Y} \mapsto \mathbb{R}^+$ is a cost function on $\mathcal{Y}$. On the one hand, the class of maps $\mathcal{F}$ may be chosen to be sufficiently rich to approximate the optimal transport maps for the measures $\mu$ and $\nu$ of interest defined on the respective spaces $\mathcal{X}$, $\mathcal{Y}$. On the other hand, $\mathcal{F}$ may be chosen to ease up the computational burden and facilitate meaningful interpretations. For instance, $\mathcal{F}$ may be a class of linear functions (e.g. rigid transformations) \citep{perrot2016mapping_est_discre_OT,alvarez2020unsupervised_joint} or neural networks \citep{seguy2017LSOT}. 

At a high level, our approach will be analogous to~\eqref{eq:map}, except that $\mathcal{X}$ and $\mathcal{Y}$ are taken to be spaces of functions, as we are motivated by applications in functional domains. As an illustration for a toy example, Figure \ref{fig:one} depicts a transport map that sends sample paths from the famous \texttt{Swiss-roll} curve data set to that of the target \texttt{Wave} curve data set. In a real-world application considered later in Section~\ref{section:Experiments}, $\mathcal{X}$ and $\mathcal{Y}$ represent the space of (smooth) trajectories of robot motions. 
A natural and powerful approach to such data domains is the framework of functional data analysis. In this framework, the data may be viewed as samples of random functions. In particular, we will be working with distributions on Hilbert spaces of functions, while $\mathcal{F}$ is a suitable class of operators acting on such Hilbert spaces. We proceed to a brief background of FDA in the sequel.

\subsection{Functional data analysis} 
Functional data analysis adopts the perspective that certain types of data may be viewed as samples of random functions, which are viewed as random elements taking value in Hilbert spaces of functions~\citep{hsing2015theoretical_fda}. The data analysis techniques on functional data involve operators acting on Hilbert spaces. Let $A: H_1 \to H_2$ be a bounded linear operator, where $H_1$ (respectively, $H_2$) is a Hilbert space equipped with scalar product $\langle \cdot, \cdot \rangle_{H_1}$ (respectively, $\langle \cdot, \cdot \rangle_{H_2}$) and $\{U_{i}\}_{i \geq 1} (\{V_{j}\}_{j \geq 1})$ is the Hilbert basis in $H_1$ ($H_2$). We will focus on a class of compact integral operators, namely Hilbert-Schmidt operators, that are sufficiently rich for many applications and yet amenable to analysis and computation. $A$ is said to be Hilbert-Schmidt if $\sum_{i \geq 1}\| A U_i \|_{H_2}^2<\infty$ for any Hilbert basis $\{U_i\}_{i \geq 1}$. The space of Hilbert-Schmidt operators between $H_1$ and $H_2$, to be denoted by $\mathscr{B}_{HS}(H_1, H_2)$, is also a Hilbert space endowed with the scalar product $\langle A, B \rangle_{HS} = \sum_i \langle A U_i , B U_{i} \rangle_{H_2}$ and the corresponding Hilbert-Schmidt norm is denoted by $\| \cdot \|_{HS}$.

Denote the outer product operator between two elements $e_i \in H_i$ for $i=1,2$ by $e_1\otimes e_2: H_1 \to H_2$, which is defined by $(e_1 \otimes e_2) f= \langle e_1,f\rangle_{H_1}e_2$ for $f \in H_1$.
An important fact of Hilbert-Schmidt operators is given as follows
(see, e.g., Theorem 4.4.5 of~\citep{hsing2015theoretical_fda}). 
\begin{theorem}
\label{theorem:CONS}
The linear space $\mathscr{B}_{HS}(H_1, H_2)$ is a separable Hilbert space when equipped with the HS inner product. For any choice of complete orthonormal basis system (CONS) $\{ U_{i} \}$ and $\{V_{j}\}$ for $H_1$ and $H_2$ respectively, $\{ U_{i} \otimes V_{j} \}$ forms a CONS for $\mathscr{B}_{HS}(H_1,H_2)$. 
\end{theorem} 

As a result, the following representation of Hilbert-Schmidt operators and their norm will be useful.

\begin{lemma}
\label{lemma:HS norm}
Let $\{U_{i}\}_{i=1}^{\infty},\{V_{j}\}_{j=1}^{\infty}$ be a CONS for $H_1,H_2$, respectively. Then any Hilbert-Schmidt operator $T\in \mathscr{B}_{HS}(H_1,H_2)$ can be decomposed as
\begin{eqnarray}
T= \sum_{i,j} \lambda_{ji} U_{i} \otimes V_{j}, \text{  where   } \|T\|_{HS}^2= \sum_{i,j} \lambda_{ji}^2.
\end{eqnarray}
\end{lemma}

\begin{figure*}[t!]

\subfloat[samples paths]{%
\includegraphics[height=0.12\textheight]{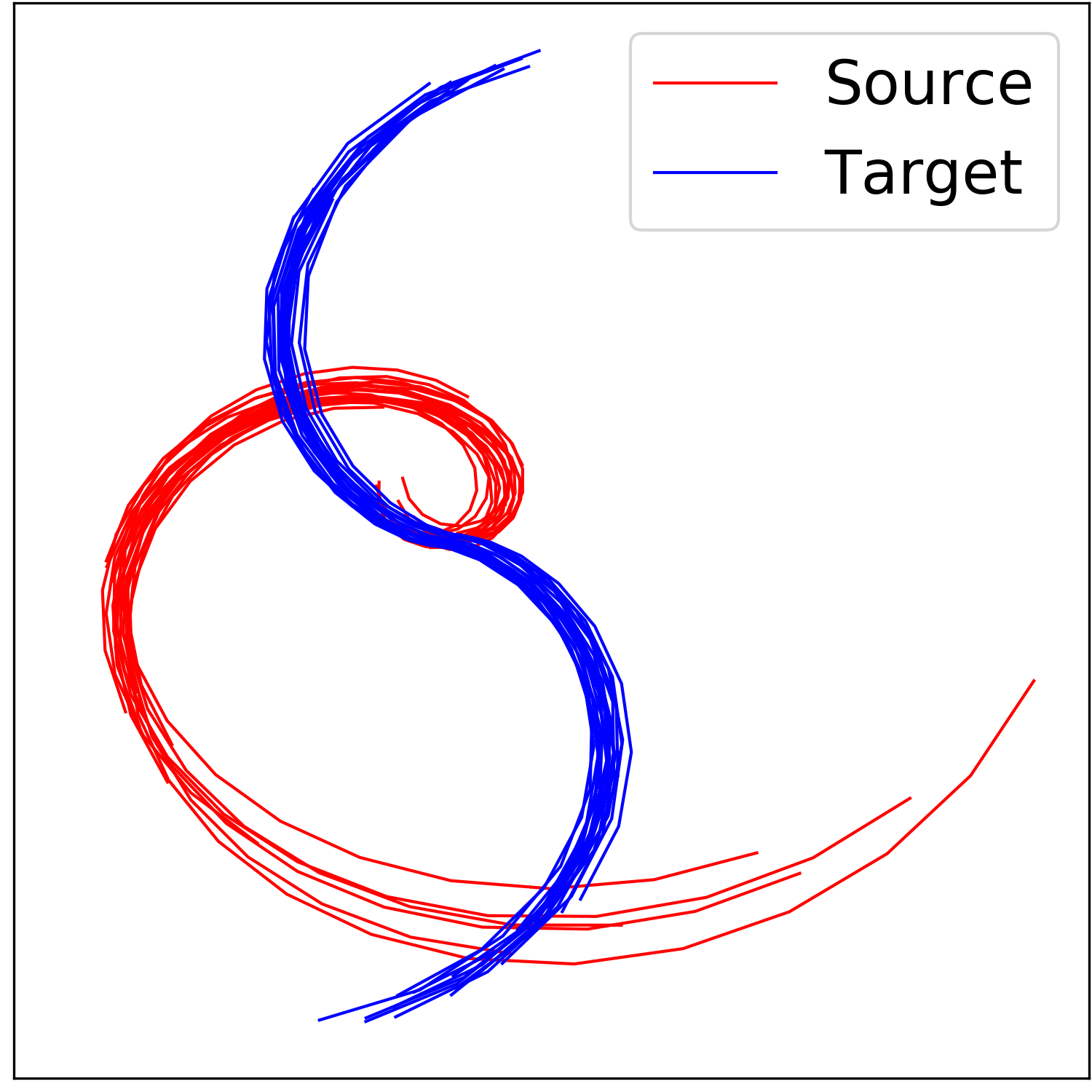}}
\hspace*{\fill}
\subfloat[mapping of individual sample paths]{%
\includegraphics[height=0.12\textheight]{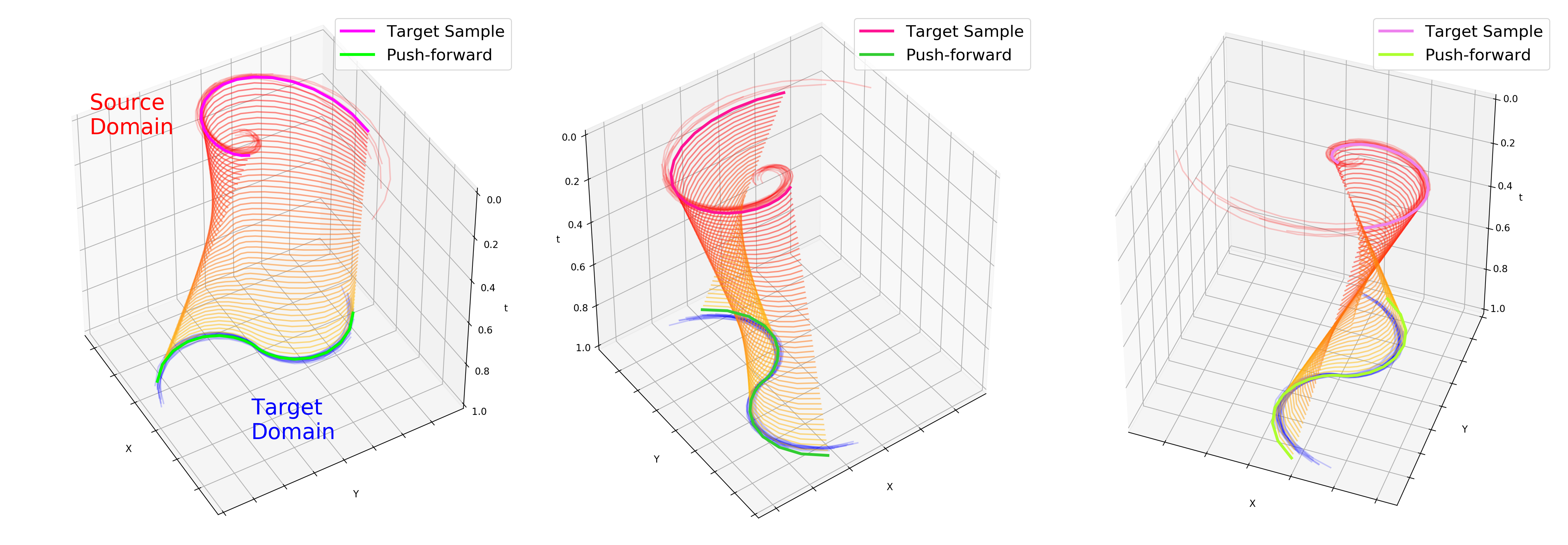}}
\hspace*{\fill}
\subfloat[The pushforward]{%
\includegraphics[height=0.12\textheight]{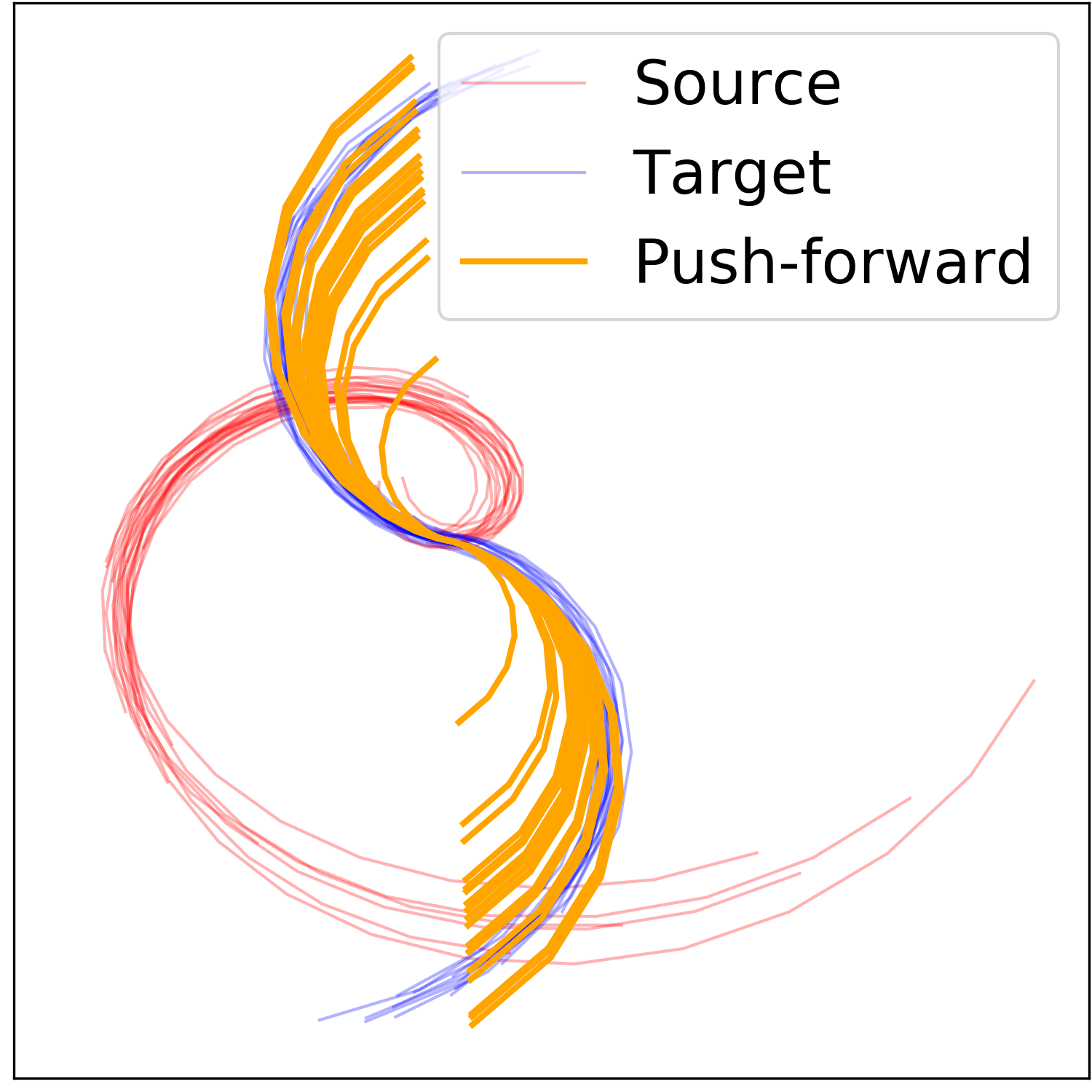}}

\begin{subfigure}[b]{\textwidth}
\centering
\includegraphics[width=\linewidth]{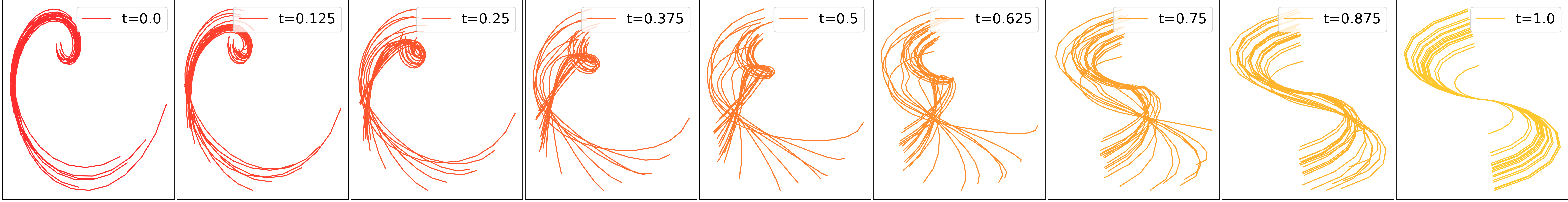}%
\caption{the geodesic parameterized by $t\in [0,1]$ connecting the source and target domains.}
\end{subfigure}
\caption{Illustration of an estimated pushforward map that sends sample paths from the source of {\color{red} Swiss-roll curves} to the target of {\color{blue} Wave curves}. (a) Datasets are a collection of continuous sample paths. (b) Three individual samples are mapped from source to target. (c) Resulting curves obtained by applying the pushforward map to the source's samples. (d) Illustration of the resulting geodesic between source and target distributions.}
\label{fig:one}
\end{figure*}

\section{Functional optimal transport: optimization and convergence analysis}
\label{section:OT on Hilbert}

We are ready to introduce a formulation for the functional optimal transport (FOT) problem, by reposing on the foundation of functional data analysis described earlier. Then we shall introduce estimators for solving the functional optimal transport problem from empirical data. Since we are formulating an infinite dimensional optimization problem, care must be taken to ensure the existence, uniqueness, and consistency of our proposed estimators, given sampled functions from source and target domains.


 
Given Hilbert spaces of functions $H_1$ and $H_2$, which are endowed with Borel probability measures $\mu$ and $\nu$, respectively, we wish to find a Borel map $\Gamma: H_1 \mapsto H_2$ such that $\nu$ is the pushforward measure of $\mu$ by $\Gamma$. Expressing this statement probabilistically, if $f \sim \mu$ represents a random element of $H_1$, then $\Gamma f $ is a random element of $H_2$ and $\Gamma f \sim \nu$. As noted in Section~\ref{section:prelim_knowledge}, such a map may not always exist. Thus one is interested in finding a map by which the resulting pushforward measure approximates as well as possible the target distribution. This motivates the following formulation:
\begin{eqnarray}
\label{eq:gamma}
\Gamma := \arg\inf_{T \in \mathscr{B}_{HS}(H_1, H_2)} W_2(T_\# \mu,\nu),
\end{eqnarray}
where $T_\# \mu$ is the pushforward of $\mu$ by $T$, and $W_2$ is the 2-Wasserstein distance of probability measures on the metric space $(H_2, \|\cdot\|_{H_2})$ (we suppress the dependence on $H_2$ of the notation $W_2$ for ease of notations because we only consider 2-Wasserstein distance on $H_2$ in this work). The space of solutions of Eq.~\eqref{eq:gamma} may still be large and the problem itself might be ill-posed. Thus we consider imposing a shrinkage penalty, which leads to the problem of finding the infimum of the following objective function $J: \Bspace \rightarrow \mathbb{R}_+$:
\begin{eqnarray}
\label{eq:shrinkage}
\inf_{T \in \Bspace} J(T),\; J(T) :=  W_2^2(T_\# \mu,\nu) + \eta \|T\|_{HS}^2,
\end{eqnarray}
where $\eta > 0$. It is natural to study the objective function $J$ and ask if it has a unique minimizer. 
To characterize this problem precisely, we shall place a mild condition on the moments of $\mu$ and $\nu$, which are typically assumed for probability measures on Hilbert spaces \citep{lei2020convergence}. 
\begin{enumerate}
    \item[(A.1)] Assume \begin{equation}\label{cond:finite2moments}
    E_{f_1 \sim \mu} \|f_1\|_{H_1}^2 < \infty, \quad E_{f_2 \sim \nu} \|f_2\|_{H_2}^2 < \infty.
\end{equation}
\end{enumerate}

Several key properties of objective function \eqref{eq:shrinkage} can be established as follows. 
\begin{lemma}\label{lem:propsofJ} Under assumption (A.1), the following statements hold.
    \begin{enumerate}[itemsep=-0.5mm]
        \item[(i)] $W_2(T \# \mu, \nu)$ is a Lipschitz continuous function of $T \in \Bspace(H_1, H_2)$, which implies that 
        $J:\Bspace(H_1, H_2) \rightarrow \mathbb{R}_+$ is also continuous.
        \item[(ii)] $J$ is a strictly convex function.
        \item[(iii)] There are constants $C_1,C_2 >0$ such that $J(T)\leq C_1 \|T \|_{HS}^2 + C_2\;\; \forall T\in \Bspace(H_1, H_2)$.
        \item[(iv)] $\lim_{\|T\|_{HS} \to \infty} J(T) = \infty$.
    \end{enumerate}
\end{lemma}
Thanks to Lemma~\ref{lem:propsofJ}, the existence and uniqueness properties can be established.

\begin{theorem}\label{thm:existuniq}
    Given (A.1), there exists a unique minimizer $T_0$ for problem \eqref{eq:shrinkage}.
\end{theorem}

The challenge of solving~\eqref{eq:shrinkage} is that this is an optimization problem in the infinite-dimensional space of operators $\Bspace$. To alleviate this complexity, we reduce the problem to a suitable finite-dimensional approximation. We follow techniques in numerical functional analysis by taking a finite number of basis functions. 

In particular, for some finite $K_1, K_2$, let 
$\BKspace = Span(\{U_i \otimes V_j: i=\overline{1, K_1}, j = \overline{1, K_2}\})$, where $K = (K_1, K_2)$. This yields the optimization problem of $J(T)$ over the space $T \in \BKspace$. The following result validates the choice of approximate optimization.

\begin{lemma}\label{lem:convergeTK}
For each $K=(K_1,K_2)$, there exists a unique minimizer $T_K$ of $J$ over $\BKspace$. Moreover, $T_K \to T_0$ in $\|\cdot\|_{HS}$ as $K_1, K_2\to \infty$.
\end{lemma}

\paragraph{Consistency of M-estimator} 
In practice, we are given i.i.d. samples $f_{11}, f_{12}, \dots, f_{1n_1}$ from $\mu$ and $f_{21}, f_{22}, \dots, f_{2n_2}$ from $\nu$, the empirical version of our optimization problem becomes:
\begin{equation}
    \label{eq:empirical_objective}
  \inf_{T\in \Bspace}\hat{J}_{n}(T), \;\; \hat{J}_{n}(T) := W_2^2(T_\# \hat{\mu}_{n_1}, \hat{\nu}_{n_2}) + \eta \|T \|_{HS}^2,
\end{equation}
where $\hat{\mu}_{n_1} = \dfrac{1}{n_1}\sum_{l=1}^{n_1} \delta_{f_{1,l}}$ and $\hat{\nu}_{n_2} = \dfrac{1}{n_2}\sum_{k=1}^{n_2} \delta_{f_{2,k}}$ are the empirical measures, and $n = (n_1, n_2)$. 
We proceed to show that the minimizer of this problem exists and provides a consistent estimate of the minimizer of  \eqref{eq:shrinkage}. The common technique to establish consistency of M-estimators is via the uniform convergence of objective functions $\hat{J}_n$ to $J$. The technical challenge here is that $\Bspace(H_1, H_2)$ is unbounded and locally non-compact. Thus care must be taken to ensure that the minimizer of~\eqref{eq:empirical_objective} is eventually bounded so that a suitable uniform convergence behavior can be established, as explicated in the following key lemma:  
\begin{lemma}
\label{lem-unifconv}
Under assumption (A.1), the following hold.
\begin{enumerate}[itemsep=-1mm]
\item For any fixed $C_0 >0$, 
\begin{equation}
    \sup_{\|T\|_{HS} \leq C_0} |\hat{J}_n(T) - J(T)| \xrightarrow{P} 0 \quad (n\to \infty).
\end{equation}
\item For any $n, K$, $\hat{J}_n$ has a unique minimizer $\hat{T}_{K, n}$ over $\BKspace$. Moreover, there exists a finite constant $D$ such that $P(\sup_{K} \|\hat{T}_{K, n} \|_{HS} < D) \to 1$ as $n\to \infty$.    
\end{enumerate}
\end{lemma}
Building upon the above results, we can establish consistency of our $M$-estimator when there are enough samples and the dimensions $K_1, K_2$ are allowed to grow with the sample size:

\begin{theorem}\label{thm:convergence}
The minimizer of Eq. \eqref{eq:empirical_objective} for $\hat{T}_{K,n} \in \BKspace$ is a consistent estimate for the minimizer of Eq. \eqref{eq:shrinkage}. Specifically, $\hat{T}_{K, n} \xrightarrow{P} T_0$ in $\|\cdot \|_{HS}$ as $K_1, K_2, n_1, n_2 \to \infty$.
\end{theorem}


It is worth emphasizing that the consistency of estimating $\hat{T}_{K,n}$ is ensured as long as sample sizes and approximate dimensions are allowed to grow. The specific schedule at which $K_1,K_2$ grows relatively to $n_1,n_2$ will determine the rate of convergence to $T_0$, which is also dependent on the choice of regularization parameter $\eta>0$, the true probability measures $\mu,\nu$, and the choice of CONS. It is of great interest to have a refined understanding on this matter. In practice, we can choose $K_1,K_2$ by a simple cross-validation technique, which we shall discuss further in the sequel.

Finally, note that in Theorem~\ref{thm:convergence}, we assume that the data samples consist of the entire sample curves $\{f_{1,l}, f_{2,k}\}$ for $l=1,\ldots, n_1$, $k=1,\ldots, n_2$. In reality, the sampled functions $f_{1,l}$ and $f_{2,k}$ may be partially given only at selected design points on their domains. A consistency theorem, Theorem~\ref{thm:consistency-nkd}, for our estimator in this more realistic setting will be given in the following section. 


\section{Approximation, functional design and optimization}
\label{section:methodology}

We will now translate the theoretical formulation for functional optimal transport and the regularized M-estimation framework presented in the preceding section into implementable algorithms. Recall that our FOT formulation is intrinsically an infinite dimensional problem: both source and target distributions are supported by infinite dimensional Hilbert spaces of functions, and so is the space of transport maps that we seek to estimate. On the other hand, we are given only a finite sample $(n_1,n_2)$ of the source and target functions, and moreover such functions are observed only at a finite number of design points on their domains. Thus our approach is to derive approximate algorithms to approach the original objective of solving Eq's ~\eqref{eq:shrinkage} and ~\eqref{eq:empirical_objective} by appropriate finite-dimensional approximations and by taking into considerations design choices for the space of functions and operators. Moreover, the consistency of our estimator in such practical forms of approximation is still maintained.



\subsection{Approximation of the HS operator}

Lemma \ref{lem:convergeTK} in the previous section paves the way for us to find an approximate solution to the original fully continuous infinite-dimensional problem, by utilizing finite sets of basis function, in the spirit of Galerkin method \citep{fletcher1984computational_Galerkin}, which is justified by the consistency theorem (Theorem~\ref{thm:convergence}). 
Thus, we can focus on solving the optimization problem given in Eq. \eqref{eq:empirical_objective} instead of Eq. \eqref{eq:shrinkage}.


Choosing a basis $\{U_i\}_{i=1}^{\infty}$ of $H_1$ and a basis $\{V_j\}_{j=1}^{\infty}$ of $H_2$, and fixing $K_1, K_2$, we want to find $T$ based on the $K_1\times K_2$ dimensional subspace of $\Bspace(H_1, H_2)$ with the basis $\{U_i \otimes V_j\}_{i=\overline{1, K_1}, j=\overline{1, K_2}}$. Lemma \ref{lemma:HS norm} gives us the following formula for $T$ and its norm
\begin{eqnarray}
 T = \sum_{i=1}^{K_1} \sum_{j=1}^{K_2} \lambda_{ji} U_i \otimes V_j, \quad \|T\|_{HS}^2 = \sum_{i=1}^{K_1} \sum_{j=1}^{K_2} \lambda_{ji}^2.
\end{eqnarray}
As $T$ is represented by matrix $\mathbf{\Lambda} := (\lambda_{ji})_{j,i=1}^{K_2,K_1}$, the cost to move function $f_{1,l}$ in $H_1$ to $f_{2,k}$ in $H_2$ is
\begin{equation}\label{eq:cost_old}
    \|Tf_{1,l} - f_{2,k} \|^2 = \left\|\sum_{i=1}^{K_1} \sum_{j=1}^{K_2} \lambda_{ji} V_j \langle f_{1,l}, U_i \rangle_{H_1} - f_{2,k} \right\|_{H_2}^2 =: C_{lk}(\mathbf{\Lambda}).
\end{equation}
Hence, the optimization problem \eqref{eq:empirical_objective} as restricted to $\BKspace$ can be written as
\begin{eqnarray}\label{eq:optimize_connection}
\min_{T\in \BKspace} \hat{J}_n(T) = \min_{\mathbf{\Lambda} \in \mathbb{R}^{K_2\times K_1}, \pi \in \hat{\Pi}} \sum_{l,k=1}^{n_1,n_2} \pi_{lk} C_{lk}(\mathbf{\Lambda}) + \eta \|\mathbf{\Lambda}\|_F^2,
\end{eqnarray}
where $\|\cdot\|_F$ is the Frobenius norm, and the empirical joint measure $\hat{\Pi} := \{ \mathbf{\pi} \in (\mathbb{R}^+)^{n_1 \times n_2} |\; \mathbf{\pi} \mathbf{1}_{n_2} = \mathbf{1}_{n_1}/{n_1},\; \mathbf{\pi}^T \mathbf{1}_{n_1} = \mathbf{1}_{n_2}/{n_2} \}$ with $\mathbf{1}_n$ a length $n$ vector of ones.
Eq.\eqref{eq:optimize_connection} indicates we need to simultaneously learn the HS operator $T$ and the coupling distribution $\pi$. 

For theoretical purposes we proceed to simplify the objective \eqref{eq:optimize_connection} to arrive at a finite-dimensional formulation specific to the basis $\{U_i\}_{i=1}^{\infty}$ of $H_1$ and basis $\{V_j\}_{j=1}^{\infty}$ of $H_2$. For each $l=1,\ldots, n_1$ and $k=1,\ldots, n_2$, by Parseval's identity on $H_2$, 
\begin{eqnarray}\label{eq:simplified-costs}
    C_{lk}(\mathbf{\Lambda}) = \underbrace{\sum_{j=1}^{K_2} \left|\sum_{i=1}^{K_1} \lambda_{ji} \langle f_{1,l}, U_i  \rangle_{H_1} - \langle f_{2,k}, V_j  \rangle_{H_2}\right|^2}_{D_{lk}(\mathbf{\Lambda})} + \sum_{j=K_2+1}^{\infty} \left|\langle f_{2,k}, V_j  \rangle_{H_2}\right|^2.
\end{eqnarray}
Our optimization problem becomes 
\begin{align*}
    \sum_{l, k} \pi_{lk} C_{lk}(\mathbf{\Lambda}) + \eta \| \mathbf{\Lambda}\|_F^2 & = \sum_{l, k} \pi_{lk} D_{lk}(\mathbf{\Lambda}) + \sum_{j=K_2+1}^{\infty}\sum_{k=1}^{n_2} (\sum_{l=1}^{n_1}\pi_{lk}) \left|\langle f_{2,k}, V_j  \rangle_{H_2}\right|^2 + \eta \| \mathbf{\Lambda}\|_F^2\\
    & = \sum_{l, k} \pi_{lk} D_{lk}(\mathbf{\Lambda}) + \sum_{j=K_2+1}^{\infty} \sum_{k=1}^{n_2} \dfrac{1}{n_2} \left|\langle f_{2,k}, V_j  \rangle_{H_2}\right|^2+ \eta \| \mathbf{\Lambda}\|_F^2.
\end{align*}
Since the second term in the above display does not depend on  $\mathbf{\Lambda}$ and $\pi$, it does not affect the optimization problem. The term  $D_{lk}(\mathbf{\Lambda})$ can be further written as
\begin{equation}\label{eq:cost_new}
     D_{lk}(\mathbf{\Lambda}) = \|\mathbf{\Lambda} a_l - b_k\|_{2}^2,
\end{equation}
where $a_{li} =  \langle f_{1,l}, U_i  \rangle_{H_1}$, and $ a_l = (a_{li})_{i=1}^{K_1}$ are vectors in $\mathbb{R}^{K_1}$ (coordinates of $f_{1,l}$ in the first $K_1$ basis); $b_{kj} =  \langle f_{2,k}, V_j  \rangle_{H_2}$, and $b_k = (b_{kj})_{j=1}^{K_2}$ are vectors in $\mathbb{R}^{K_2}$ (coordinates of $f_{2,k}$ in the first $K_2$ basis). This leads to an equivalent presentation for \eqref{eq:optimize_connection}
\begin{eqnarray}\label{eq:optimize_connection_new}
\min_{T\in \BKspace} \hat{J}_n(T) = \min_{\mathbf{\Lambda} \in \mathbb{R}^{K_2\times K_1}, \pi \in \hat{\Pi}} \sum_{l,k=1}^{n_1,n_2} \pi_{lk} D_{lk}(\mathbf{\Lambda}) + \eta \|\mathbf{\Lambda}\|_F^2.
\end{eqnarray}
In this way, we easily see a direct connection between a finite-dimensional approximation of the functional optimal transport relative to a pair of orthonormal bases $\{U_i\}_{H_1}$ and $\{V_j\}_{H_2}$ to a corresponding OT problem on fixed (finite) dimensional vectors.

\subsection{Functional data computation via design points}

In real-world applications with data in the functional domains, one typically does not directly observe functions $(f_{1,l})_{l=1}^{n_1}$ and $(f_{2,k})_{k=1}^{n_2}$ but only their values $(\mathbf{y}_{1,l})_{l=1}^{n_1}$ and $(\mathbf{y}_{2,k})_{k=1}^{n_2}$ at design points  $(\mathbf{x}_{1,l})_{l=1}^{n_1}$ and $(\mathbf{x}_{2,k})_{k=1}^{n_2}$, respectively, where $\mathbf{x}_{1,l} \in X_1^{d_{1,l}}, \mathbf{y}_{1,l} \in \mathbb{R}^{d_{1,l}}, \mathbf{x}_{2,k} \in X_2^{d_{2,k}}, \mathbf{y}_{2,k} \in \mathbb{R}^{d_{2,k}}\,\,\forall \, l=1,\ldots, n_1; k=1,\ldots, n_2$, { and $X_1$ and $X_2$ denote the space of design points for the source and the target domain, respectively}. In other words, $d_{1,l}$ and $d_{2,k}$ denote the possibly varying number of design points observed for the sampled function $f_{1,l}$ in the source and $f_{2,k}$ in the target domain, respectively.
In order to evaluate the objective~\eqref{eq:optimize_connection_new}, the relevant inner products in $H_1$ and $H_2$ must be approximated using the observed values of sampled functions given at the design points. 

{ As a concrete example, suppose that $H_1 = H_2 = L^2([0, 1])$ so that the ordered design points $\mathbf{x}_{1,l} \in [0, 1]^{d_{1,l}}$ and $\mathbf{x}_{2,k} \in [0, 1]^{d_{2,k}}$ for all $l, k$.
Moreover, assume that the support of $\mu$ and $\nu$ are contained in the subsets of continuous functions of $L_2$, and the basis functions $\{U_i\}$ and $\{V_j\}$ are continuous, then a simple numerical strategy that one can use to approximate $\langle f_{1,l}, U_i \rangle_{H_1}$ is by the Riemann sum approximation
\begin{equation}\label{eq:approximation-Riemann}
    \langle f_{1,l}, U_i \rangle_{d} := \sum_{j=2}^{d_{1,l}} ({\mathbf{x}_{1l, j} - \mathbf{x}_{1l, j-1}}) f(\mathbf{x}_{1l, j}) U_i(\mathbf{x}_{1l, j}).
\end{equation}
In the above display, the subindex $d$ is used to signify the fact our estimate of the relevant inner products for the sampled function is calculated using design points. We also say that $d\to \infty$ if $d_{1,l}, d_{2,k} \to\infty \, \forall\, l, k$. 
Thanks to the continuity of the sampled functions, for any functions $f_1, f_2$ being in the support of $\mu, \nu$ respectively, we have as $d \to \infty$
\begin{equation}\label{eq:approximation-scheme}
    \langle f_1, U_i \rangle_{d} \to \langle f_1, U_i \rangle_{H_1}, \quad \langle f_2, V_j \rangle_{d} \to \langle f_2, V_j \rangle_{H_2} \quad \forall i, j\in \mathbb{N}.
\end{equation}
It is clear that the more design points where the function values are observed, the better the representation of continuous functions via a given orthonormal basis. In fact, this condition is sufficient for us to establish the consistency of our estimation procedure for the transport map as the number of design points increases. We formalize this intuition with the following consistency theorem, where we want to emphasize that this result holds for all approximation schemes satisfying Eq.~\eqref{eq:approximation-scheme}, not only the approximation~\eqref{eq:approximation-Riemann}.}

\begin{theorem}\label{thm:consistency-nkd}
(i)        For every $n_1, n_2, K_1, K_2$ and sequences of design points in source and target domains, the cost function
\begin{eqnarray}\label{eq:optimize_connection_new_d}
    \hat{J}_{n, K, d}(\mathbf{\Lambda}) = \min_{\pi \in \hat{\Pi}} \sum_{l,k=1}^{n_1,n_2} \pi_{lk} D_{lkd}(\mathbf{\Lambda}) + \eta \|\mathbf{\Lambda}\|_F^2,
\end{eqnarray}
where 
\begin{equation*}
     D_{lkd}(\mathbf{\Lambda}) = \|\mathbf{\Lambda} a_{ld} - b_{kd}\|_2^2,
\end{equation*}
in which $a_{ld} =  (\langle f_{1,l}, U_i  \rangle_{d})_{i=1}^{K_1}$ and  $b_{kd} =  (\langle f_{2,k}, V_j  \rangle_{d})_{i=1}^{K_2}\, \forall l, k$, has unique minimizer $\mathbf{\Lambda}_{n, K, d} \in \mathbb{R}^{K_2\times K_1}$ that corresponds to operator $T_{n, K, d}$. 

(ii) Suppose that for any natural index pair $(i, j)$, there holds
\begin{equation}
    \langle f, U_i \rangle_d \to \langle f, U_i \rangle_{H_1},  \langle g, V_j \rangle_d \to \langle g, V_j \rangle_{H_2}, 
\end{equation}
almost surely as $d\to \infty$, where $f \sim \mu$ and $g \sim \nu$. Then for any sequences of $n_1, n_2, K_1, K_2 \to \infty$ and $d\to \infty$ with a rate depends on $n_1, n_2, K_1, K_2$, we have $T_{n, K, d} \xrightarrow{P} T_0$ in $\|\cdot \|_{HS}$. Here, $T_0$ denotes the minimizer of the population version of FOT given in Eq.~\eqref{eq:shrinkage}.
\end{theorem}


We make the following remarks regarding the evaluation of equivalent objectives given in Eq.~\eqref{eq:optimize_connection}, \eqref{eq:optimize_connection_new} and their estimate \eqref{eq:optimize_connection_new_d}.

\begin{itemize}
\item 
It is worth noting that the objective function \eqref{eq:optimize_connection_new_d} can be computed easily from the sampled function observations. Moreover, our method works even in the case where different functions are observed at different design points (with possibly different numbers of design points). It is quite obvious that one \emph{cannot} treat each function as a multidimensional vector to apply existing multivariate OT techniques in this case due to the dimensions mismatch.


\item The objectives derived in the foregoing require the selection of basis functions for the Hilbert space in both the source and target domains. 
Since our method requires finite-dimensional approximations, a particular choice of orthonormal bases may have a substantial impact on the number of basis functions that one ends up using for approximating the support of the distributions (of the source and the target domain), and for the representation of the approximate pushforward map going from one domain to another. 
Note that increasing $K_1$ and $K_2$ can lower the objective function, but it may negatively affect the generalization of the estimate as we only observe a finite number of sampled functions (at a finite number of design points). Cross-validation is a simple and very effective technique for choosing $K_1, K_2$ and regularization parameters $\eta, \gamma$. This technique will be demonstrated via a simulation study in Section~\ref{section_simulation}.

\item As an example of the choice of basis functions for $H_1$ and $H_2$, we may take those that arise from a user-specified kernel via Mercer's theorem. Recall that if $K$ is a continuous, symmetric, non-negative definite kernel on a measurable space $(E, \mathcal{B}, \mu)$ then it admits the following representation
\begin{equation}
        K(s, t) = \sum_{j=1}^{\infty} \lambda_j \phi_j(s) \phi_j(t),
    \end{equation}
where the convergence is absolute and uniform, and $(\phi_j)_{j=1}^{\infty}$ forms an orthogonal basis for $L^2(E, \mathcal{B}, \mu)$. Examples of such bases can be found in \citep{wang2008karhunen, zhu1997gaussian_eigenfunction}.

\item { A more adaptive approach for estimating basis functions is to use empirical samples. This can be achieved through the analysis of functional principal components (FPCA).}
{
Since FPCA is employed in some of our numerical experiments, we provide a brief introduction here. The literature on FPCA is extensive and covers a wide range of topics, such as incorporating smoothness~\citep{foutz2010research_fda,liu2017functional_fpca}, robustness~\citep{gervini2008robust_fda}, and sparsity~\citep{kayano2010sparse_fda_fpca}.
For theoretical perspectives of FPCA, see \cite{dauxois1982asymptotic_fpca,yao2005functional_fpca,hsing2015theoretical_fda}.

A basic starting point of FPCA is a result of Karhunen and Lo\'eve, who developed the optimal series expansion theory for continuous stochastic processes. Given $N$   real valued functions $y_1,\ldots, y_N \in L^2$ defined 
on a closed interval $\mathcal{T} \subset \mathbb{R}$.  We would like to select a collection of weight functions $\beta: \mathcal{T} \rightarrow \mathbb{R}$ that outlines the most significant types of variation, which is captured by
$$f_i = \langle \beta, y_i \rangle_2 = \int \beta(t) y_i(t) dt, i=1,...,N.$$
Similar to the multivariate principal component analysis technique, one first finds a weight function $\beta_1$ by solving the maximization problem:
$$\max_{\beta} \frac{1}{N} \sum_{i=1}^N ( \langle \beta_1, y_i \rangle_2)^2, \| \beta_1 \|^2_2=1.$$
Then, inductively, for $m > 1$, one finds the next weight function $\beta_m$ that maximizes $\frac{1}{N} \sum_{i=1}^N (\int \beta_m x_i)^2$ with the normalizing restriction $\| \beta_m\|^2_2=1$ and the orthogonality restrictions $\langle \beta_k, \beta_m \rangle_2=0, k<m$. The obtained weight functions are called the principal components and can serve as the basis functions. In practice, the maximization problem can be seen as an eigenvalue problem so that we can employ efficient methods~\citep{ramos2022scikit_fda} such as singular value decomposition (Section 2.5 of \citep{bie2021functional_fda}). 
}
\end{itemize}

\subsection{Optimization algorithms} 
\label{sec:opt}
We are now ready to describe in detail the optimization algorithm for solving Eq.~\eqref{eq:optimize_connection}, or equivalently, Eq.~\eqref{eq:optimize_connection_new}.
Recall that for functions $(f_{1,l})_{l=1}^{n_1}$ and $(f_{2,k})_{k=1}^{n_2}$ we observe their values $(\mathbf{y}_{1,l})_{l=1}^{n_1}$ and $(\mathbf{y}_{2,k})_{k=1}^{n_2}$ at design points $(\mathbf{x}_{1,l})_{l=1}^{n_1}$ and $(\mathbf{x}_{2,k})_{k=1}^{n_2}$, respectively, where $\mathbf{x}_{1,l}, \mathbf{y}_{1,l} \in \mathbb{R}^{d_{1,l}}, \mathbf{x}_{2,k}, \mathbf{y}_{2,k} \in \mathbb{R}^{d_{2,k}}\,\,\forall \, l, k$. So the objective function that we use for the optimization is 
\begin{eqnarray}
    \argmin_{\mathbf{\Lambda} \in \mathbb{R}^{K_2 \times K_1}, \pi \in \hat{\Pi} }  \sum_{l,k=1}^{n_1,n_2} \pi_{lk} C_{lk}(\mathbf{\Lambda})   + \eta \|\mathbf{\Lambda} \|_{F}^2 ,
    \label{eq:obj_44}
\end{eqnarray}
where
\begin{eqnarray}
\label{eq:cost_function_matrix}
C_{lk}(\mathbf{\Lambda}) = \left\| \mathbf{V}_{2k} \mathbf{\Lambda} \mathbf{U}^T_{1l} \mathbf{y}_{1,l} - \mathbf{y}_{2,k} \right\|_2^2,
\end{eqnarray}
where $\mathbf{U}_{1l} = [U_{1}(\mathbf{x}_{1,l}), \dots, U_{K_1}(\mathbf{x}_{1,l})] \in \mathbb{R}^{d_{1,l}\times K_1}, \mathbf{V}_{2k} = [V_{1}(\mathbf{x}_{2,k}), \dots, V_{K_2}(\mathbf{x}_{2,k})] \in \mathbb{R}^{d_{2l}\times K_2}$ are basis functions that evaluated on the given design points (equivalently, we can work directly  with~\eqref{eq:optimize_connection_new_d}).


To speed up the computation of the classical optimal transport objective, a useful technique is to include a negative entropic penalty \citep{Cuturi-Sinkhorn-13} defined as  
$\Omega_{\gamma}(\pi) = \gamma_h \sum_{l,k=1}^{n_1,n_2}\pi_{lk} \log \pi_{lk}$. However, entropic regularization keeps the probabilistic coupling dense and causes the lack of sparsity \citep{blondel2018smooth_sparse}. To promote sparsity, we can impose an $\ell_p$ penalty by taking $\Omega_{\gamma}(\pi) = \gamma_p \sum_{l,k=1}^{n_1, n_2} \pi_{lk}^p$, for some $p \geq 1, \gamma_p > 0$. This ensures that the optimal coupling $(\pi_{lk})$ has fewer active parameters thereby easing up the computing burden for large datasets. This can be considered as promoting a robustness criterion in addition to shrinkage, a similar behavior associated with the Huber loss \citep{Huber-loss}.
Hence, by imposing an additional regularization term to Eq. (\ref{eq:obj_44}), we have our objective as
\begin{eqnarray}
\label{eq:optimize_matrix_continuous}
\argmin_{\mathbf{\Lambda} \in \mathbb{R}^{K_2 \times K_1}, \pi \in \hat{\Pi} }  \sum_{l,k=1}^{n_1,n_2} C_{lk}(\mathbf{\Lambda})  \pi_{lk} + \eta \|\mathbf{\Lambda} \|_{F}^2 + \Omega_{\gamma}(\pi),
\end{eqnarray}
where $\eta > 0$ is the regularization coefficient and $\Omega_{\gamma}(\pi)$ is the additional regularization term. 


\begin{algorithm}[h]
   \caption{Joint Learning of $\mathbf{\Lambda}$ and $\pi$}  
   \label{alg:gradient_descent}
\begin{algorithmic}
   \STATE {\bfseries Input:} Observed functional data $\{ f_{1,l}=(\mathbf{x}_{1,l}, \mathbf{y}_{1,l}) \}_{l=1}^{n_1}$ and $\{ f_{2,k}=(\mathbf{x}_{2,k}, \mathbf{y}_{2,k}) \}_{k=1}^{n_2}$, coefficient $\gamma_h$, $\gamma_p$, $\eta$, and learning rate $l_r$, source and target CONS $\{U_i(\cdot)\}_{i=1}^{K_1}$,  $\{V_j(\cdot)\}_{j=1}^{K_2}$.  
   \STATE  Initial value  $\mathbf{\Lambda}_0 \xleftarrow{} \mathbf{\Lambda}_{ini}$, ${\pi}_0 \xleftarrow{} {\pi}_{ini}$.
   \STATE $\mathbf{U}_{1l} = [U_1(\mathbf{x}_{1,l}),...,U_{K_1}(\mathbf{x}_{1,l})]$, $\mathbf{V}_{2k} = [V_1(\mathbf{x}_{2,k}),...,V_{K_2}(\mathbf{x}_{2,k})]$  \null\hfill $\#$ Evaluate eigenfunctions  
   \FOR{$t=1$ {\bfseries to} $T_{\max}$}
   \STATE $\#$ Step 1. Update ${\pi}_{t-1}$ 
   \STATE ${C}_{lk} \xleftarrow{} \| \mathbf{V}_{2k} \mathbf{\Lambda}_t \mathbf{U}^T_{1l} \mathbf{y}_{1,l}  - \mathbf{y}_{2,k} \|^2_F$ \null\hfill $\#$ Cost matrix by  Eq.(\ref{eq:cost_function_matrix}) 
   \STATE  ${\pi}_t \xleftarrow{} \text{argmin}_{{\pi}} \mathcal{L}({\pi}, \lambda; \rho)$ \null \hfill $\#$ Fix $\mathbf{\Lambda}$ update $\pi$
   \STATE $\#$ Step 2. Update $\mathbf{\Lambda}_{t-1}$ with gradient descent
   \STATE Learn $\mathbf{\Lambda}_t$, solve Eq. \eqref{eq:optimize_matrix_continuous} with fixed ${\pi}_t$ using gradient descent
   \ENDFOR
   \STATE {\bfseries Output:} ${\pi}_{T_{max}}$, $\mathbf{\Lambda}_{T_{max}}$
\end{algorithmic}
\end{algorithm}


We provide a solution for local minima of this objective via an alternative minimization over $\mathbf{\Lambda}$ and $\pi$, whereby the first is fixed while the second is minimized, followed by the second fixed and the first minimized. 
The algorithm is described in Algorithm~\ref{alg:gradient_descent} and the explicit calculations are given below. 
Note that here we introduce our algorithm following the most general setting by using Eq. (\ref{eq:cost_function_matrix}) as the transportation cost. 
Also we use the power regularization $\Omega_{\gamma}(\pi) = \gamma_p \sum_{l,k=1}^{n_1, n_2} \pi_{lk}^p$ in our objective. Later we will show that using the entropy regularization can let us utilize the Sinkhorn algorithm during the update.

\noindent \textbf{\underline{Updating $\mathbf{\Lambda}$ with $\pi$ fixed:}} Here we want to solve
\begin{eqnarray}
\label{eq_app:optimization_matrix_continuous_Lambda}
\mathbf{\Lambda}_t =  \argmin_{\mathbf{\Lambda} \in \mathbb{R}^{K_2 \times K_1}} L(\mathbf{\Lambda}, \pi) = \argmin_{\mathbf{\Lambda} \in \mathbb{R}^{K_2 \times K_1}} \sum_{l,k=1}^{n_1,n_2}  \pi_{lk} C_{lk}(\mathbf{\Lambda})  + \eta \|\mathbf{\Lambda} \|_{F}^2. 
\end{eqnarray}
The minimum is achieved by performing gradient descent updates, where the gradient is:
\begin{eqnarray}
\nabla_{\mathbf{\Lambda}} L(\mathbf{\Lambda}, \pi) = 
2 \sum_{l=1}^{n_1} \sum_{k=1}^{n_2} \pi_{lk} \left[ ( \mathbf{\Lambda} \mathbf{U}^T_{1l} \mathbf{y}_{1,l} - \mathbf{V}^T_{2k} \mathbf{y}_{2,k}  ) \mathbf{y}_{1,l}^T \mathbf{U}_{1l} \right] + 2 \eta \mathbf{\Lambda}. 
\end{eqnarray}

\noindent \textbf{\underline{Updating $\pi$ with $\Lambda$ fixed:}} Now we want to solve
\begin{eqnarray}
\label{eq_app:optimization_matrix_continuous_pi}
{\pi}_t =  \argmin_{\pi \in \hat{\Pi}} L(\mathbf{\Lambda}, \pi) = \argmin_{\pi \in \hat{\Pi}} \sum_{l,k=1}^{n_1,n_2} C_{lk}(\mathbf{\Lambda})  \pi_{lk}  + \gamma_p \sum_{l,k=1}^{n_1, n_2} \pi_{lk}^p.
\end{eqnarray}
To optimize for the probabilistic coupling $\pi$, we can consider this as a constrained linear programming problem { and solve it through the augmented Lagrangian method~\citep{afonso2010augmented_cov_opt}. 

It is straightforward to extend the optimization framework to accommodate discrete source and target probability measures given by the (non-uniform) weights $p^s = (p^s_1,…,p^s_{n_1})$ and $p^t=(p^t_1,…,p^t_{n_2})$; the constraints are depicted as $\pi \in \hat{\Pi}:= \{ \sum_{l}^{n_1} \pi_{lk} = p^t_k, \forall k;  \sum_{k}^{n_2} \pi_{lk} = p^t_l, \forall l \} $. Here, we add a slack variable $s$ to enforce the inequality constraints $ \forall p_{ij} \geq 0$. Then the augmented Lagrangian takes the form}
\begin{eqnarray}
\begin{split}
    &\mathcal{L}( \mathbf{\pi}, s_{lk}, \lambda^a, \lambda^b, \lambda ) 
    =  \sum_{l,k=1}^{n_1,n_2}  C_{lk} \pi_{lk} + \gamma_p \sum_{l,k=1}^{n_1, n_2} \pi_{lk}^p \\
    & + \sum_{k=1}^{n_2} \lambda^a_k (\sum_{l=1}^{n_1} \pi_{lk} - p^t_k) + \sum_{l=1}^{n_1} \lambda^b_l (\sum_{k=1}^{n_2} \pi_{lk} - p^s_l) 
    + \frac{\rho_k}{2} (\sum_{l=1}^{n_1} \pi_{lk} - p^t_k)^2 + \frac{\rho_l}{2} (\sum_{k=1}^{n_2} \pi_{lk} - p^s_l)^2 \\
    & + \sum_{l,k=1}^{n_1,n_2} \lambda_{lk}(\pi_{lk} - s_{lk}) + \sum_{l,k=1}^{n_1,n_2} \frac{\rho_{lk}}{2}(\pi_{lk} - s_{lk})^2.
\end{split}
\end{eqnarray}
In the above display, $\mathbf{\lambda}^a \in \mathbb{R}^{n_1 \times 1}$, $\mathbf{\lambda}^b \in \mathbb{R}^{n_2 \times 1}$, $\mathbf{\lambda} \in \mathbb{R}^{n_1 \times n_2}$ are Lagrange multipliers, $s_{lk} \in \mathbb{R}^{n_1 \times n_2}$ are the slack variables. The sub-problem is
\begin{eqnarray}
\begin{split}
    & {\pi}_t, {s_{lk}}_t = \argmin_{\pi, s_{lk}} \mathcal{L} (\mathbf{\pi}, s_{lk},
    \lambda^k, \lambda^l, \lambda^{lk}) \\
    & \lambda^a_k \leftarrow \lambda^a_k + \rho_k (\sum_{l=1}^{n_1} \pi_{lk} - p^t_k) \\
    & \lambda^l_t = \lambda^l_{t-1} + \rho_l (\sum_{k=1}^{n_2} \pi_{lk} - p^s_l) \\
    & \lambda^{lk}_t = \lambda^{lk}_{t-1} + \rho_{lk} (\sum_{l,k=1}^{n_1,n_2} \pi_{lk} - s_{lk}).
\end{split}
\end{eqnarray}

\textbf{}


\noindent\textbf{\underline{Entropic regularization:}}
We may alternatively set the additional regularization term to be the negative entropy $\Omega_{\gamma}(\pi) = \gamma_h \sum_{l,k=1}^{n_1,n_2}\pi_{lk} \log \pi_{lk}$ to leverage the computational efficiency of the Sinkhorn algorithm. In that case, when updating $\pi$ with $\mathbf{\Lambda}$ fixed, our problem reduces to an entropic regularized optimal transport problem:
\begin{eqnarray}
\label{eq_app:optimization_sinkhorn}
{\pi}_t =  \argmin_{\pi \in \hat{\Pi}} L(\mathbf{\Lambda}, \pi) = \argmin_{\pi \in \hat{\Pi}} \sum_{l,k=1}^{n_1,n_2} C_{lk}(\mathbf{\Lambda})  \pi_{lk}  + \gamma_h \sum_{l,k=1}^{n_1,n_2}\pi_{lk} \log \pi_{lk}.
\end{eqnarray}
This formulation reverts to a strictly convex optimization problem and we can efficiently obtain the solution via the Sinkhorn-Knopp algorithm \citep{Cuturi-Sinkhorn-13}. See Algorithm \ref{alg:sinkhorn_appendix}.

\begin{algorithm}[h!]
   \caption{Sinkhorn algorithm  }  
   \label{alg:sinkhorn_appendix}
\begin{algorithmic}
   \STATE {\bfseries Input:} Cost matrix $\mathbf{C} \in \mathbb{R}^{N \times n}$, entropy coefficient $\gamma_h$
   \STATE $\mathbf{K} \xleftarrow{} \exp (- \mathbf{C} / \gamma_h)$, $ \mathbf{\nu} \xleftarrow{} \frac{\mathbf{1}_n}{n}$
   \WHILE{ not converged}
   \STATE $\mathbf{\mu} \xleftarrow{}  \frac{\mathbf{1}_N}{N} \oslash \mathbf{K} \mathbf{\nu} $
   \STATE $\mathbf{\nu} \xleftarrow{}  \frac{\mathbf{1}_n}{n} \oslash \mathbf{K}^T \mathbf{\mu} $
   \ENDWHILE
   \STATE $\mathbf{\Pi} \xleftarrow{} \textrm{diag}(\mathbf{\mu}) \mathbf{K} \textrm{diag}(\mathbf{\nu})$
   \STATE {\bfseries Output:} $\mathbf{\Pi}$
\end{algorithmic}
\end{algorithm}
To summarize our learning scheme, during each iteration, our algorithm performs a gradient-type update for $\mathbf{\Lambda}$ with $\pi$ fixed, followed by a step that updates the $\pi$. For the latter step, the algorithm either minimizes $\pi$ following the Lagrangian multiplier method when using the power regularization, or invokes the Sinkhorn algorithm when using entropic regularization. 

{

We end the description of the algorithms with the following additional remarks.
\begin{itemize}
\item In its final form in Eq.~\eqref{eq:optimize_matrix_continuous}, our optimization formulation has a number of regularization terms. It is interesting to note how the different penalty terms play complementary roles in the final estimate of the joint parameter $(T,\Pi)$ which represents the optimal coupling. Specifically, the penalty $\|T\|_{HS}$ (and equivalently, $\|\Lambda\|_{F}$) is required so as the overall optimization is well-posed (with unique solution for $T$ in the space of Hilbert-Schmidt operators). The entropic regularization term for $\Pi$ serves to speed up for the computation, while the powered penalty helps to induce a sparser representation for the coupling, in addition to typically reducing the variance in its estimate from empirical data. In the next section, the roles of these regularization terms will be assessed via a simulation study.


\item Beyond the augmented Lagrangian method and Sinkhorn algorithm there are a variety of optimization approaches, e.g., the Alternating Direction Method of Multipliers (ADMM)~\citep{ghadimi2014optimal_admm} and the Hungarian algorithm~\citep{kuhn1955hungarian}, which may be employed for solving the discretized optimal transport problem. Each method exhibits specific advantages from a computational standpoint; for instance, ADMM is adept at handling distributed computing contexts.  
A thorough investigation of the variety of aforementioned algorithms and their properties within the context of the functional optimal transport problem is of interest and a subject of future research.
\end{itemize}
}

{
\noindent\textbf{\underline{Functional PCA:}} 
While the basis functions can be specified from a list of orthogonal basis families such as the Hermite polynomials or via Mercer kernels, a more data-driven approach is to estimate the basis from data using the functional PCA approach, as briefly introduced in the previous subsection. 

Let $\{y_i(x)\}_{i=1}^N$ denote the function values observed at design points $x$ so that from each function sample a vector $y_i$ of length $M$ is obtained. Then the data are presented by $\mathbf{Y}=[ \mathbf{y}_1, \ldots, \mathbf{y}_N]^\top$. 
The covariance operator associated with the function samples is approximated by matrix $\mathbf{V} = \frac{1}{N} \mathbf{Y}^\top \mathbf{Y}$. The integrals can be approximated as $\int f(x) dx \approx \sum_{m=1}^M w_m f(x_m) = \mathbf{w}^\top \mathbf{y}$ where $\mathbf{w}^\top= (w_1,…,w_M)$ is the weight vector that characterizes the numerical quadrature. Using $\mathbf{W} = \text{diag}(w)$ to denote the diagonal matrix of dimensions $M \times M$ whose elements in the diagonal are the elements of $w$, the eigen-equations for the discrete representation of functional data are

$$\mathbf{V} \mathbf{W} \beta_j (x) = \lambda_j \beta_j(x); j=1,\ldots,M.$$

Let $\nu_j(x) = \mathbf{W}^{1/2} \beta_j(x)$, then the above equations become $\mathbf{W}^{1/2}\left( \frac{1}{N} \mathbf{Y}^\top \mathbf{Y} \right) \mathbf{W}^{1/2} \nu_j (x) = \lambda_j \nu_j(x)$ for $j=1,\ldots,M.$ which can be solved by applying SVD to $\frac{1}{N} \mathbf{Y}\mathbf{W}^{1/2} = \mathbf{U} \mathbf{S} \mathbf{V}^\top$, where $\mathbf{U}=(\mathbf{u}_1, …,\mathbf{u}_N)$ is an orthonormal matrix and $\{ \mathbf{u}_n\}_{n=1}^N$ is a basis in $\mathbb{R}^N$. The corresponding eigenfunctions are represented by eigenvectors $\beta_j = \mathbf{W}^{-1/2} \mathbf{v}_j, j=1,…,M$.
Therefore, we obtain an effective approximation of basis functions when the sample curves are observed at regularly spaced design points. 
}

\section{Experiments}
\label{section:Experiments}

In this section we present a thorough simulation study to demonstrate the viability and effectiveness of our method and to validate the theory presented above. We will also compare our functional optimal transport method to other existing domain adaptation techniques in the literature. Finally, we will describe an application of FOT to a real-world task of multivariate robot-arm motion prediction.
    
\subsection{Simulation studies on the synthetic continuous functional dataset}
\label{section_simulation}
\subsubsection{Verifying consistency and interpretation of the transport map}
First, we present simulation studies to demonstrate that one can recover the "true" pushforward map via cross-validation.  
 We explicitly constructed a ground-truth map $T_0$ that has finite intrinsic dimensions $K_1^* = K_2^*=15$. Then we obtained the target curves by pushing forward source curves via $T_0$. The FOT algorithm is then applied to the data while $\hat{K}_1$ and $\hat{K}_2$ gradually being increased. The results are illustrated in Fig. \ref{fig:simlation}. They demonstrate the effects of varying the number of basis eigenfunctions $\hat{K}=(\hat{K_1}, \hat{K_2})$. We observed that the performance of the estimated map steadily improved as $\hat{K}$ increased until it exceeded $K^*$. As expected, further increasing the number of eigenfunctions did not reduce the learning objective. 

Next, we validate Lemma \ref{lem:convergeTK} by evaluating $\hat{T}_{\hat{K}}$ from an infinite dimensional map that transports sinusoidal functions. The Frobenius norm between the optimal $T^*_K$ and estimated $\hat{T}_{{K}}$, $\|T^*_K -\hat{T}_{{K}} \|_F$, decreases as ${K}$ increases. In both simulations, we set sample sizes $n_1=n_2=30$.
For hyperparameters, set $\gamma_h=20$, $\eta=1$. The results were found to be quite robust to other values of these hyperparameters. 
{
Finally, we verify Theorem \ref{thm:consistency-nkd}, by varying the numbers of sample ($n_1,n_2$) in estimating the optimal transport (OT) problems between two empirical measures. It is well-known that for any absolutely continuous measure $\mu$ on $\mathbb{R}^{d}$, we have $\mathbb{E} W_1(\hat{\mu}_n, \mu) \lesssim n^{-1/d}$, where $\hat{\mu}_n$ is the empirical measure of $\mu$ with $n$ samples \citep{dudley1969speed}. 
Here, we provide quantitative results to investigate the convergence of estimation with regard to sample size.
Similar to our previous setting, we gradually increased the sample size of both source and target ($n_1=n_2=n$). We set $K_1=K_2=30$ and used the same hyperparameters as before. We repeated the experiment 10 times for each sample size. As shown in Fig. (\ref{fig:consistency_for_n_design_points}), the Frobenius norm consistently decreases with increased sample sizes.
}

\begin{figure*}[ht]
\begin{subfigure}[]{0.95\linewidth}
  \centering
  \includegraphics[width=0.45\linewidth]{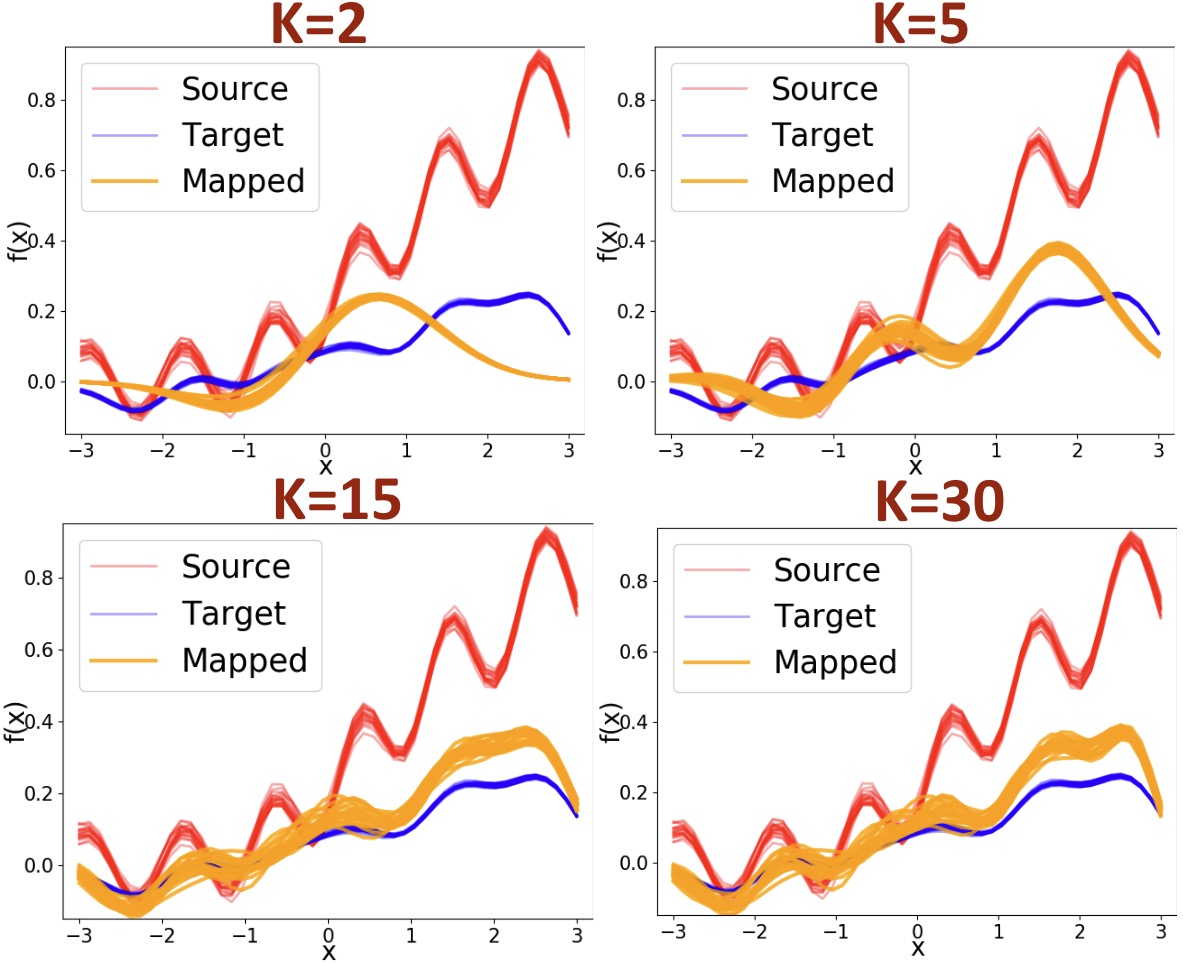}
  \hfill 
  \includegraphics[width=0.50\linewidth]{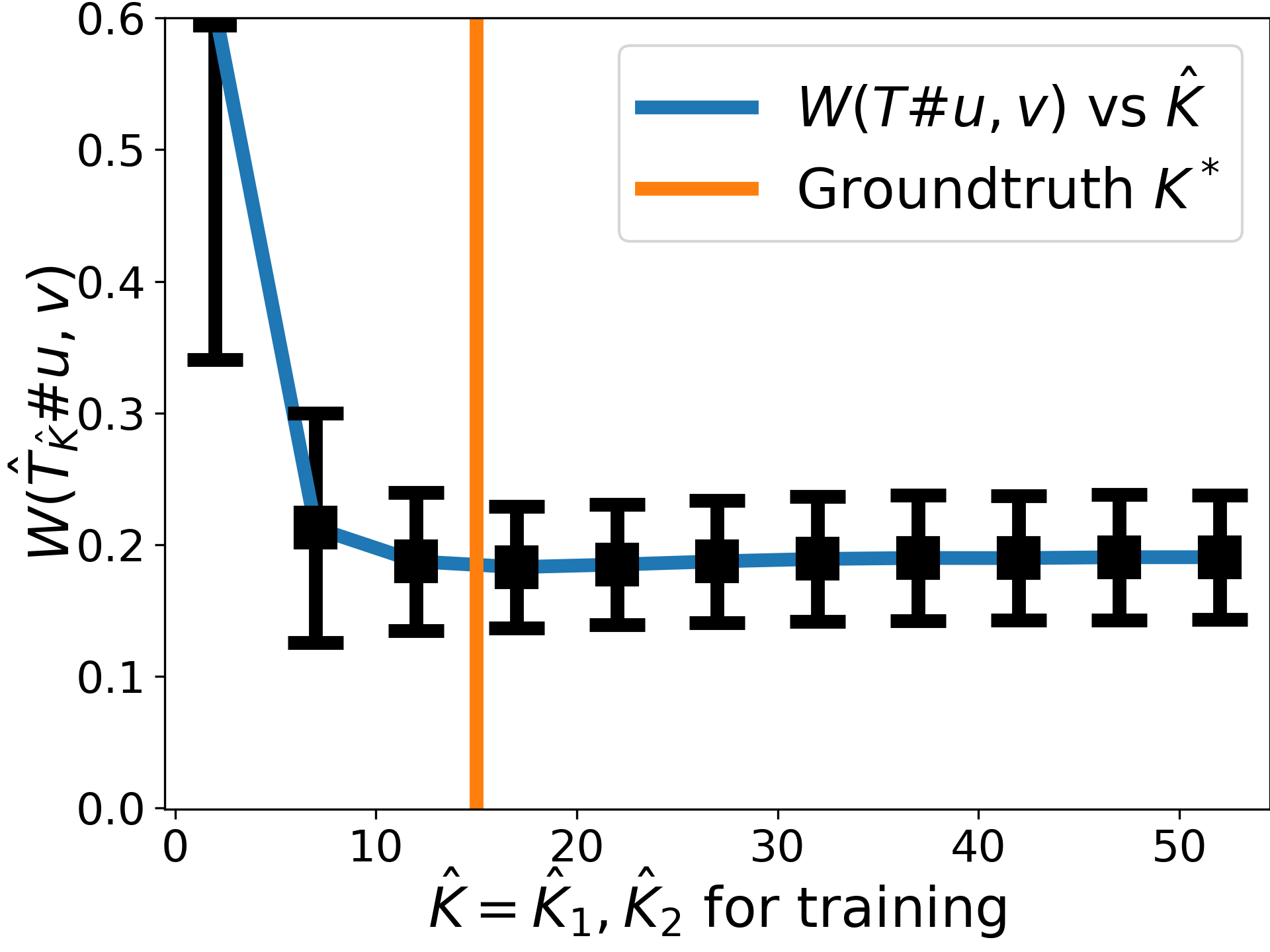}
  \caption{ As $\hat{K}$ increases, $T_{\hat{K}}\# f_1$ moves toward $f_2$ and $W(T_{\hat{K}}\# \hat{u}, \hat{v})$ decreases until $\hat{K} \geq K^*$. }
\end{subfigure}
\hfill
\begin{subfigure}[]{0.17\linewidth}
  \centering 
  \includegraphics[width=1.0\linewidth]{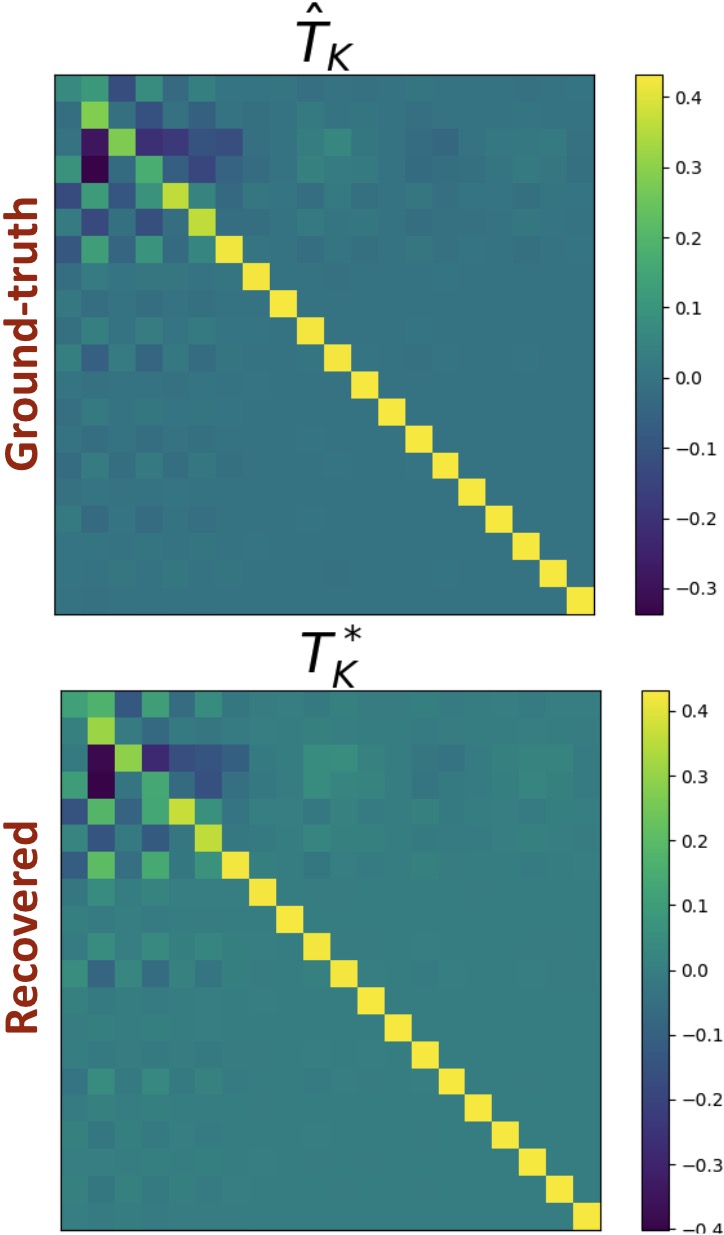}\caption{Estimate $T^*_K$}
  \label{fig:k_reconstruction}
\end{subfigure}
\begin{subfigure}[]{0.40\linewidth}
  \centering
  \includegraphics[width=1.0\linewidth]{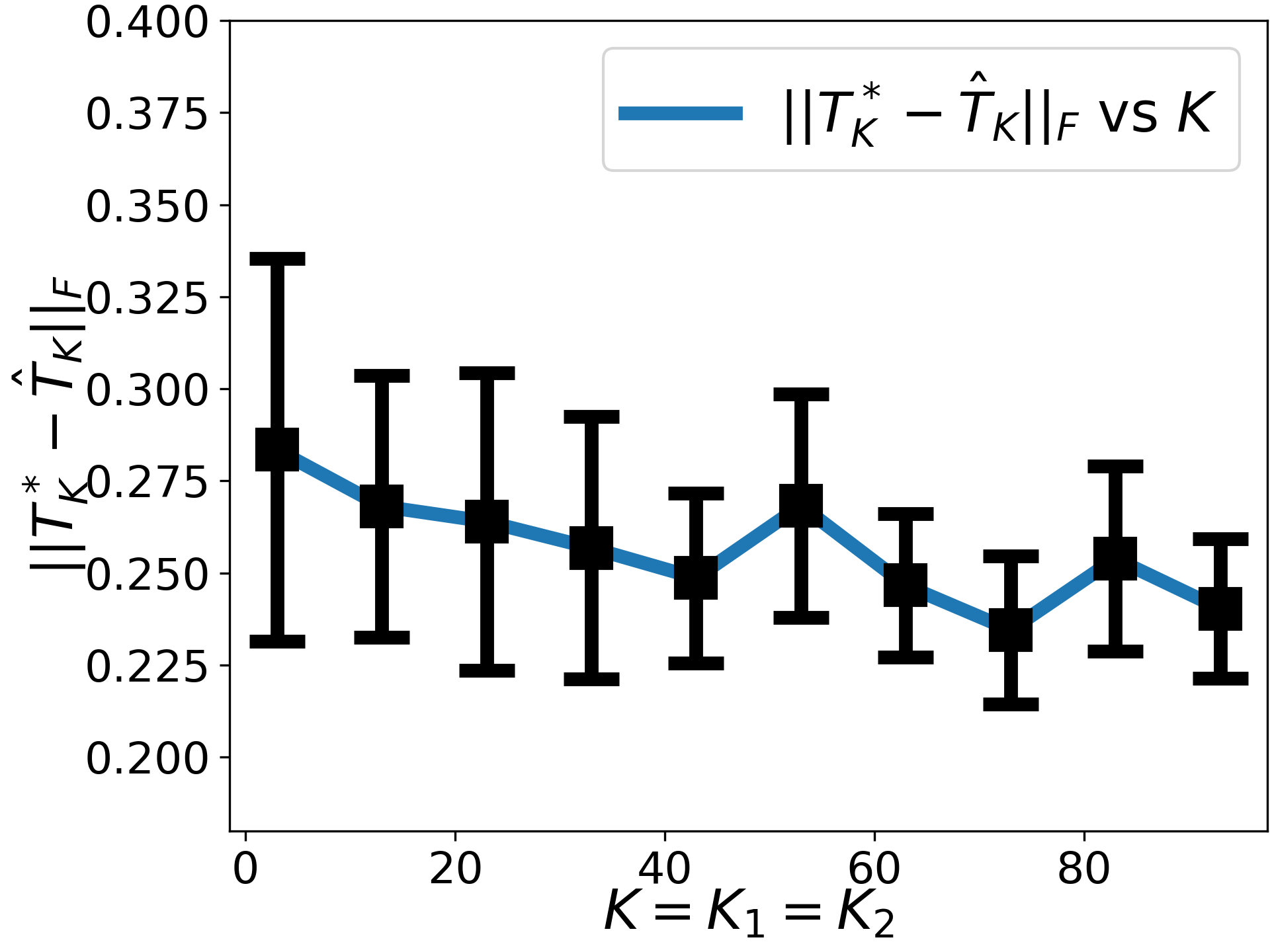}\caption{Number of basis functions}
  \label{fig:consistence_K} 
\end{subfigure}
\begin{subfigure}[]{0.40\linewidth}
  \centering   
  \includegraphics[width=1.0\linewidth]{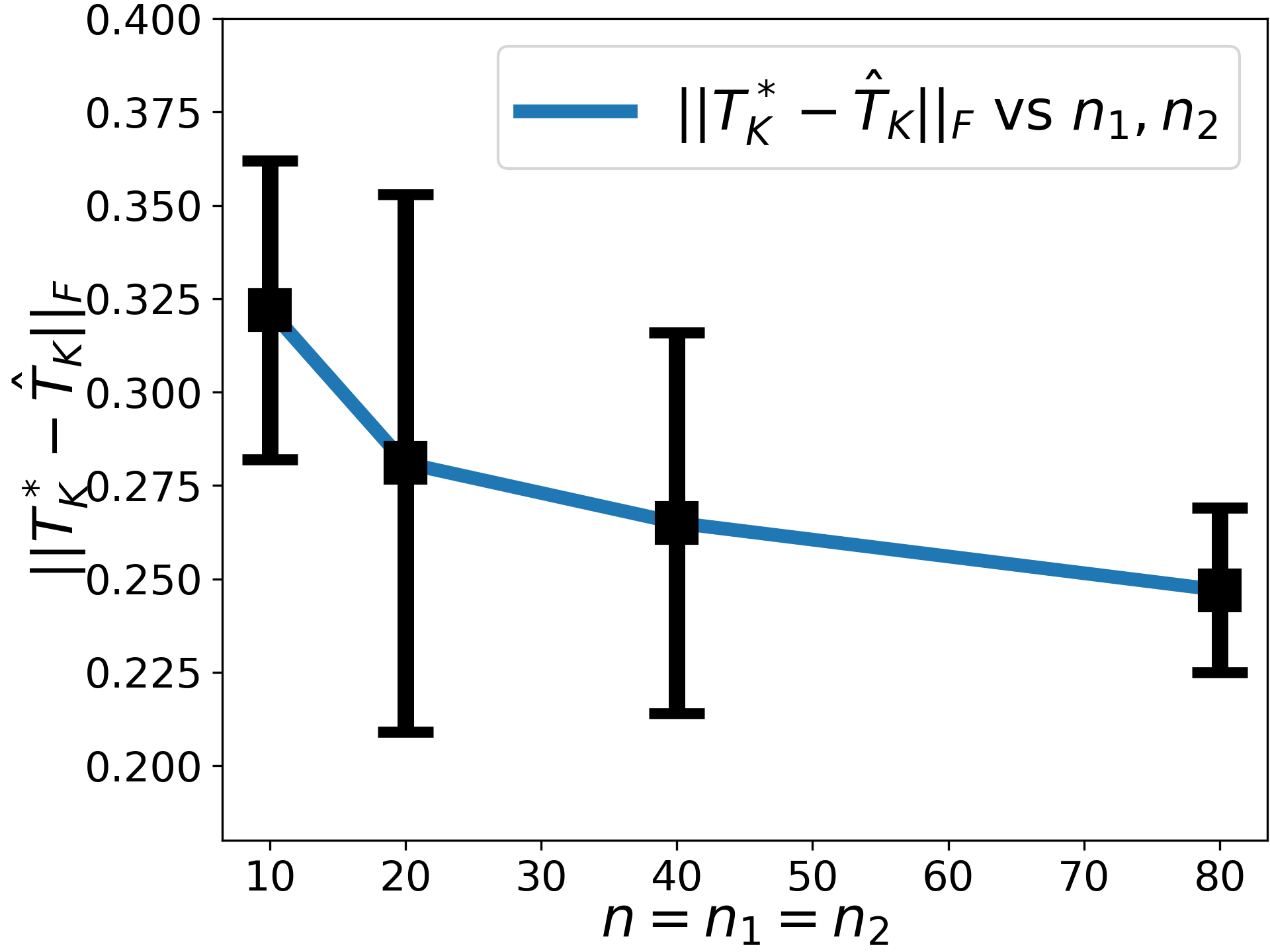}\caption{Sample size}
  \label{fig:consistency_for_n_design_points}
\end{subfigure}
\caption{
{ The experimental results that verify the convergence properties. The estimated linear operator $\hat{T}_K$ effectively recovers the groundtruth ${T}^*_K$ as shown in Fig. (\ref{fig:k_reconstruction}).
The map estimation error improves (decreases) with increased number of basis functions (Fig. (\ref{fig:consistence_K})) and sample curves $n_1,n_2$ (Fig. (\ref{fig:consistency_for_n_design_points})).
}
}
\label{fig:simlation}
\end{figure*}




\subsubsection{Map estimation}
\textbf{Synthetic data simulation:} 
We evaluated our FOT method on a synthetic dataset in which the source and target data samples were generated from a mixture of sinusoidal functions. Each sample $\{y_i (x_i)\}_{i=1}^{n}$ is a realization evaluated from
a (random) function 
$
y_i = A_k \sin (\omega_k x_i + \phi_k) + m_k
$
where the amplitude $A_k$, angular frequency $\omega_k$, phase $\phi_k$ and translation $m_k$ are random parameters generated from a probability distribution, i.e., $[A_k, \omega_k, \phi_k,m_k] \sim P(\theta_k)$, and $\theta_k$ represents the parameter vector associated with a mixture component.


\noindent\textbf{Baseline comparison:} We compared FOT method to several existing map estimation methods on the synthetic \textit{mixture of sinusoidal functions} dataset. Sample paths were drawn from sinusoidal functions with random parameters. Then, curves were evaluated on random index sets. 
In Fig. \ref{fig:toy_exp_qualitative}, FOT was compared against the following baselines: (i) Transport map of Gaussian processes
\citep{mallasto2017learning_wGP,masarotto2019procrustes_GP}, where a closed-form optimal transport map is available, (ii) Large-scale optimal transport (LSOT)
\citep{seguy2017LSOT}, and (iii) Mapping estimation for discrete OT (DSOT) \citep{perrot2016mapping_est_discre_OT}. 
For all discrete OT methods, which were not designed for functional data per se,
the functional data were treated as point clouds of high dimensional vectors. 

We observed that FOT did a remarkably good job at transporting source sample curves to match closely target samples. By contrast, GPOT only altered the oscillation of curves but failed to capture the target distribution's multi-modality. This failure is attributed to the GPOT method's Gaussian process (and thus unimodal distribution) assumption which clearly did not hold in this example. The poor performance of LSOT and DSOT can be attributed to the fact these methods essentially ignored the smoothness of the sampled curves. In other words, these multivariate adaptations were not suitable for handling functional data. 

\begin{figure*}[ht]
\begin{subfigure}[]{0.32\linewidth}
  \centering
  \includegraphics[width=1.0\linewidth]{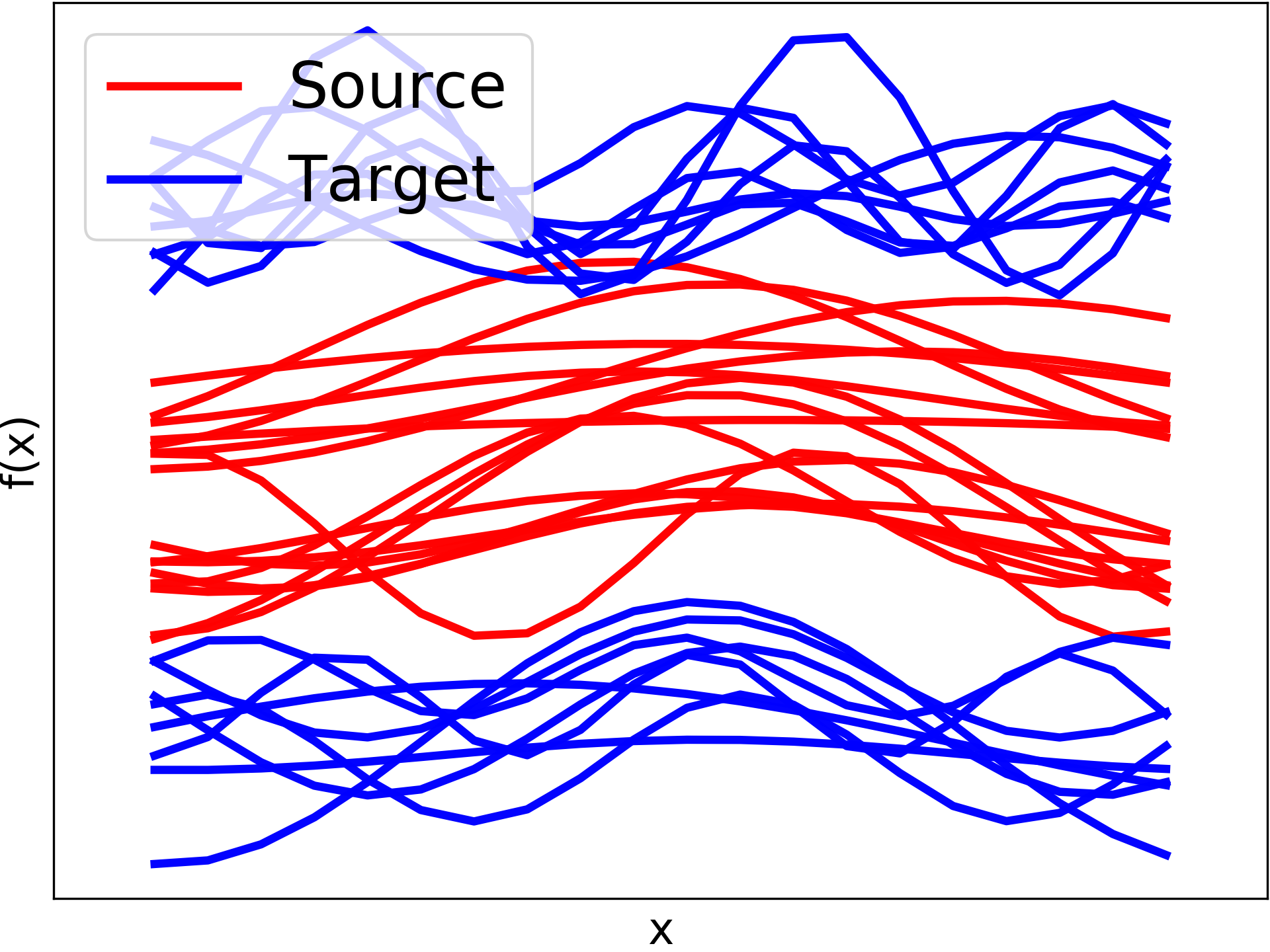}
  
  \caption{data}
  \label{fig:toy_exp_qualit_data}
\end{subfigure}
\hfill
\begin{subfigure}[]{0.32\linewidth}
  \centering
  \includegraphics[width=1.0\linewidth]{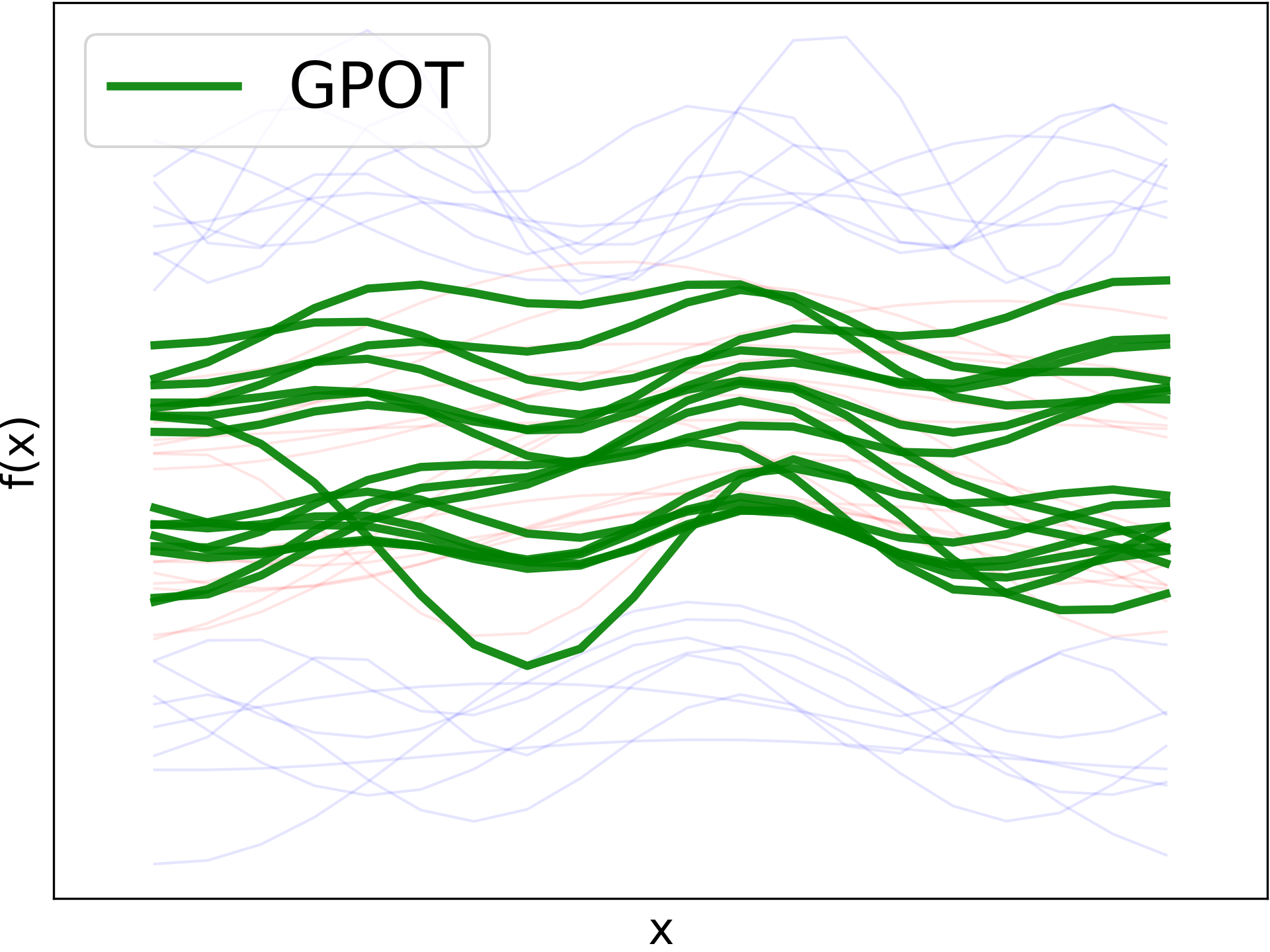}
  
  \caption{GPOT}
  \label{fig:toy_exp_qualit_GPOT}
\end{subfigure}
\hfill
\begin{subfigure}[]{0.32\linewidth}
  \centering
  \includegraphics[width=1.0\linewidth]{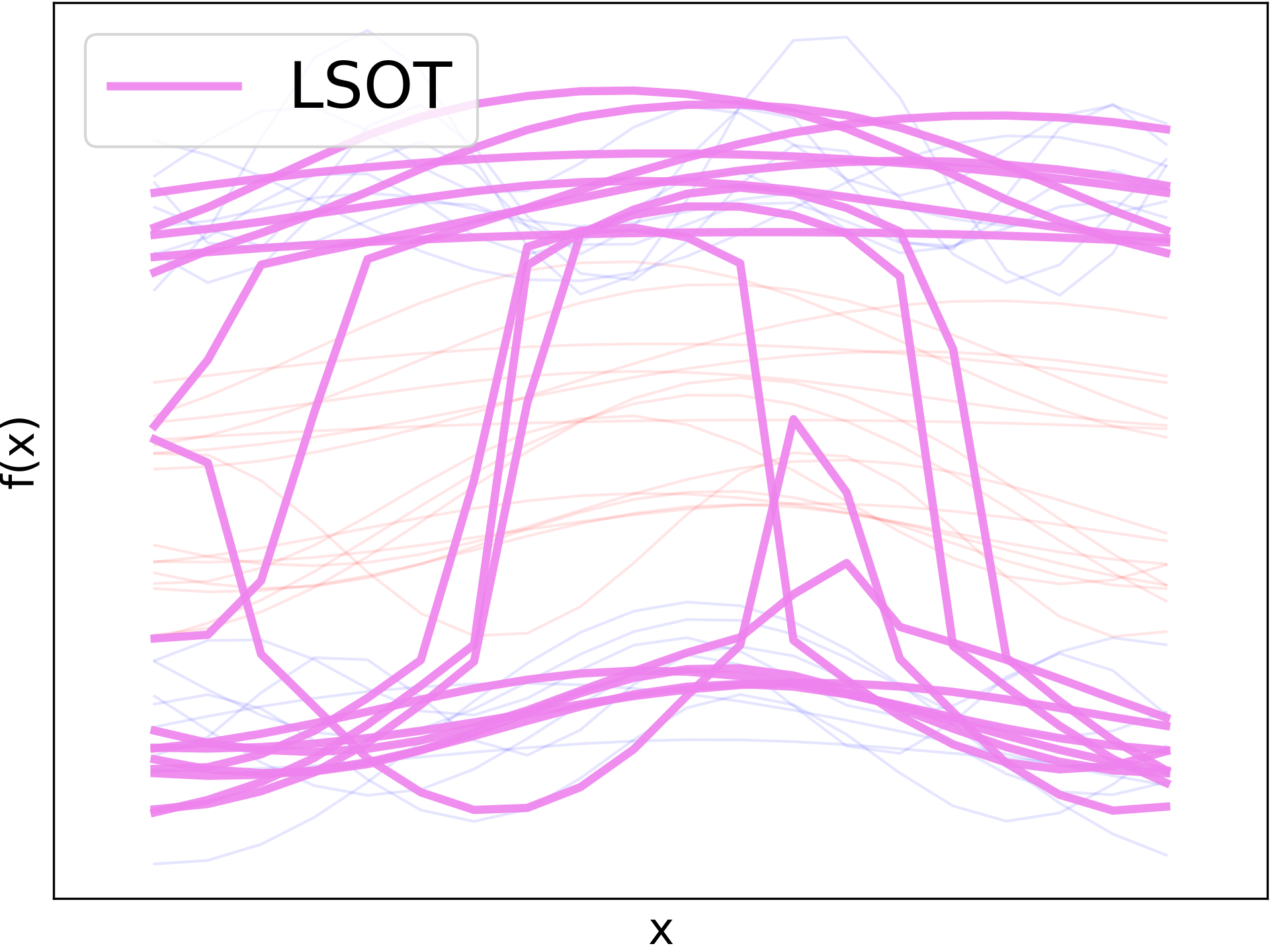}
  
  \caption{LSOT}
  \label{fig:toy_exp_qualit_LSOT}
\end{subfigure}
\hfill
\begin{subfigure}[]{0.32\linewidth}
  \centering
  \includegraphics[width=1.0\linewidth]{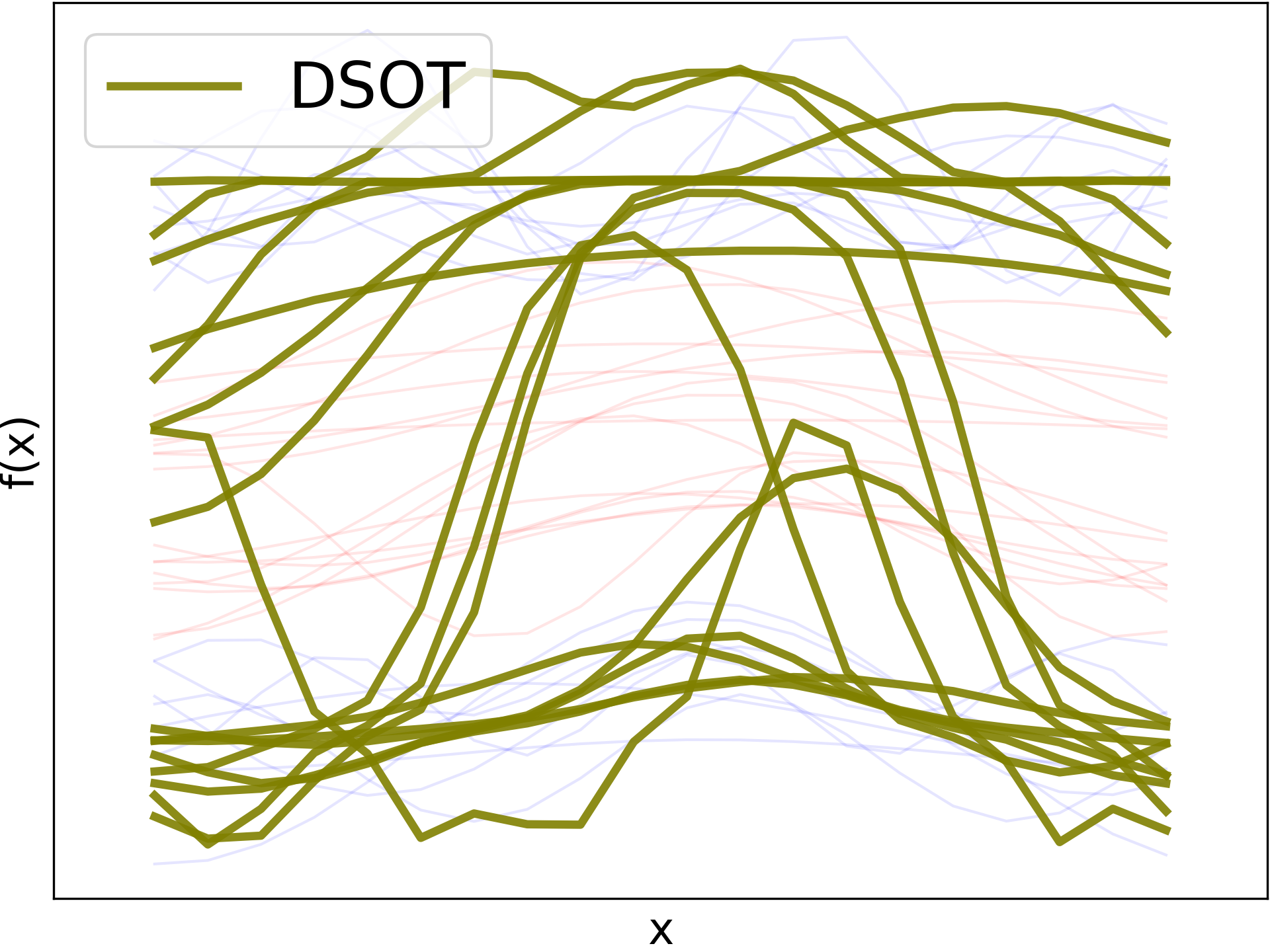}
  
  \caption{DSOT}
  \label{fig:toy_exp_qualit_FOT_DSOT}
\end{subfigure}
\hfill
\begin{subfigure}[]{0.32\linewidth}
  \centering
  \includegraphics[width=1.0\linewidth]{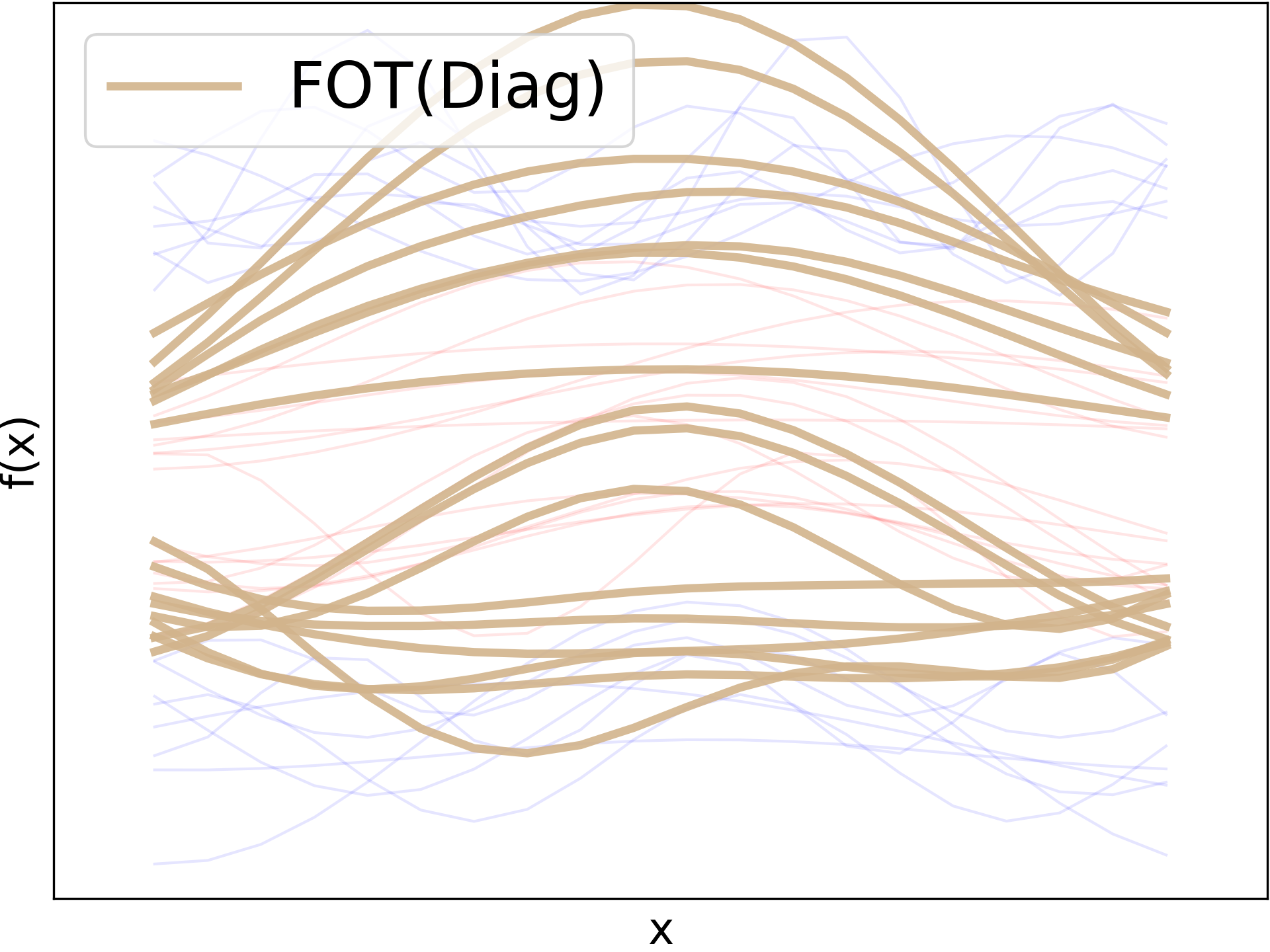}
  \caption{FOT with diagonal $\Lambda$}
  \label{fig:toy_exp_qualit_FOT_diag}
\end{subfigure}
\hfill
\begin{subfigure}[]{0.32\linewidth}
  \centering
  \includegraphics[width=1.0\linewidth]{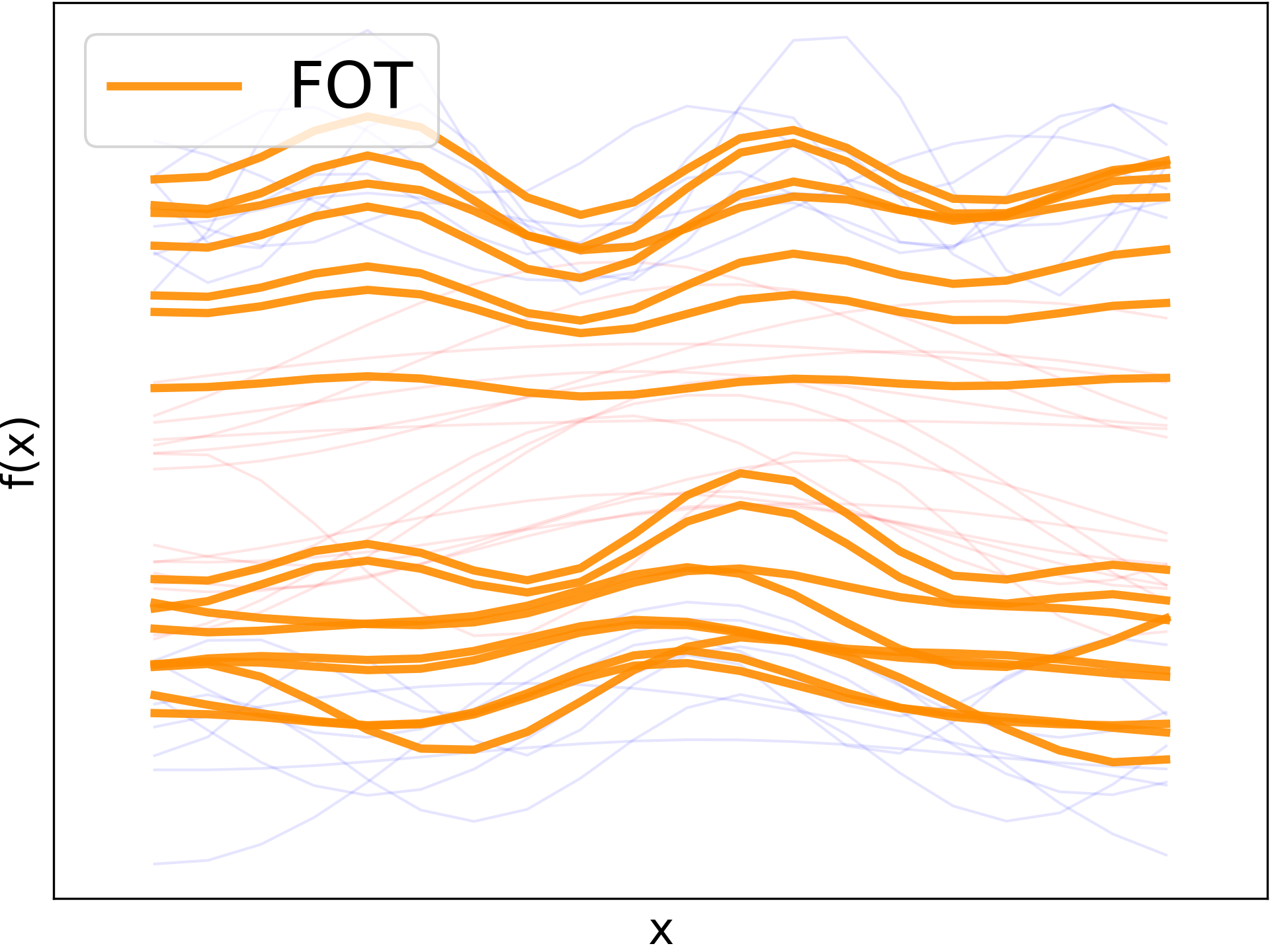}
  \caption{FOT}
  \label{fig:toy_exp_qualit_FOT}
\end{subfigure}

\caption{Pushforward measures of functions obtained by various approaches on mixtures of sinusoidal functions data: (a) Sample functions from {\color{red} source } and {\color{blue}target} domain. The resulting pushforward measures obtained by (b) GPOT \citep{mallasto2017learning_wGP}; (c) LSOT \citep{ke2019_LSOT_proj_pursuit}; and (d) DSOT \citep{perrot2016mapping_est_discre_OT}; and (e)(f) our method FOT. In (e) we parameterized the $\mathbf{\Lambda}$ as only a diagonal function. A full matrix $\mathbf{\Lambda}$ (f) produces a more expressive pushforward.}
\label{fig:toy_exp_qualitative}
\end{figure*}

For a quantitative comparison, we used the Wasserstein distance to measure how well the pushforward measure of (the empirical distribution of) source samples matches the target samples: 
\begin{eqnarray}
L = \min_{\mathbf{\Pi}} \frac{1}{n_L} \sum_{l, k} d( T (\mathbf{f_1}_l), \mathbf{f_2}_k) \Pi_{lk}. 
\end{eqnarray}
Here, $d(\mathbf{x}, \mathbf{y}):= \| \mathbf{x} - \mathbf{y} \|_2^2$, $ \{ T(\mathbf{f_1}_i) \}_{i=1}^{n_l}$ and $\{ \mathbf{f_2}_i \}_{i=1}^{n_k}$ are mapped samples and target samples, $T(\cdot)$ denotes the map given by different methods, $n_L$ the length of each sample function and $\Pi$ the probabilistic coupling. The experiments are labeled by $k_{source} \rightarrow k_{target}$, where $k_{source}$ and $k_{target}$ indicate the number of mixture components of the source and the target distribution, respectively. More mixture components typically entail more complex data distributions, and thus more complex transportation plan (more on this below).
As shown in Table \ref{table:toy_quantitative}, the pushforward map obtained by FOT performed the best in matching target sample functions quantitatively using the Wasserstein distance based objective $L$.

\begin{figure}[h]{
\centering
\begin{tabular}[h]{ p{1.2cm}|p{1.1cm}|p{1.1cm}|p{1.1cm}|p{1.1cm}|p{1.1cm}  }
 \hline
 \hline
 Method & 1 $\rightarrow$ 1 & 1$\rightarrow$2 & 2$\rightarrow$1 & 2$\rightarrow$2 & 2$\rightarrow$3  \\
 \hline
 GPOT   & 17.560    & 12.895 &   15.263 & 61.561 & 39.159 \\
 LSOT &   133.434  & 94.229   & 117.832 & 929.108 & 663.461 \\
 DSOT    & 6.871 & 13.226 &  9.679 & 46.521 & 41.009 \\
 FOT &   \textbf{2.873}  & \textbf{11.982} &  \textbf{3.316} & \textbf{44.071} & \textbf{32.547} \\
 \hline
\end{tabular}
\captionof{table}{Quantitative comparison on the mixture of sinusoidal functions data.
The maps obtained by FOT method achieved the best performance under the Wasserstein distance objective.  }
\label{table:toy_quantitative}
}
\end{figure}

\begin{figure*}[ht]
\begin{subfigure}[]{0.95\linewidth}
  \centering
  \includegraphics[width=1.0\linewidth]{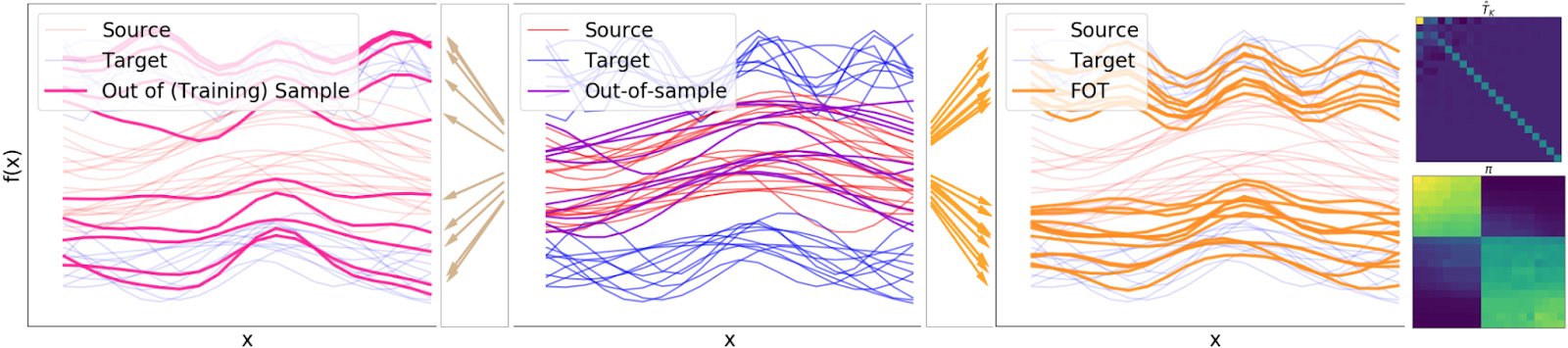}\caption{Out-of-sample curves. The rightmost and lower heatmap represents the coupling $\pi$ and reveals the multimodality.}
  \label{fig:out-of-sample}
\end{subfigure}
\hfill
\centering
\begin{subfigure}[]{0.31\linewidth}
  \centering
  \includegraphics[width=1.0\linewidth]{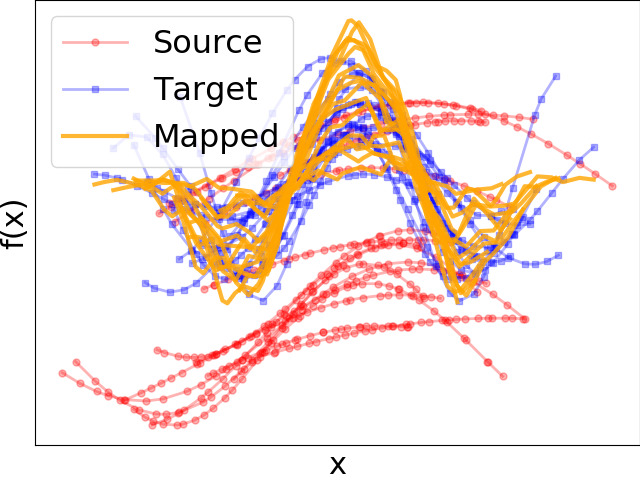}\caption{Varying design points}
  \label{fig:continuous-index}
\end{subfigure}
\begin{subfigure}[]{0.31\linewidth}
  \centering
  \includegraphics[width=1.0\linewidth]{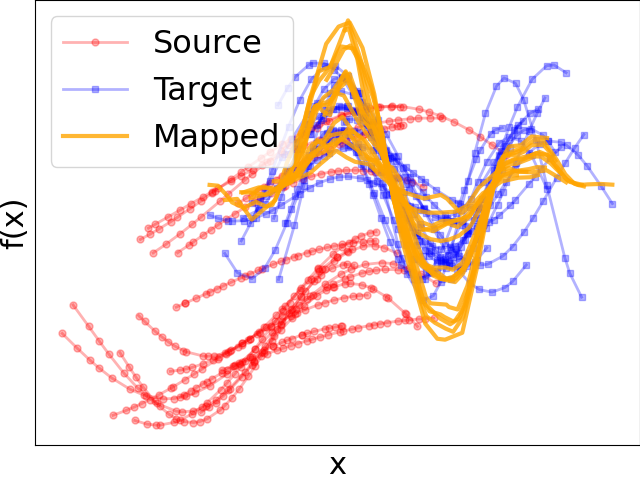}\caption{Distinct design point sets}
  \label{fig:continuous-index-out-of-domain}
\end{subfigure}
\caption{FOT performance on out-of-sample curves. }
\end{figure*}

\begin{figure*}[h]
\begin{subfigure}[]{0.32\linewidth}
  \centering
  \includegraphics[width=1.0\linewidth]{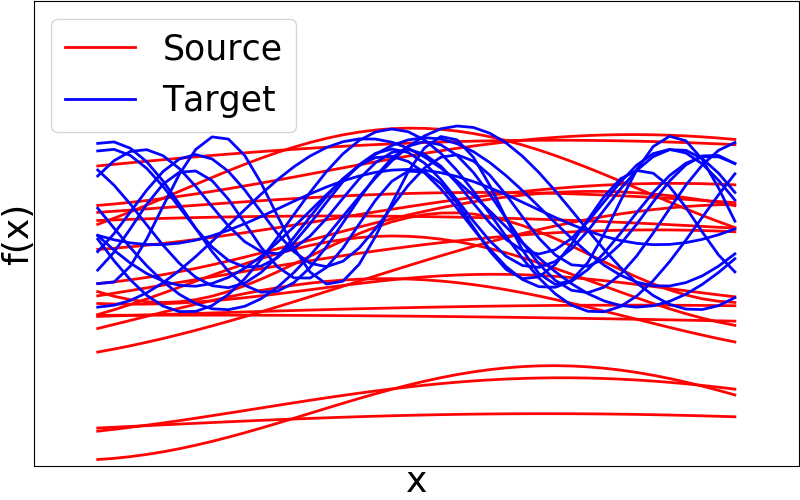}
  \caption{Data}
  
\end{subfigure}
\hfill
\begin{subfigure}[]{0.32\linewidth}
  \centering
  \includegraphics[width=1.0\linewidth]{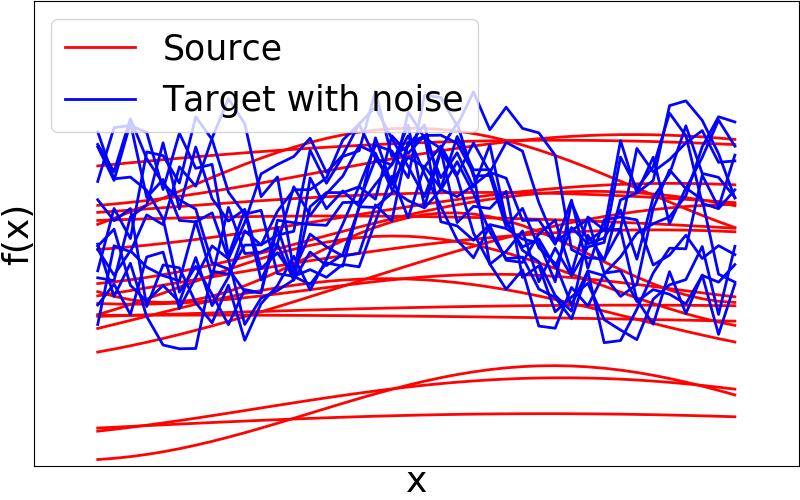}
  \caption{Data with noise}
  
\end{subfigure}
\hfill
\begin{subfigure}[]{0.32\linewidth}
  \centering
  \includegraphics[width=1.0\linewidth]{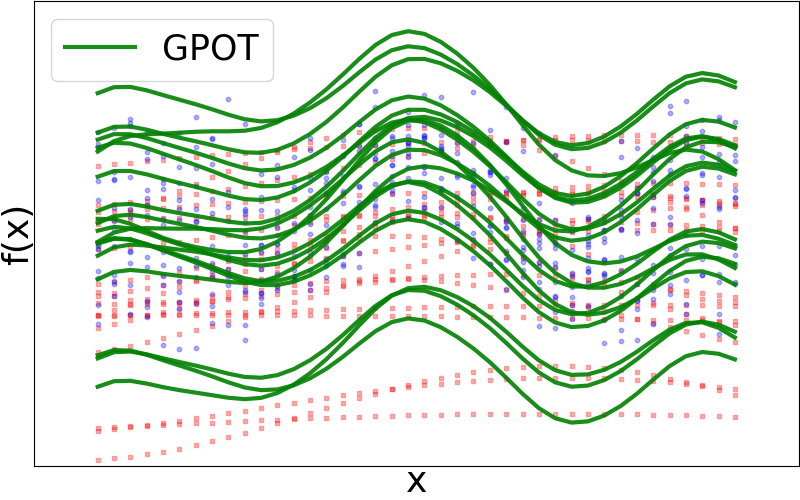}
  \caption{\citep{mallasto2017learning_wGP}}
  
\end{subfigure}
\hfill
\begin{subfigure}[]{0.32\linewidth}
  \centering
  \includegraphics[width=1.0\linewidth]{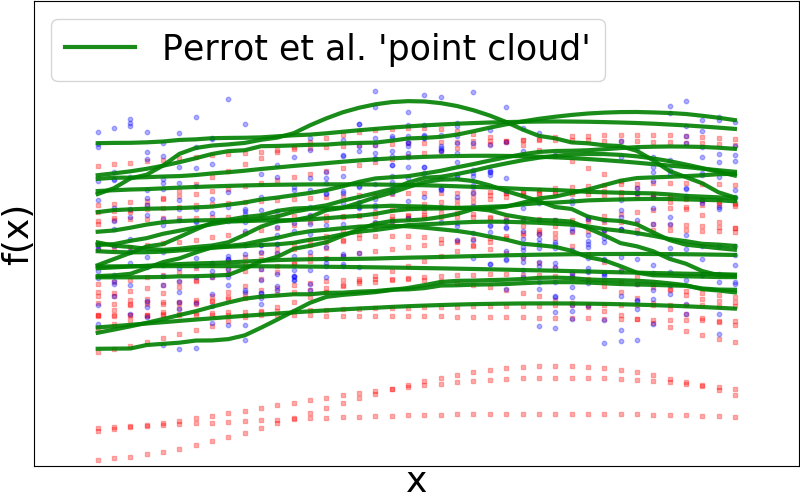}
  \caption{ \citep{perrot2016mapping_est_discre_OT}}
  
\end{subfigure}
\hfill
\begin{subfigure}[]{0.32\linewidth}
  \centering
  \includegraphics[width=1.0\linewidth]{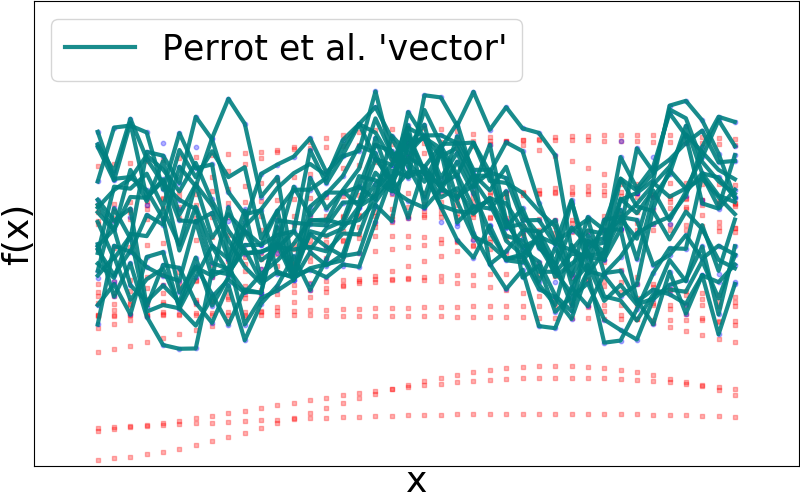}
  \caption{ \citep{perrot2016mapping_est_discre_OT}}
  
\end{subfigure}
\hfill
\begin{subfigure}[]{0.32\linewidth}
  \centering
  \includegraphics[width=1.0\linewidth]{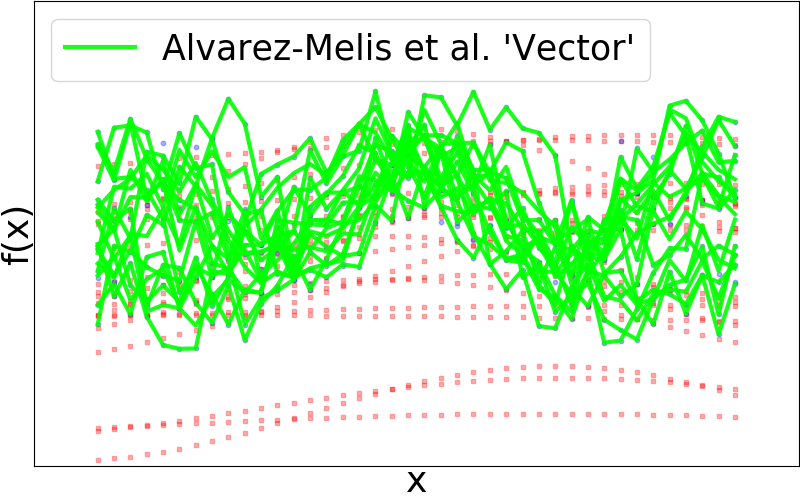}
  \caption{ \citep{alvarez2019towards_joint}}

\end{subfigure}
\hfill
\begin{subfigure}[]{0.32\linewidth}
  \centering
  \includegraphics[width=1.0\linewidth]{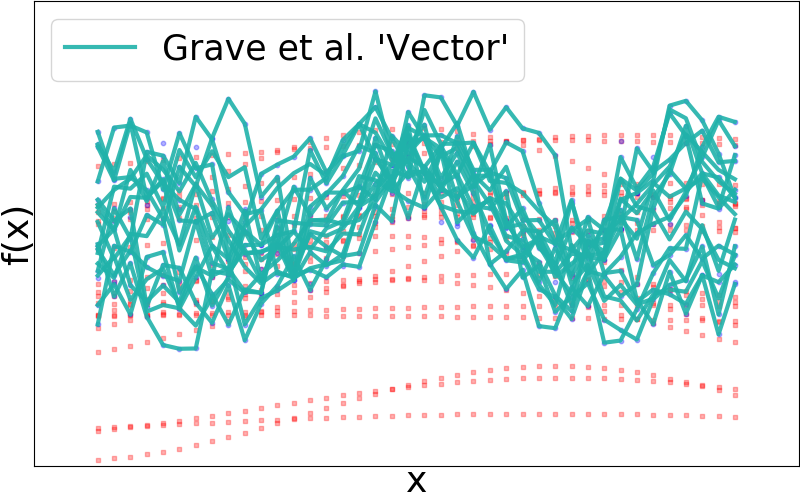}
  \caption{ \citep{grave2019unsupervised}}
  
\end{subfigure}
\hfill
\begin{subfigure}[]{0.32\linewidth}
  \centering
  \includegraphics[width=1.0\linewidth]{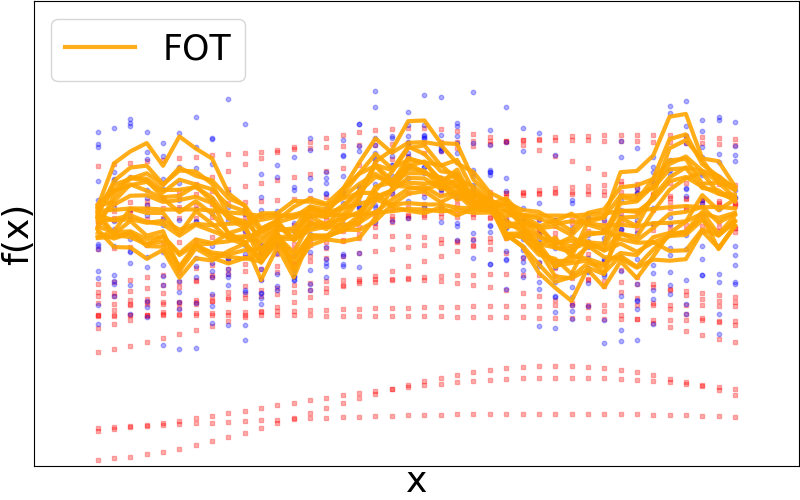}
  \caption{FOT with $\eta = 1$}
   
\end{subfigure}
\hfill
\begin{subfigure}[]{0.32\linewidth}
  \centering
  \includegraphics[width=1.0\linewidth]{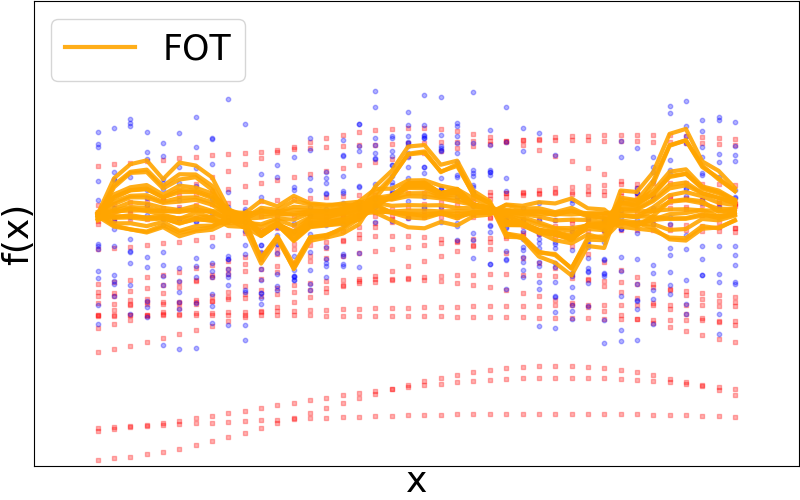}
  \caption{FOT with $\eta = 40$}

\end{subfigure}

\caption{Panel (b) depicts functional data generated by adding non-continuous noise to the smooth curves shown in panel (a). 
(c) depicts results obtained by applying the method of~\citep{mallasto2017learning_wGP}. (d) and (e) depict results obtained by applying the method of \citep{perrot2016mapping_est_discre_OT}. (f) depicts results obtained by applying the method of \citep{alvarez2019towards_joint}. (g) depicts results obtained by the method of \citep{grave2019unsupervised}. Panels (h) and (i) depict results obtained by FOT using different regularization parameters. }
\label{fig:nonfunctionalex}
\end{figure*}

\begin{figure*}[ht]
\begin{subfigure}[]{0.32\linewidth}
  \centering
  \includegraphics[width=1.0\linewidth]{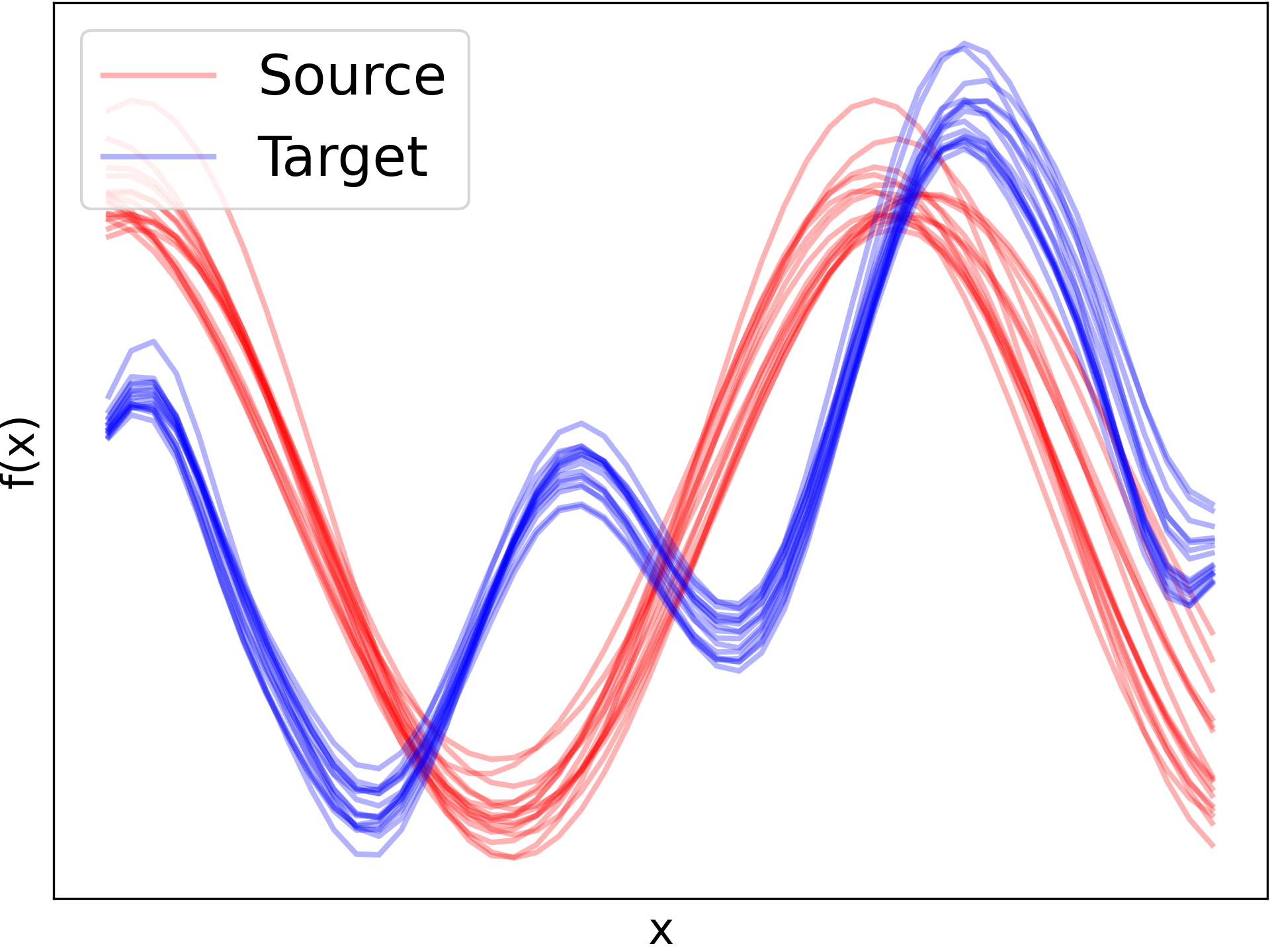}
  \caption{functional data}
  \label{fig:toy_exp_fpca_data}
\end{subfigure}
\hfill
\begin{subfigure}[]{0.32\linewidth}
  \centering
  \includegraphics[width=1.0\linewidth]{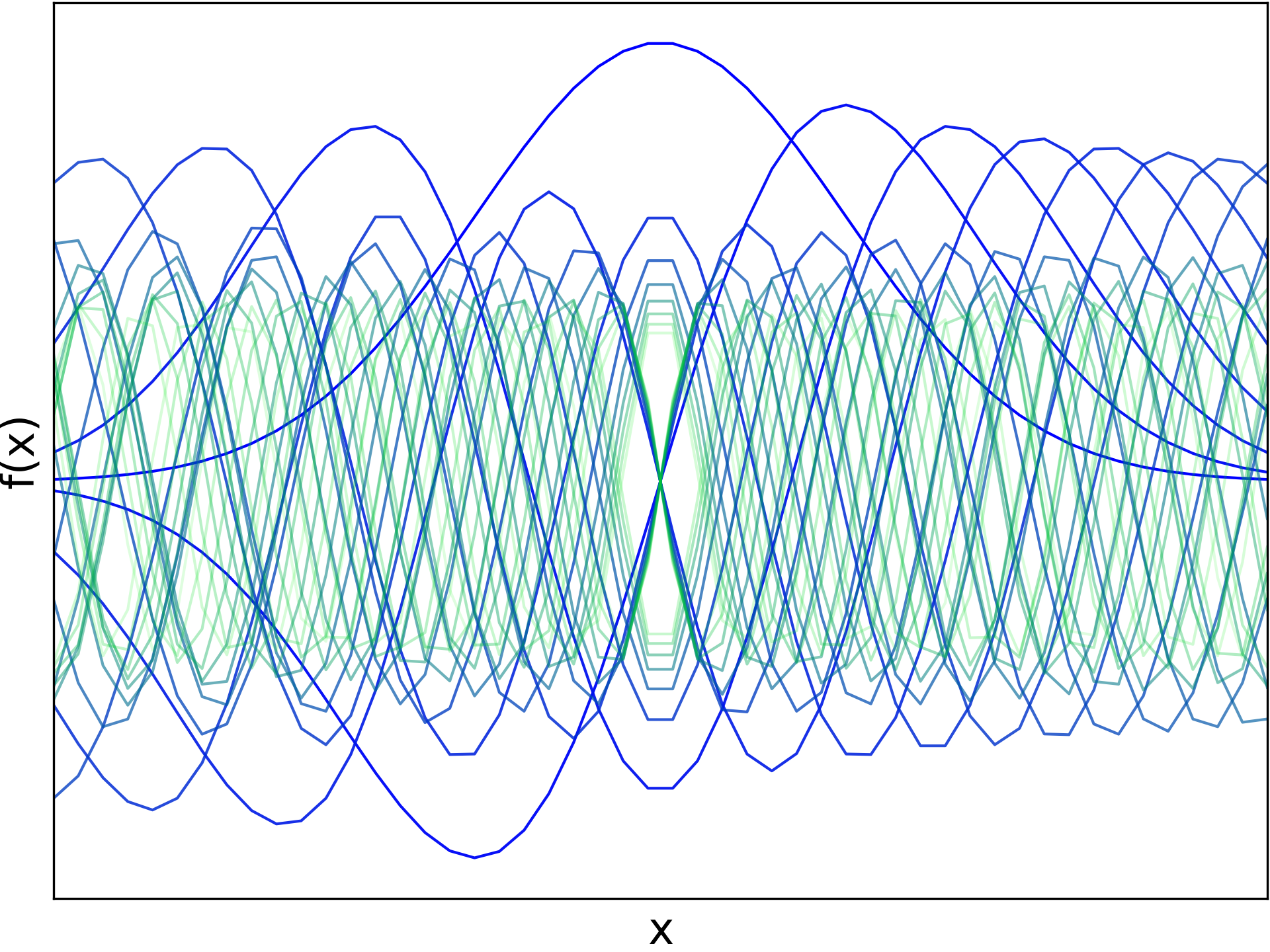}
  
  \caption{Basis functions for data generation}
  \label{fig:toy_exp_fpca_given_basis}
\end{subfigure}
\hfill
\begin{subfigure}[]{0.32\linewidth}
  \centering
  \includegraphics[width=1.0\linewidth]{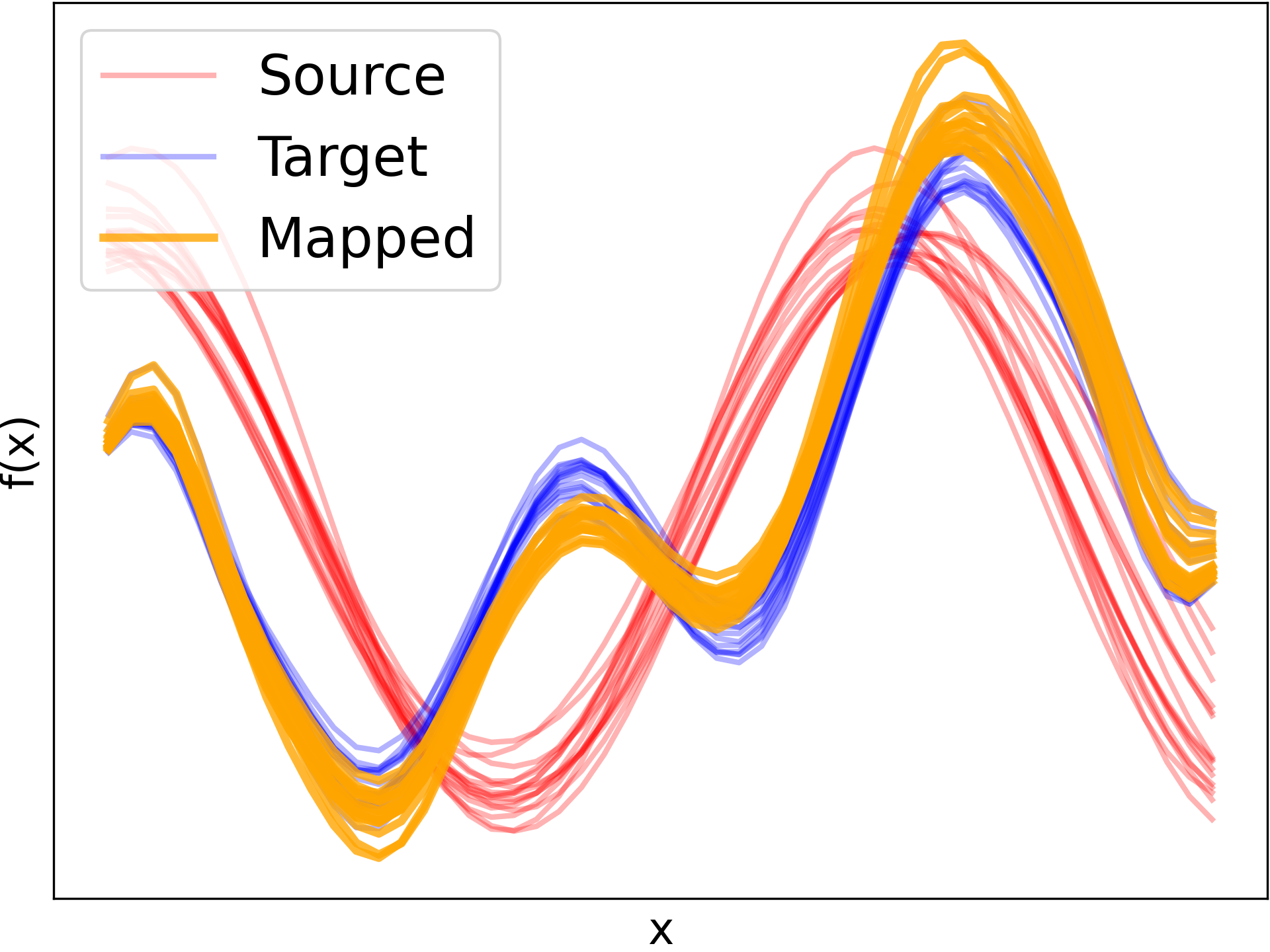}
  
  \caption{Result with given basis functions}
  \label{fig:toy_exp_fpca_pushforward_given_basis}
\end{subfigure}
\hfill
\begin{subfigure}[]{0.32\linewidth}
  \centering
  \includegraphics[width=1.0\linewidth]{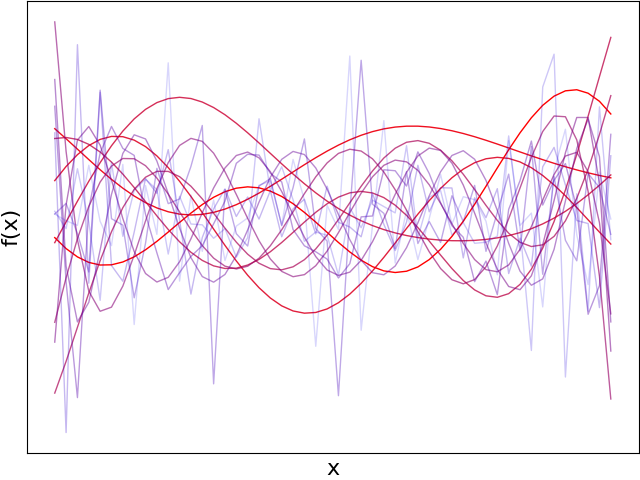}
  
  \caption{Learned basis functions $\hat{V}$}
  \label{fig:toy_exp_fpca_learned_basis_src}
\end{subfigure}
\hfill
\begin{subfigure}[]{0.32\linewidth}
  \centering
  \includegraphics[width=1.0\linewidth]{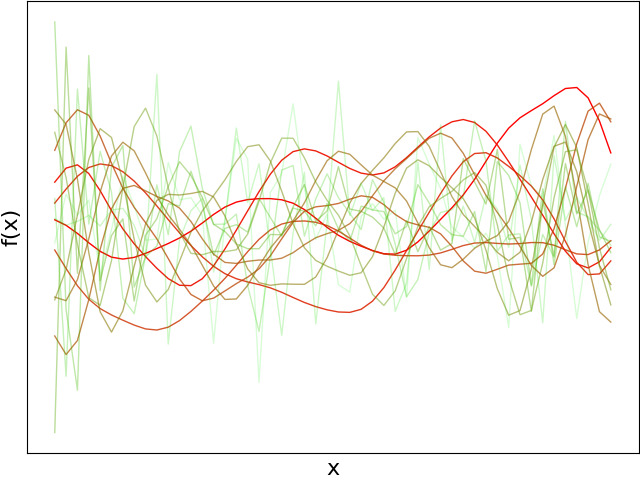}
  \caption{Learned basis functions $\hat{U}$}
  \label{fig:toy_exp_fpca_learned_basis_tgt}
\end{subfigure}
\hfill
\begin{subfigure}[]{0.32\linewidth}
  \centering
  \includegraphics[width=1.0\linewidth]{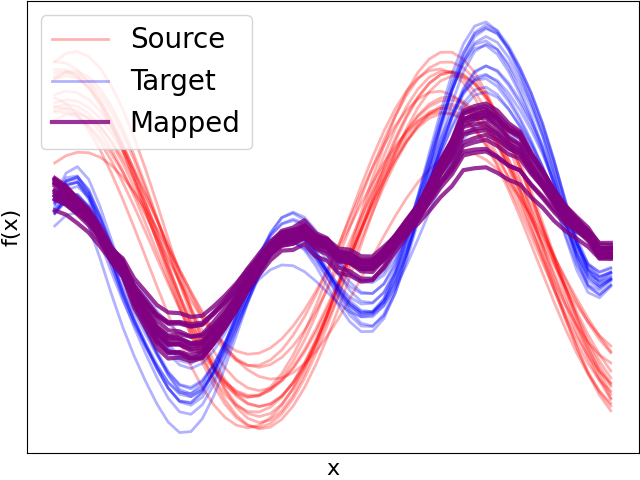}
  \caption{Result with estimated basis}
  \label{fig:toy_exp_fpca_pushforward_learned_basis}
\end{subfigure} 
\caption{{
Estimate basis functions from data using FPCA. When the ground-truth basis functions are provided, the estimated map can achieve perfect results (Fig. (\ref{fig:toy_exp_fpca_pushforward_given_basis})) with given basis (Fig. (\ref{fig:toy_exp_fpca_given_basis})). 
Using the estimated basis functions (Fig. (\ref{fig:toy_exp_fpca_learned_basis_src}) and Fig. (\ref{fig:toy_exp_fpca_learned_basis_tgt})), the map can still have satisfactory results (Fig. (\ref{fig:toy_exp_fpca_pushforward_learned_basis})). 
}}
\label{fig:toy_exp_fpca}
\end{figure*}

\begin{figure*}[ht]
\begin{subfigure}[]{0.32\linewidth}
  \centering
  \includegraphics[width=1.0\linewidth]{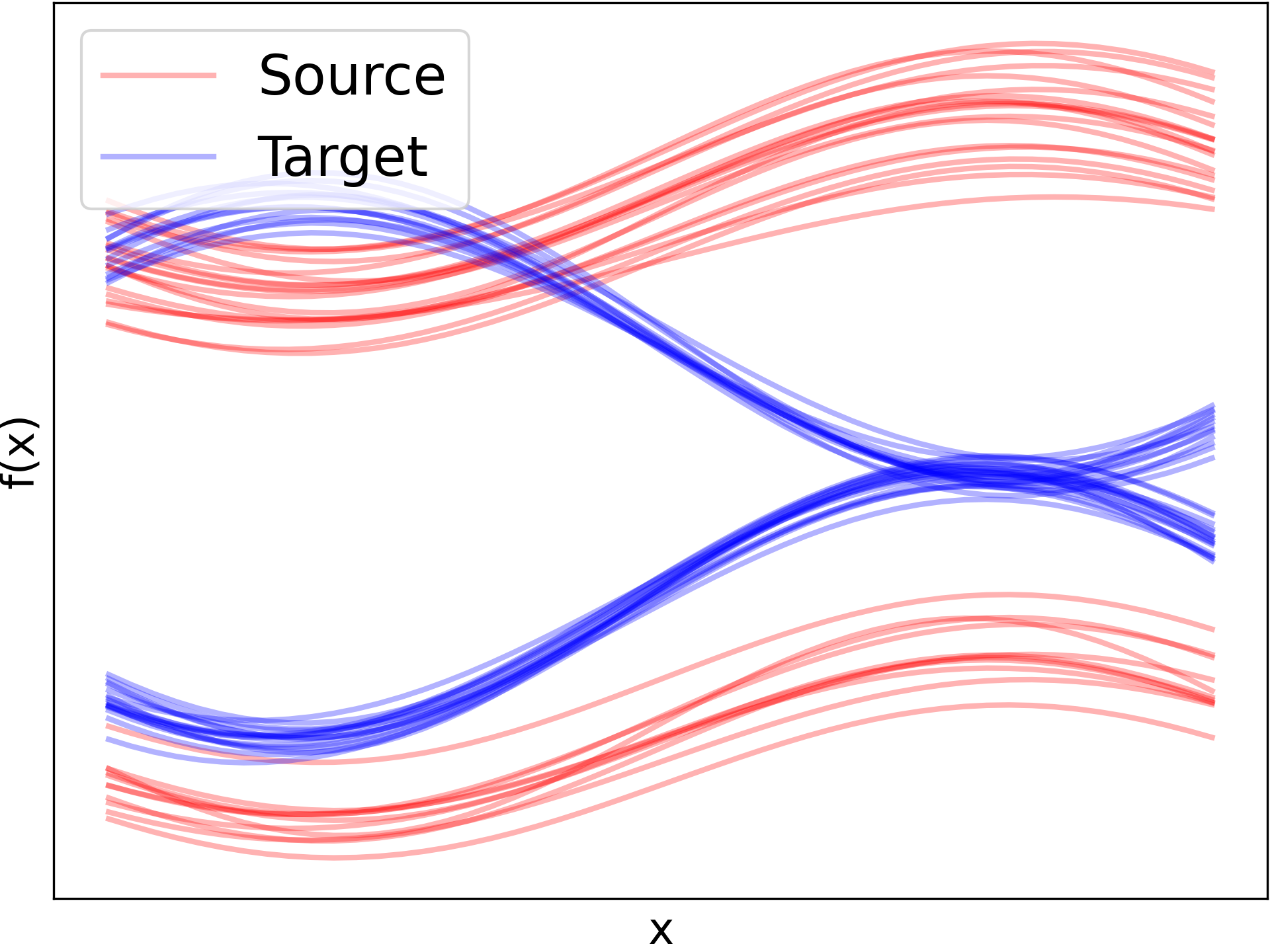}
  \caption{functional data}
  \label{fig:toy_exp_fpca_data_MM}
\end{subfigure}
\hfill
\begin{subfigure}[]{0.32\linewidth}
  \centering
  \includegraphics[width=1.0\linewidth]{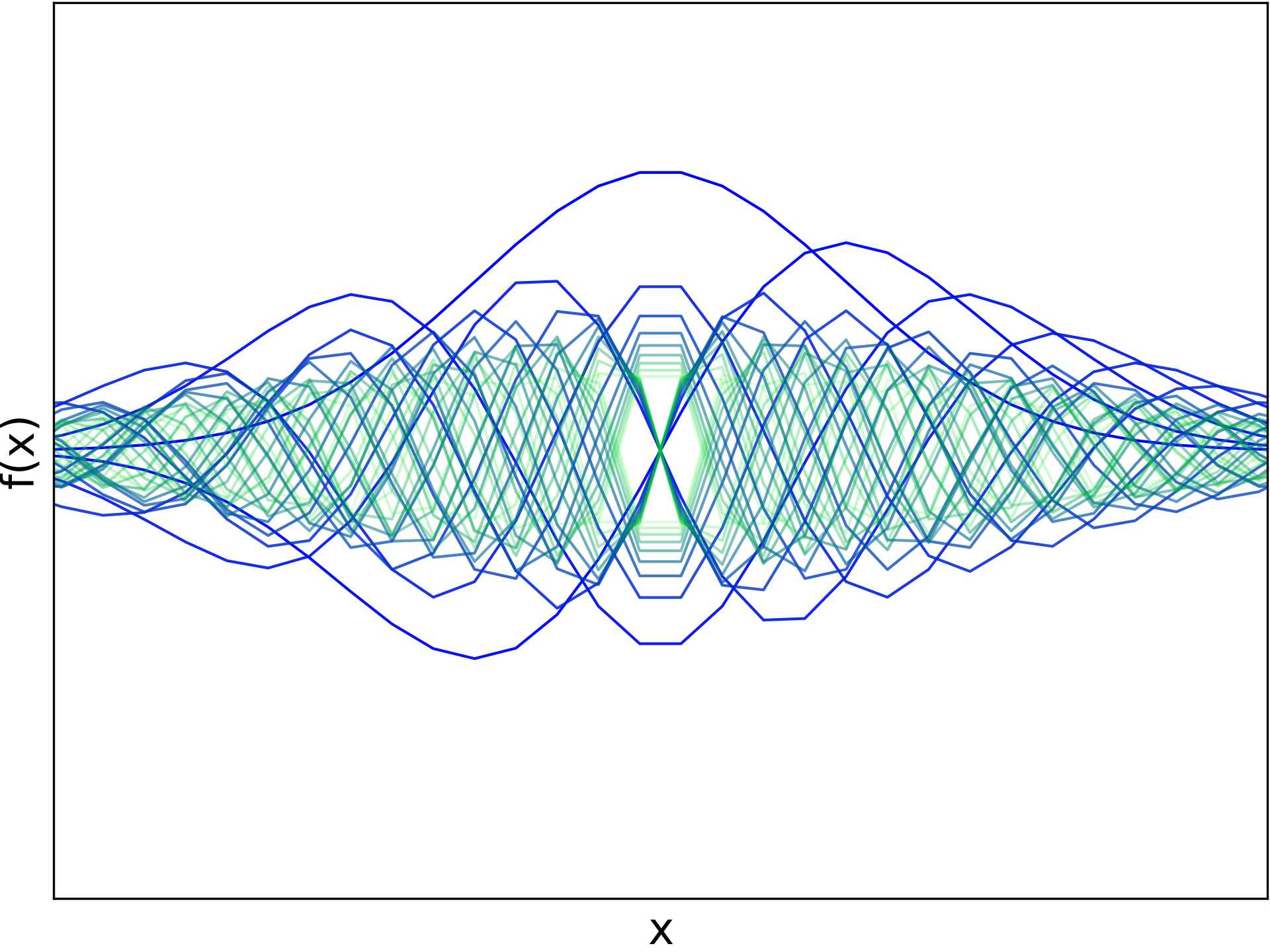}
  
  \caption{Basis functions for data generation}
  \label{fig:toy_exp_fpca_given_basis_MM}
\end{subfigure}
\hfill
\begin{subfigure}[]{0.32\linewidth}
  \centering
  \includegraphics[width=1.0\linewidth]{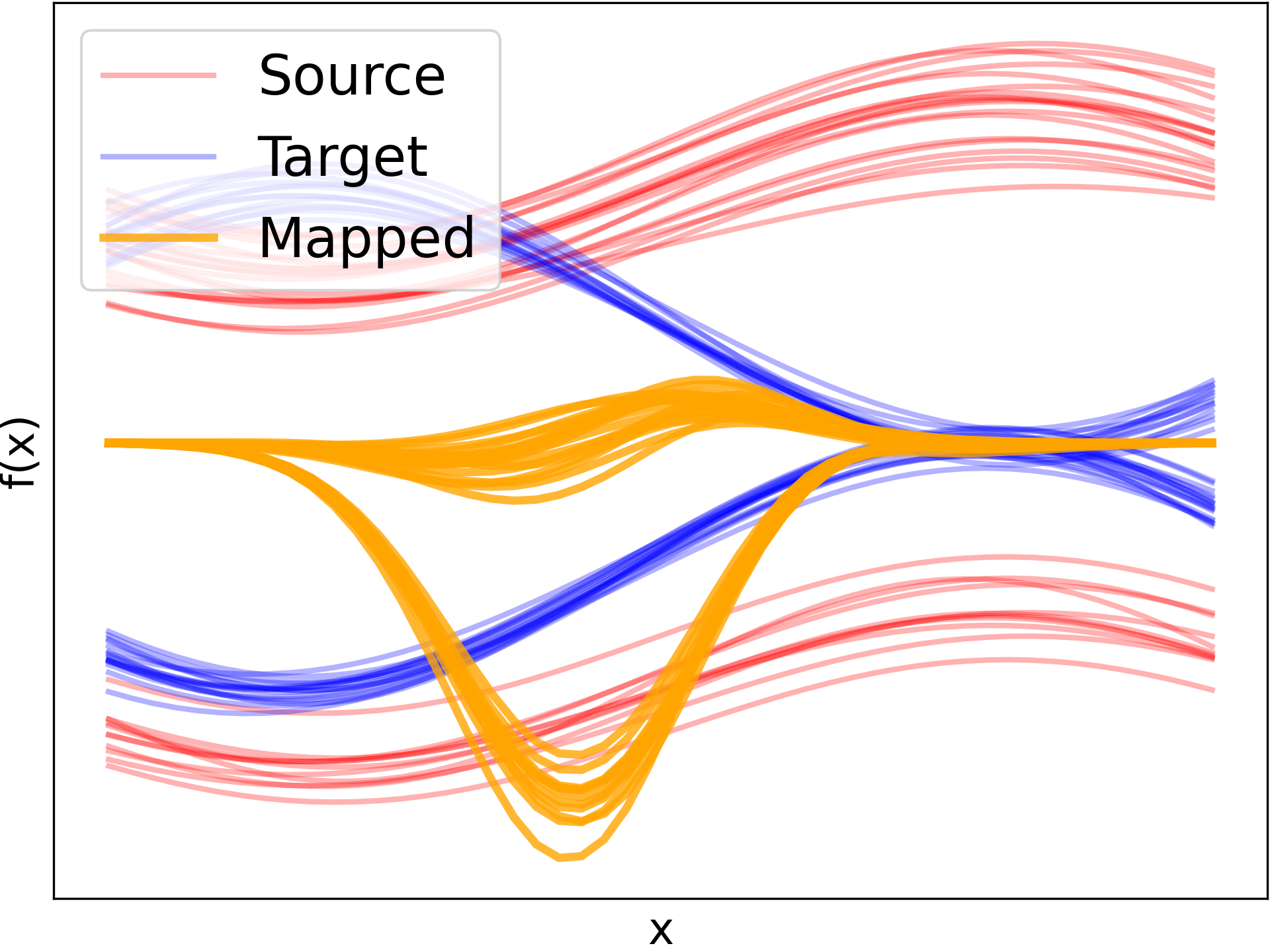}
  
  \caption{Result with misspecified basis functions
  }
  \label{fig:toy_exp_fpca_pushforward_given_basis_MM}
\end{subfigure}
\hfill
\begin{subfigure}[]{0.32\linewidth}
  \centering
  \includegraphics[width=1.0\linewidth]{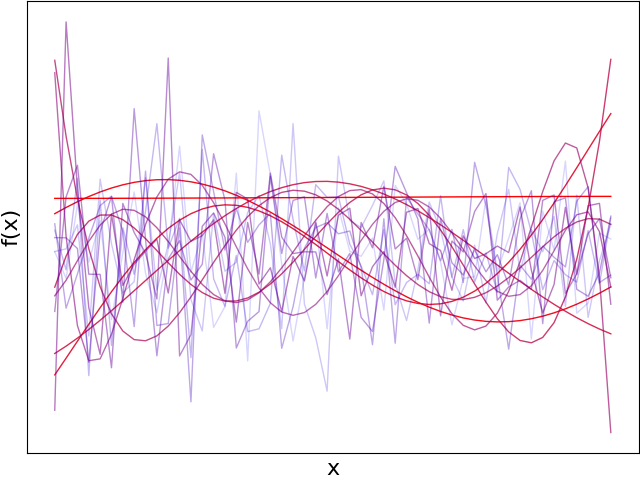}
  
  \caption{Learned basis $\hat{V}$}
  \label{fig:toy_exp_fpca_learned_basis_src_MM}
\end{subfigure}
\hfill
\begin{subfigure}[]{0.32\linewidth}
  \centering
  \includegraphics[width=1.0\linewidth]{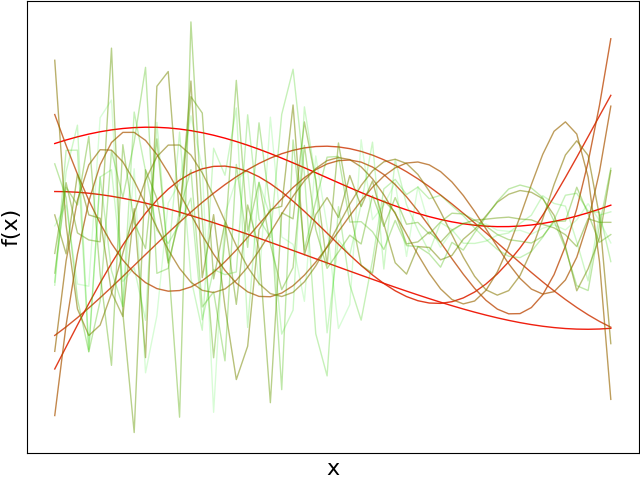}
  \caption{Learned basis $\hat{U}$}
  \label{fig:toy_exp_fpca_learned_basis_tgt_MM}
\end{subfigure}
\hfill
\begin{subfigure}[]{0.32\linewidth}
  \centering
  \includegraphics[width=1.0\linewidth]{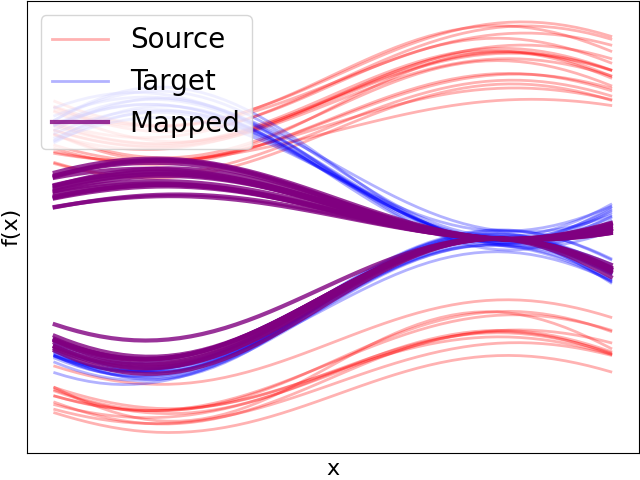}
  \caption{Result with learned basis}
  \label{fig:toy_exp_fpca_pushforward_learned_basis_MM}
\end{subfigure} 
\caption{
{
Defining appropriate basis functions (Fig. (\ref{fig:toy_exp_fpca_given_basis_MM})) for optimal map estimations itself is a challenging task. 
Employing FPCA (Fig. (\ref{fig:toy_exp_fpca_learned_basis_src_MM}) and Fig. (\ref{fig:toy_exp_fpca_learned_basis_tgt_MM})) yields superior map estimation outcomes (Fig. (\ref{fig:toy_exp_fpca_pushforward_learned_basis_MM})). In contrast, misspecified basis functions from randomly selected basis function coefficients result in sub-optimal map pushforward samples (Fig. (\ref{fig:toy_exp_fpca_pushforward_given_basis_MM})).
}}
\label{fig:toy_exp_fpca_MM}
\end{figure*}

\noindent \textbf{Continuity/ Multimodality preserving properties:} As shown in Fig. (\ref{fig:out-of-sample}), the map learned by FOT does a good job at pushing forward out-of-sample curves that were not observed during training. Moreover, the coupling $\pi$ reveals the multi-modality in the data: as the source distribution is unimodal but the target distribution is bimodal, there is a "splitting" behavior in how the sampled curves from the source are distributed and transported to the target. Finally, Fig. (\ref{fig:continuous-index}) shows that FOT is very effective for functional data evaluated at different design points. On the upper right panel of Fig. (~\ref{fig:out-of-sample}), the estimated integral operator $\hat{\mathbf{T}}_K$ is shown; note that it is close to an identity matrix while having some permutations around the first elements. We show the estimated coupling $\pi$ on the lower right panel. The coupling clearly reveals the underlying cluster structure in the target function data.

\noindent\textbf{Comparison to OT methods for finite-dimensional vectors:} Although one can always apply existing OT map estimation methods such as that of \cite{alvarez2019towards_joint,perrot2016mapping_est_discre_OT, grave2019unsupervised} to functional data by simply discretizing continuous functions into fixed-dimension vector measurements, we shall demonstrate the ineffectiveness of such an approach due to its failure to properly account for the functional nature (e.g. smoothness) of the data in the source and/or target domain (see Fig.~\ref{fig:nonfunctionalex}). In particular, we present experiments and comparisons with more baseline methods under the same settings considered in Section \ref{section_simulation}.  In these experiments we assume all the functional data are evaluated on a set of fixed-size design points to apply conventional OT map estimation methods for fixed-dimensional vectors directly. In addition, the observed continuous function data is perturbed by non-continuous noise.
Under this setting, all baseline OT formulations neglect the smooth nature of functional data and overfit the signals contaminated with noises.
Only the pushforward of maps estimated with GPOT \citep{mallasto2017learning_wGP} and our methods 
successfully recover the smoothness of the target curves. 
This suggests the necessity of treating data as sampled functions (rather than sampled vectors). 
Plot (h) and plot (i) of Fig.~\ref{fig:nonfunctionalex} show the role played by parameter $\eta$ in controlling the smoothness of the map. 
 

\begin{figure*}[t]
\subfloat[Entropic regularization]{%
\label{fig:reg_a}
\includegraphics[height=0.132\textheight]{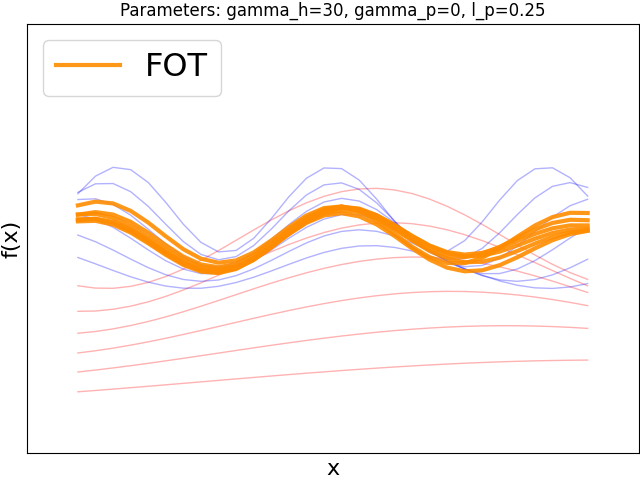}
\includegraphics[height=0.132\textheight]{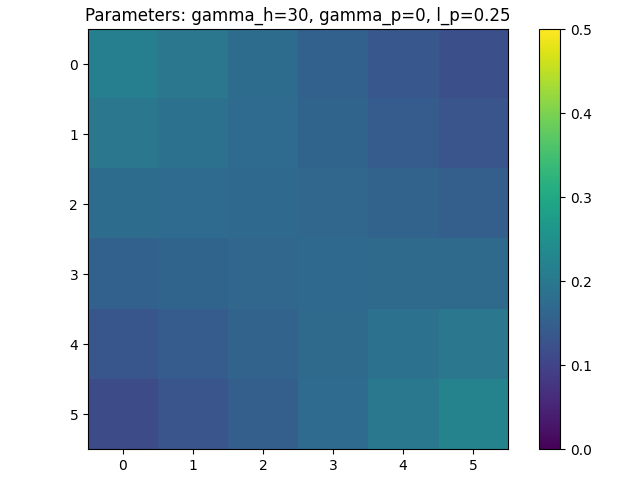}
}
\hspace*{\fill}
\subfloat[Power regularization $\gamma_p = 15$]{%
\label{fig:reg_b}
\includegraphics[height=0.132\textheight]{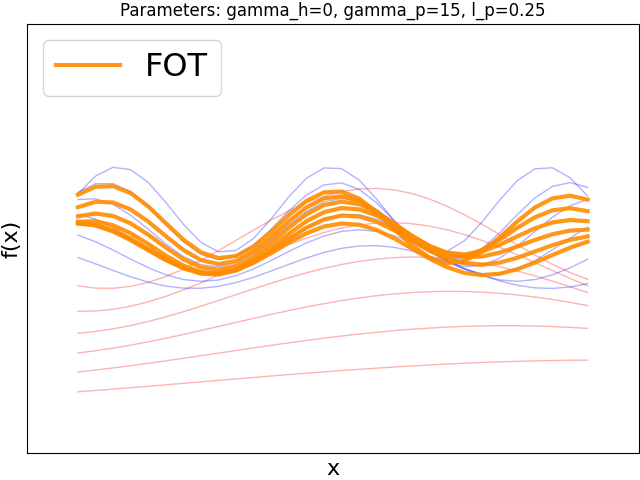}
\includegraphics[height=0.132\textheight]{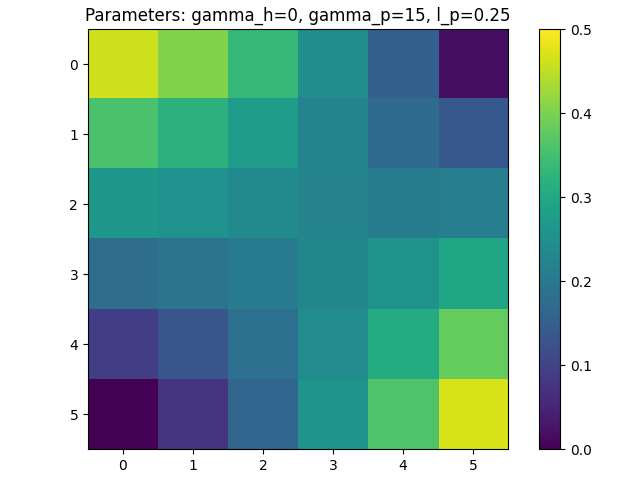}
}

\subfloat[$\gamma_p = 0$]{%
\label{fig:reg_p_0}
\includegraphics[height=0.132\textheight]{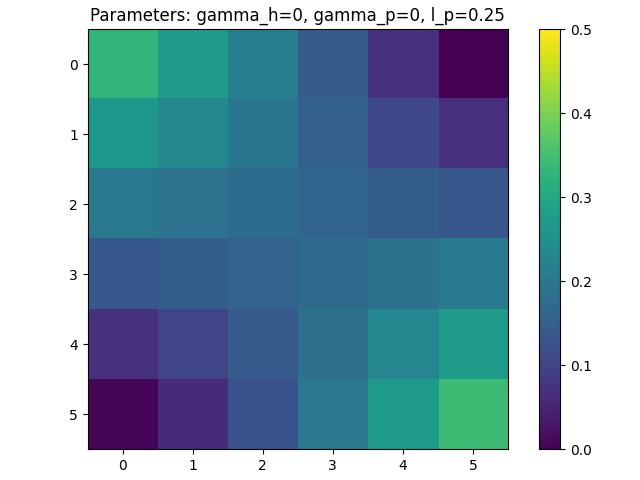}}
\hspace*{\fill}
\subfloat[$\gamma_p = 4$]{%
\label{fig:reg_p_4}
\includegraphics[height=0.132\textheight]{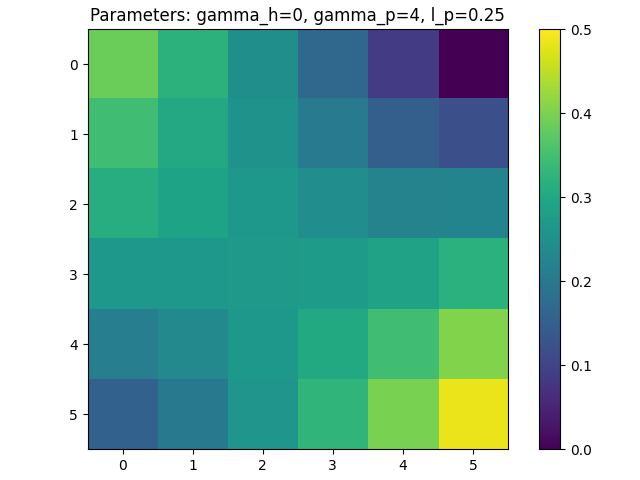}}
\hspace*{\fill}
\subfloat[$\gamma_p = 10$]{%
\label{fig:reg_p_10}
\includegraphics[height=0.132\textheight]{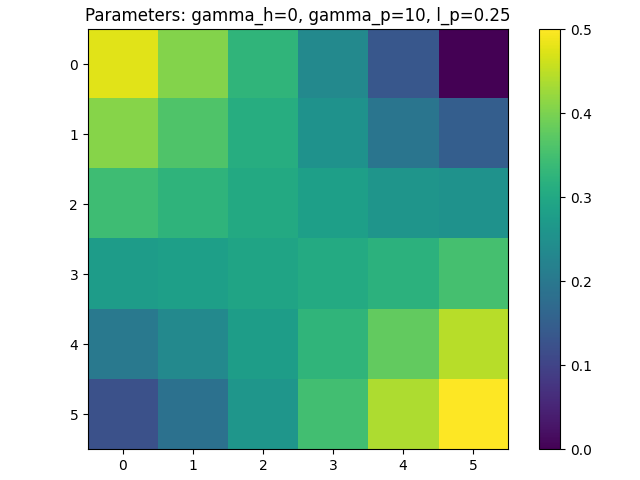}}
\hspace*{\fill}
\subfloat[entropy]{%
\label{fig:reg_p_vs_e}
\includegraphics[height=0.132\textheight]{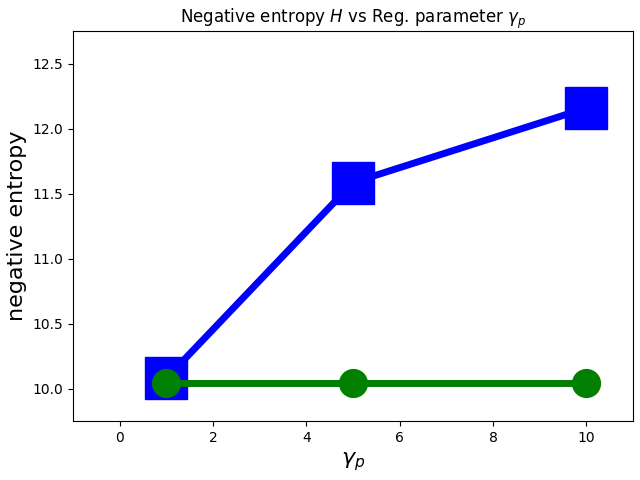}}
\caption{Panel (a) and (b) depict the pushforward samples and coupling matrix obtained by entropic and power regularization, respectively.  The entropic regularization leads to a smooth coupling while the corresponding pushforward samples concentrate near the average of the target samples, suggesting a rather poor transport map. Panels (b)--(f) show that the coupling becomes sparser as we increase the power regularization coefficient $\gamma_p$. 
Panel (f) confirms that the negative entropy of the coupling matrix increases with coefficient $\gamma_p$.} 
\label{fig:reg}
\end{figure*}

\paragraph{Selection of basis functions:}
We investigate the selection of basis functions using the FPCA approach introduced in Section \ref{sec:opt}. We generate the original data using Hermite polynomials as basis functions. 
As illustrated in Fig. \ref{fig:toy_exp_fpca}, knowing the ground-truth basis function when estimating the transport map leads to a consistent pushforward map. 
We also display the estimated basis functions for both the source (Fig. \ref{fig:toy_exp_fpca_learned_basis_src}) and target (Fig. (\ref{fig:toy_exp_fpca_learned_basis_tgt})) sample curves. 
{ Following the FPCA procedure described in Sec. \ref{sec:opt}, we assembled the function values $y_i(x)$ into matrix $\mathbf{Y}$, with each row representing a function sample as a vector. Leveraging the shared index set, we applied Singular Value Decomposition (SVD) to $\frac{1}{N} \mathbf{Y}$, yielding $\mathbf{U} \mathbf{S} \mathbf{V}^\top$ where the columns of $\mathbf{U}$ are the estimation of basis functions.}
The estimated basis functions are ordered by the degree of oscillation. As we can see, the estimated pushforward map does not match the pushforward curves perfectly but is still satisfactory, as illustrated in Fig. (\ref{fig:toy_exp_fpca_pushforward_learned_basis}). 

In numerous situations, the basis functions for complex multimodal distributions may not be readily available; estimating a highly complex parametrized family of basis functions effectively can be challenging. Under these conditions, employing a data-driven approach such as FPCA for obtaining the most relevant basis functions can yield improved performance. 
The flexibility and usefulness of the FPCA approach are illustrated in  Fig. \ref{fig:toy_exp_fpca_MM} which demonstrates that our algorithm augmented with the FPCA-based basis functions can effectively adapt to previously unobserved data distributions.

\noindent\textbf{Effects of regularization:} 
The final set of simulations is designed to evaluate the effects of regularization terms in the FOT formulation. The results of this study are depicted in Fig. \ref{fig:reg}.
We observe that the power regularizer finds sparser coupling distributions than entropic regularization. This phenomenon is expected as the entropic penalty keeps the coupling strictly positive in its support. The lack of sparsity can be problematic, especially in the case of iterative map estimation, where a sparse coupling distribution is crucial for learning a meaningful and expressive transport map between domains of functions. 
It is worth noting that, the estimated coupling represents the joint probability density matrix and it reveals the clustering structure of data (cf. Fig. (\ref{fig:out-of-sample})). 

\subsection{Optimal transport domain adaptation for robot arm's multivariate sequences of motion}

Recent advances in robotics include many novel data-driven approaches such as motion prediction \citep{jetchev2009trajectory_icml}, human-robot interaction \citep{liu2018serocs_HRI}, and others \citep{tompkins2020online_slam_ot, xu2020task_agnositc_DPGP}.
However, generalizing knowledge across different automated tasks for a robot, and generalizing across robots, are considered challenging since data collection in the real world is expensive and time-consuming. 
A variety of approaches have been developed to tackle these problems, such as domain adaptation \citep{bousmalis2018_robotarm_domain_adaptation}, transfer learning \citep{weiss2016survey_transferlearning}, and so on \citep{tobin2017robotarm_domainrandom, finn2017model_MAML}.

\textbf{Optimal transport based domain adaptation:} 
{ We propose to eliminate the heterogeneity in robot learning datasets by following the formulation of}
optimal transport based domain adaptation (OTDA) \citep{courty2016optimal_OTDA}. 
{ Specifically, }
the pipeline consists of the following three steps: 1) learn an optimal transport map, 2) apply the pushforward map on the observed source samples towards the target domain, and 3) train a motion predictor on the pushforwarded samples that lie in the target domain. 

Although it might be possible to discretize and interpolate data to fixed-size vectors of observed measurements, trajectories of robot motion are intrinsically \textit{continuous functions of time of various lengths}. Functional optimal transport provides a natural solution for this challenging task over existing OT map estimation methods for discrete samples.

{
Typically, a robot motion dataset $\{f_i\}_{i=0}^N$ contains multiple trails $f_i$ for a specific task. Although these trials are somewhat similar, they differ slightly due to various real-world factors due to sensor and actuator noises, and human intervention. 
Each trail can be viewed as a multidimensional $n_t$ length timeseries representing the joint or end effector location over time $f_i:= \{(f_{i,1}, t_{i,1}), ..., (f_{i,n_t}, t_{i,n_t}) \}$. Each robot motion dataset is viewed as samples for a distribution $\mu$, via the empirical distribution $\hat{\mu}=\sum_{i=1}^{N} p_{i} \delta_{\hat{f_i}}$, where $\delta_{\hat{f_i}}$ is the Dirac measure at $\hat{f_i}$, which is the embedded function of $f_i$ in the Hilbert space using basis functions, and $p_{i}$ are probability masses associated to the $i$-th sample ($\sum_{i} p_{i}=1$). 
Given a source and target robot-motion dataset $D_s$ and $D_t$ (with empirical measures $\hat{\mu}_{s}$ and $\hat{\mu}_{t}$, respectively), we assume that there are sufficient samples in the source $\{ f_{s,i}\}_{i=1}^{N_s}$ but \textit{only limited samples} in the target $\{ f_{t,j}\}_{j=1}^{N_t}$, where $N_s \gg N_t$. To obtain an ML model in the target domain, we want to leverage the knowledge in the source by using a transport map that pushes forward all samples $\{f_i\}$ to match the distributions. At this point, the transport map can be estimated using the functional optimal transport formulation, by solving $T^* = \argmin_T W_2 (T_\# \hat{\mu}_s, \hat{\mu}_t ) + \eta \| T \|^2_F$.
}





\begin{figure}[h]\centering

\includegraphics[width=0.9\linewidth]{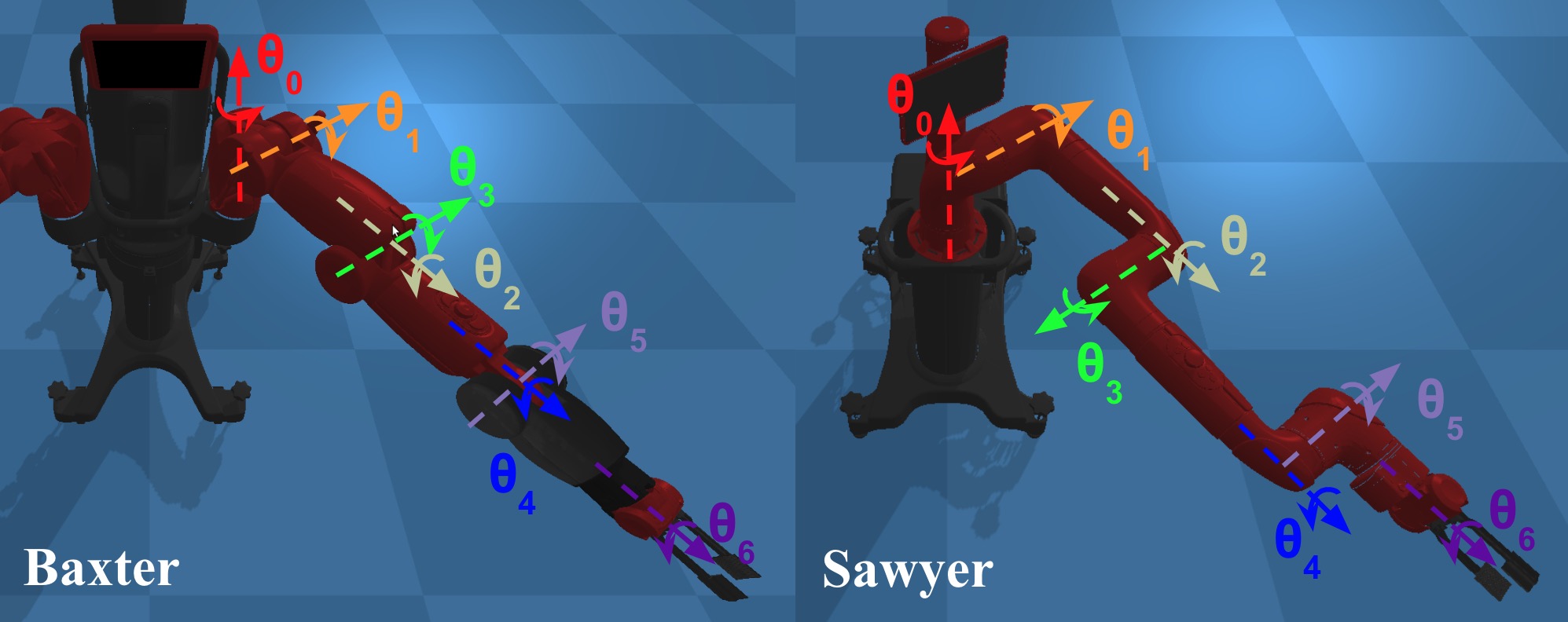}
\caption{The structure of the Baxter robot and the Sawyer robot used in MIME dataset and Roboturk dataset. Their arms share a similar structure as they both have 7 joints and one end effector. This allows us to perform domain adaptation between these two datasets.
}\label{fig:robot_arm_compare}
\end{figure}

\begin{figure}[h]\centering
\begin{subfigure}{0.99\textwidth}
        \centering
        \includegraphics[width=0.95\linewidth]{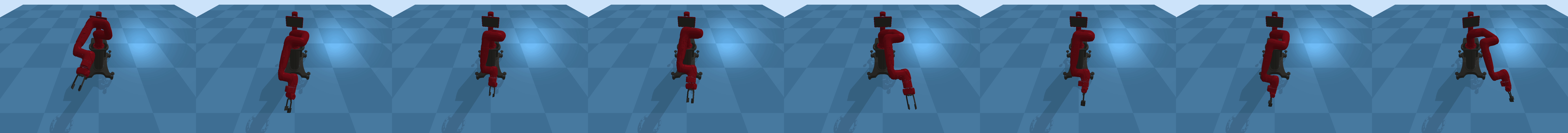}%
        
        \includegraphics[width=0.99\linewidth]{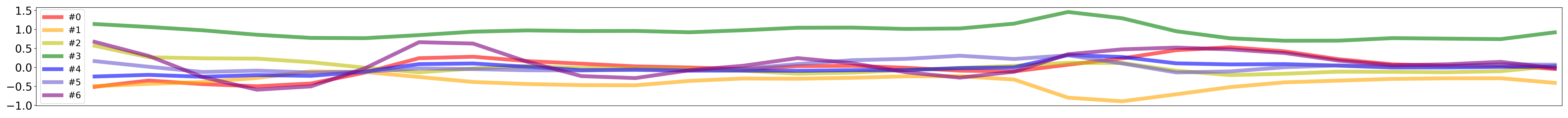}
        \caption{Source motion:  "Roboturk-bins-Bread"  by Sawyer robot.}
        \label{fig:robotarm_source}
        
        \includegraphics[width=0.95\linewidth]{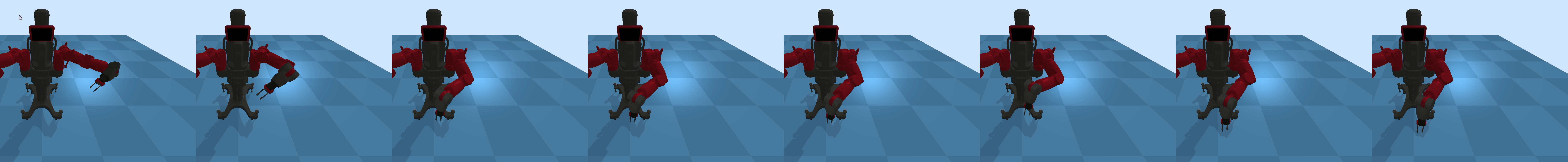}%
        
        \includegraphics[width=0.99\linewidth]{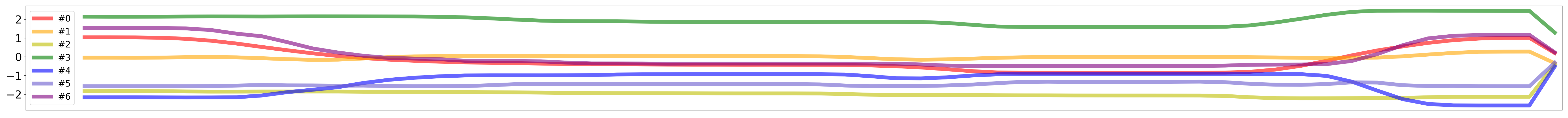}
        \caption{Target motion:  "MIME Picking (left-hand)"  by Baxter robot.}
        \label{fig:robotarm_target}
        
        \includegraphics[width=0.95\linewidth]{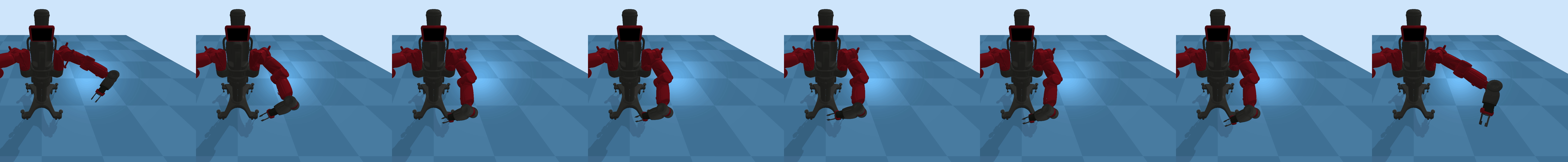}%
        
        \includegraphics[width=0.99\linewidth]{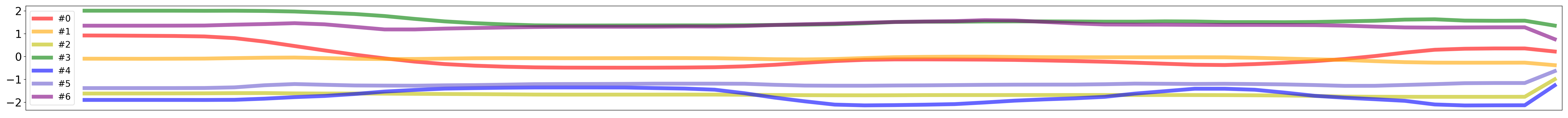}
        \captionsetup{justification=centering}
        \caption{The pushforward motion of the transport map and  the target motion \newline look similar to each other but differ slightly. }
        \label{fig:robotarm_mapped}
        
    \end{subfigure}
\caption{Pushforward of robot arm motions: In each sub-figure, a robot-arm motion is visualized as image clips, while the robot-arm joint angles are plotted as a multivariate time series, the x-axis is time and are omitted for cleanliness.}
\label{fig_robot}
\end{figure}


\textbf{Datasets:} The \textbf{MIME} Dataset \citep{sharma2018_mime_dataset} contains 8000+ motions across 20 tasks 
collected on a two-armed Baxter robot. 
The \textbf{Roboturk} Dataset \citep{mandlekar2018roboturk} is collected by a Sawyer robot over 111 hours.
As shown in Fig. (\ref{fig:robot_arm_compare}), both robot arms have 7 joints with similar but slightly different configurations, which enable us to learn domain adaptation between the two.
We picked two tasks, \textit{Pouring (left arm) and Picking (left arm)},
from MIME dataset and two tasks, \textit{(bins-Bread, pegs-RoundNut)},
from Roboturk dataset.
We considered each task as an individual domain. 

\textbf{Pushforward of robot motions:}  Our method successfully learns the transport map that pushes forward samples from one task domain to another. The source dataset contains motion records
from task \textit{bins-full} in the Roboturk dataset while the target includes motion records from task \textit{Pour (left-arm)} in the MIME dataset. We visualize the motion by displaying the robot joint angles sequences in a physics-based robot simulation gym \citep{erickson2020assistivegym}. Animated motions can be found here\footnote{More examples can be found here: \indent\href{https://sites.google.com/view/functional-optimal-transport}{https://sites.google.com/view/functional-optimal-transport}.}. In Fig. \ref{fig_robot}, we show image clips of each move along with a plot of time series of joint angles.
We can see from the robot simulation that the pushforward sequence in Fig. \ref{fig:robotarm_mapped} matches with the target motion in Fig. \ref{fig:robotarm_target} while simultaneously preserving certain features of the source motion in Fig. \ref{fig:robotarm_source}.

\begin{figure}[t] 
\resizebox{\textwidth}{!}{
\begin{tabular}{ p{1.3cm}|p{1.5cm}|p{1.5cm}|p{1.5cm}|p{1.5cm} | p{1.5cm} |p{1.5cm}|p{1.5cm}|p{1.5cm}|p{1.5cm} | p{1.5cm} }
 \hline
 Method  &  LSTM & ANP & RANP  &  MAML\textsuperscript{*} &  TL\textsuperscript{*} &  FOT\textsubscript{LSTM} &   FOT\textsubscript{ANP} & FOT\textsubscript{RANP} &  FOT\textsubscript{MAML} &  FOT\textsubscript{TL}  \\ 
 \hline
 \hline
 R1$\to$M1     & {2.0217} &  1.3261  &  1.9874 &  0.0307 & 0.5743 & 0.0271 & 0.0963 & 0.0687 & 0.0165 & 0.0277  \\
 R1$\to$M2    & 1.6821 &  1.0951  &  1.5681 &  0.0374 & 0.7083  & 0.0414 & 0.1642 & 0.1331 & 0.0191 & 0.0446    \\
 R2$\to$M1    & 1.3963 &  0.6642  &  1.7256 &  0.0327 & 0.2491  & 0.0277 & 0.0951 & 0.0696 & 0.0202 & 0.0906   \\
 R2$\to$M2    & 1.1952 &  0.6307  &  1.3659 &  0.0477 & 0.4020  & 0.0331 & 0.1620 & 0.1554 & 0.0167 & 0.0406  \\
 
 \hline 
\end{tabular}
}
\captionof{table}{MSE error results of different predictive models. R1: Roboturk-bins-bread, R2: Roboturk-pegs-RoundNut, M1:MIME1-Pour-left, M2: MIME12-Picking-left. }
\label{table:robot_arm_DA_results}
\end{figure}

{\textbf{Motion prediction}}: 
{
For the \textit{Robot Arm Motion Prediction} task, a motion trajectory $f_i = (f_{i,1}, t_{i,1}),...,(f_{i,l}, t_{i,l})$ of length $l$ consists of a set of vectors $f_{i,j} \in \mathbb{R}^d$ with associated timestamps $t_{i,j}$,
where the time series trajectories are governed by continuous functions of time $f_S(t) : t \in \mathbb{R} \mapsto S \in \mathbb{R}^d$.  
Since the task is to predict the future $l_f$ points based on the past $l_p$ points, ignoring the index of individual trajectories, we arrange the data to have the format 
$
X_{t} = \{ (f_{t+1}, t + 1), ..., (f_{t+l_p}, t + l_p) \}
$,
$
Y_{t} = \{ (f_{t+l_p + 1}, t + l_p + 1), ..., (f_{t+l_p+l_f}, t + l_p + l_f) \}
$.}
Our task is learning a predictive model that minimizes the squared prediction error in the target domain $\label{transfer_learning}
\arg \min_\theta  \sum^M_{i=1}(F_{\theta}(X^t_i ) - Y^t_i)^2$
where $Y^t_i$ is the true label from target domain and $ \hat{Y}^t_i = F_{\theta}(X^t_i)$ is the predictive label estimated by a model trained on source domain $(X^s, Y^s)$ and a subset of target domain $(X^{tm}, Y^{tm})$. 
{ It is worth noting that if the distribution of the testing set differs from that of the target set, a predictive model trained solely on the training set will experience a significant performance drop on the target set. To address this issue, several paradigms have been proposed, including transfer learning, meta-learning, and few-shot learning.} 

{
\textbf{Methods}: 
We considered 5 baselines for this task, including 
\begin{itemize}
\item [(1)] a conventional baseline approach we employ is a vanilla LSTM that is trained exclusively on the source data samples. Despite the fact that LSTM models are recognized for their ability to capture temporal dependencies, their performance may still be restricted due to the distributional discrepancy;
\item [(2)] the Attentive Neural Process (ANP) \citep{kim2019_ANP}, which is a deep Bayesian model that learns a distribution of functions. ANP can be seen as models that do few-shot learning since it looks for predictive distribution conditioning on context data;
\item [(3)] the RANP model \citep{qin2019_ANPRNN}, which is an extension of the ANP models, which incorporates the temporal dependency of data using an LSTM model. Similar to the ANP, the RANP is also adept at few-shot learning tasks;
\item [(4)] the MAML model (Model-Agnostic Meta-Learning), which is a popular meta-learning algorithm \citep{finn2017model_MAML};
It is designed to enable fast adaptation of a model's parameters to new tasks with limited data. To enable a fair comparison, we learn the meta parameters on a small set of various tasks and perform limited $t=1$ adaptation step on the target domain;
\item [(5)] and finally, the conventional transfer learning (TL) \citep{weiss2016survey_transferlearning} method, where we first pre-trained the model on the source domain and then fine-tuned it on the target domain. 
To leverage our proposed FOT, we use the estimated map to pushforward all source samples to match the target distribution, and then use these samples for the training, adaptation, or the fine-tuning process for the above baseline methods. 
\end{itemize}
}

\textbf{Results:} The results are given in Table \ref{table:robot_arm_DA_results}. 
{
It is noted that the vanilla and few-shot training approaches are having large errors, as expected, while MAML and transfer learning have better generalization ability as they have the access to some target samples.
However, we have also observed that, somewhat surprisingly, utilizing pushforward samples from the FOT transport map enhances the performance of LSTM, NP, and RANP to surpass the baseline performance of both meta-learning and transfer learning approaches. 
Additionally, even MAML and TL approaches can benefit from utilizing the mapped samples from the FOT. This is because the pushforward data offer additional samples that adhere to the target distribution, which helps mitigate the distributional gap due to the model misspecification.}

All experiments were implemented with Numpy and PyTorch (matrix computation scaling) using one GTX2080TI GPU and a Linux desktop with 32GB memory.
For all simulations, we set the optimization coefficients as $\rho_k = 800 \times \mathbf{1} \in \mathbb{R}^{N \times 1}$, $\rho_l = 800 \times \mathbf{1} \in \mathbb{R}^{n \times 1}$, $\eta = 0.001$, $\gamma_h = 40$, $\gamma_p = -10$, power $p=3$. The learning rate for updating $\mathbf{\Lambda}$ is $lr_{\Lambda} = 4e-4$, the learning rate for updating $\pi_{lk}$ is $lr_{\pi} = 1e-5$. The maximum iteration step is set as $T_{max} = 1000$. 
{ In the experimental results on the robot-arm datasets, the hyperparameters are set by $\gamma_h=30$, $\gamma_p=-30$, and power $p=3$. The same hyperparameters as in the simulation experiments were employed. }
We found that our algorithm's performance was not sensitive to varying hyperparameters. 
{ Specifically, when hyperparameters are perturbed around these values, the performance of the downstream domain adaptation experiments remains stable.}

\section{Proofs}
\label{section:proof}
In this section, we provide proofs for theoretical results in Section \ref{section:OT on Hilbert} and Section \ref{section:methodology}. We first define some notations regarding the proofs.

\paragraph{Notations.} Fix Borel probability measures $\mu$ on $H_1$ and $\nu$ on $H_2$. We define the cost function (without regularization term) to be $\Phi(T): = W_2(T\# \mu, \nu)$ for $T\in \Bspace(H_1, H_2)$. For the ease of notation, as in the main text we write $n$ for $(n_1, n_2)$, $K$ for $(K_1, K_2)$, $\Bspace$ for $\Bspace(H_1, H_2)$ and $\BKspace$ for its restriction on the space spanned by the first $K_1\times K_2$ basis operators. $\|\cdot\|_{HS}$ and $\|\cdot\|_{op}$ are used to denote the Hilbert-Schmidt norm and operator norm on operators, respectively. It is known that $\|T\|_{op}\leq \|T\|_{HS}$ for all operator $T$.

In this section we often deal with convergence of a sequence with multiple indices. Specifically, we say a tuple $m=(m_1, \dots, m_p)\to \infty$ when $m_1\to \infty, \dots, m_p\to \infty$. By saying that a sequence $A(m_1, m_2,\dots, m_p)$ of index $m=(m_1, \dots, m_p)$ converge to a number $a$ as $m\to \infty$, it is meant that for all $\epsilon > 0$, there exists $M_1, \dots, M_p$ such that for all $m_1 > M_1,\dots, m_p > M_p$, we have
 \begin{equation}
    | A(m_1, \dots, m_p) - a| < \epsilon.
 \end{equation}

We write $(m_1, m_2,\dots, m_p) > (m'_1, m'_2,\dots, m'_p)$ if $m_1 > m'_1, \dots, m_p > m'_p$. 

We say a function $f:\Bspace(H_1, H_2)\to \mathbb{R}$ is \textit{coercive} if \begin{equation}\label{eq:coercive}
    \lim_{\|T \|_{HS} \to \infty} f(T) = \infty,
\end{equation}
and it is \textit{(weakly) lower semi-continuous} if \begin{equation}
     f(T_0) \leq \liminf_{k\to \infty} f(T_k),
\end{equation}
 for all sequences $T_k$ (weakly) converging to $T_0$. Further details on convergence in a strong and weak sense in Hilbert spaces can be found in standard texts on functional analysis, e.g., \citep{yosida1995functional}.
 
Now we are going to prove the results presented in Section 3 of the main text. 
For ease of the readers, we recall all statements before proving them.

\paragraph{Existence and uniqueness}
First, we verify some properties of the objective function $J$.
\begin{customlemma}{\ref{lem:propsofJ}}\label{three}
The following statements hold.
    \begin{enumerate}[itemsep=0mm]
        \item[(i)] $W_2(T \# \mu, \nu)$ is a Lipschitz continuous function of $T \in \Bspace(H_1, H_2)$, which implies that 
        $J:\Bspace \rightarrow \mathbb{R}_+$ is also continuous.
        \item[(ii)] $J$ is a strictly convex function.
        \item[(iii)] There are constants $C_1,C_2 >0$ such that $J(T)\leq C_1 \|T \|_{HS}^2 + C_2\;\; \forall T\in \Bspace$.
        \item[(iv)] $\lim_{\|T\|_{HS} \to \infty} J(T) = \infty$.
    \end{enumerate}
\end{customlemma}

\begin{proof}[Proof of Lemma \ref{lem:propsofJ}]
\begin{enumerate}
    \item We first show that the cost function (without regularization term) $\Phi(T) = W_2(T\# \mu, \nu)$ for $T\in \Bspace(H_1, H_2)$ is Lipschitz continuous. Indeed, consider any $T_1, T_2 \in \Bspace$, by the triangle inequality applied to Wasserstein metric,
\begin{align*}
    W_2(T_1\# \mu, \nu) - W_2(T_2\# \mu, \nu) & \leq W_2(T_1\# \mu, T_2 \# \mu) \\
    & = \left(\inf_{\pi \in \Pi(\mu, \mu)} \int_{H_1\times H_1} \|T_1 f_1 - T_2 f_2 \|_{H_2}^2 d\pi(f_1, f_2)\right)^{1/2}\\
    &\leq \left(\int_{H_1\times H_1} \|T_1 f_1 - T_2 f_2 \|_{H_2}^2 d\pi'(f_1, f_2)\right)^{1/2} \\
    &= \left(\int_{H_1} \|T_1 f_1 - T_2 f_1 \|_{H_2}^2 d\mu(f_1)\right)^{1/2}\\
    &\leq  \left(\int_{H_1} \|T_1 - T_2 \|_{op}^2 \|f_1 \|_{H_1}^2 d\mu(f_1)\right)^{1/2}\\
    &\leq \|T_1 - T_2\|_{HS} \left( \int_{H_1} \|f_1 \|_{H_1}^2 d\mu(f_1)\right)^{1/2}\\
    & = \|T_1 - T_2\|_{HS}  (E_{f\sim \mu} \|f\|_{H_1}^2)^{1/2},
\end{align*}
where $\pi'$ is the identity coupling. Hence, both $\Phi^2(T)$ and $\eta \|T \|_{HS}^2$ are continuous, which entails continuity of $J$ as well. 
    \item If we can prove that $\Phi^2(T)$ is convex with respect to $T$, then the conclusion is immediate from the strict convexity of $\eta \|T \|_{HS}^2$. We first observe that $W_2^2(\cdot, \nu)$ is convex, as for any measure $\nu_1, \nu_2$ on $H_2$ and $\lambda \in [0,1]$, if $\gamma_1$ is the optimal coupling of $(\nu_1, \nu)$ and $\gamma_2$ is the optimal coupling of $(\nu_2, \nu)$, then $\lambda \gamma_1 + (1-\lambda) \gamma_2$ is a valid coupling of $(\lambda \nu_1 + (1-\lambda)\nu_2, \nu)$, which yields
    \begin{align*}
        W_2^2(\lambda \nu_1 + (1-\lambda)\nu_2, \nu) & \leq \int_{H_1\times H_2} \|f - g \|_{H_2}^2 d(\lambda \gamma_1 + (1-\lambda) \gamma_2)(f, g)\\
        & = \lambda W_2^2(\nu_1, \nu) + (1-\lambda) W_2^2(\nu_2, \nu).
    \end{align*}
    Now the convexity of $\Phi^2(T)$ follows as for any $T_1, T_2 \in \Bspace, \lambda \in [0,1]$,
    \begin{align*}
        W_2^2(((1-\lambda)T_1 + \lambda T_2)\# \mu, \nu) & = W_2^2((1-\lambda) (T_1\# \mu) + \lambda (T_2\# \mu), \nu) \\
        &\leq (1-\lambda)W_2^2(T_1\# \mu, \nu) + \lambda W_2^2(T_2\# \mu, \nu).
    \end{align*}
    \item This can be proved by an application of Cauchy-Schwarz inequality and the fact that the operator norm is bounded above by the Hilbert-Schmidt norm. Let $\pi$ be any coupling of $\mu$ and $\nu$,
    \begin{align*}
        J(T) & = W_2^2(T \# \mu, \nu) + \eta\|T\|_{HS}^2\\
        & \leq\int_{H_1\times H_2} \|Tf_1 - f_2 \|_{H_2}^2 d\pi(f_1, f_2)  + \eta\|T\|_{HS}^2\\
        & \leq 2\int_{H_1 \times H_2} (\|Tf_1 \|_{H_2}^2 + \|f_2 \|_{H_2}^2) d\pi(f_1, f_2) + \eta\|T\|_{HS}^2\\
        & \leq 2\left(\|T \|_{HS}^2 \int_{H_1} \|f_1\|_{H_1}^2 d\mu(f_1) + \int_{H_2} \|f_2\|_{H_2}^2 d\mu(f_2)\right) + \eta\|T\|_{HS}^2\\
        & = C_1 \|T \|_{HS}^2 + C_2,
    \end{align*}
    for all $T \in B$, where $C_1 = 2 E_{f_1\sim \mu} \|f_1 \|_{H_1}^2 d\mu(f) + \eta, C_2 = 2 E_{f_2\sim \nu} \|f_2 \|_{H_2}^2 d\nu(f)$.
    \item This follows from the fact that $\Phi^2(T) \geq 0$ for all $T$ and $\eta \|T \|^2$ is coercive as in equation~\eqref{eq:coercive}.
\end{enumerate}
\end{proof}

We are ready to establish existence and uniqueness of the minimizer of $J$. The technique being used is well-known in the theory of calculus of variations (e.g., cf. Theorem 5.25. in \citep{Demengel_2012}).

\begin{customthm}{\ref{thm:existuniq}}
    There exists a unique minimizer $T_0$ for the problem \eqref{eq:shrinkage}.
\end{customthm}

\begin{proof}[Proof of Theorem \ref{thm:existuniq}] 
    As $J(T) \geq 0$ and is finite for all $T$, there exist $L_0 = \inf_{T \in \Bspace} J(T) \in [0, \infty)$. Consider any sequence $(T_k)_{k=1}^{\infty}$ such that $J(T_k) \to L_0$. We see that this sequence is bounded, as otherwise, there exists a subsequence $(T_{k_h})_{h=1}^{\infty}$ such that $\|T_{k_{h}} \|_{HS} \to \infty$. But this means $L_0 = \lim J(T_{k_h}) = \infty$ (due to the coercivity), which is a contradiction. Now, because $(T_k)$ is bounded, by Banach-Alaoglu theorem, there exists a subsequence $(T_{k_p})_{p=1}^{\infty}$ converges weakly to some $T_0$. 
    
    Next, we will prove that $J$ is weakly lower semi-continuous. Indeed, we can readily verify that $J$ is (weakly) lower semi-continuous if and only if the epigraph $\{(T, y): y \geq J(T) \}$ is (weakly) closed. Because $J$ is convex and continuous, we have the epigraph is convex and closed. Recall a theorem of Mazur (page 292 of  ~\cite{royden1988real_analysis}), which states that a convex, closed subset of a Banach space is weakly closed. This result implies that the epigraph $\{(T, y): y \geq J(T) \}$ is also weakly closed. Hence, $J$ is weakly lower semi-continuous. Thus,
    \begin{equation}
        J(T_0) \leq \liminf_{p \to \infty}J(T_{k_p}) = L_0.
    \end{equation}
    Therefore the infimum of $J$ is attained at some $T_0$. The uniqueness of $T_0$ follows from the strict convexity of $J$.
\end{proof}

\paragraph{Approximation analysis}
Next, we proceed to analyze the convergence of the minimizers of finite dimensional approximations to the original problem \eqref{eq:shrinkage}. The proof is valid thanks to the presence of the regularization term $\eta \|T\|_{HS}^2$.

\begin{customlemma}{\ref{lem:convergeTK}}
For each $K=(K_1,K_2)$, there exists a unique minimizer $T_K$ of $J$ over $\BKspace$. Moreover, $T_K \to T_0$ in $\|\cdot\|_{HS}$ as $K_1, K_2\to \infty$.
\end{customlemma}

\begin{proof}[Proof of Lemma \ref{lem:convergeTK}]
    Similar to the proof above, for every $K = (K_1, K_2)$ there exists uniquely a minimizer $T_K$ for $J$ on $\BKspace$ as $\BKspace$ is closed and convex. Denote $T_{0, K}$ the projection of $T_0$ to $\BKspace$. As $K\to \infty$, we have $T_{0, K} \to T_0$, which yields  $J(T_{0, K}) \to J(T_0)$. From the definition of minimizers, we have
    \begin{equation}
        J(T_{0, K}) \geq J(T_K) \geq J(T_0), \quad \forall \, K.
    \end{equation}
    Now let $K\to \infty$, we have $\lim_{K\to \infty} J(T_K) = J(T_0)$ thanks to the Sandwich rule. Since $J$ is convex, 
    \begin{equation}
        J(T_0) + J(T_K) \geq 2J\left(\dfrac{1}{2} (T_0 + T_K)\right),
    \end{equation}
    passing this through the limit, we also have 
    \begin{equation}
        \lim_{K\to \infty} J\left(\dfrac{1}{2} (T_0 + T_K)\right) = J(T_0).
    \end{equation}
    Now using the parallelogram rule,
    \begin{align*}
        \dfrac{\eta}{2} \|T_K - T_0\|_{HS}^2 & = \eta \left(\|T_K\|_{HS}^2 + \|T_0\|_{HS}^2 - 2 \left\|\dfrac{1}{2}(T_0 + T_K) \right\|_{HS}^2 \right)\\
        & = \left(J(T_K) + J(T_0) - 2J\left(\dfrac{1}{2} (T_0 + T_K)\right)\right)  \\
        & \quad - \left(\Phi^2(T_K) + \Phi^2(T_0) - 2\Phi^2\left(\dfrac{1}{2} (T_0 + T_K)\right)\right)\\
        &\leq \left(J(T_K) + J(T_0) - 2J\left(\dfrac{1}{2} (T_0 + T_K)\right)\right),
    \end{align*}
    as $\Phi^2$ is convex. Let $K\to \infty$, we have the last expression goes to 0. Hence, $\|T_K - T_0 \|_{HS} \to 0$. 
\end{proof}

What is remarkable in the proof above is that it works for any sequence $(T_m)_{m=1}^{\infty}$: whenever we have $J(T_m) \to J(T_0)$ then we must have $T_m \to T_0$. 

\paragraph{Uniform convergence and consistency analysis}
Now we turn our discussion to the convergence of empirical minimizers. Using the technique above, there exists uniquely minimizer $\hat{T}_{K, n}$ for $\hat{J}_{n}$ over $\BKspace$. We want to prove that $\hat{T}_{K, n}\xrightarrow{P} T_K$ uniformly in $K$ in a suitable sense and then combine with the result above to have the convergence of $\hat{T}_{K, n}$ to $T_0$. 
A standard technique in the analysis of M-estimator is to establish uniform convergence of $\hat{J}_n$ to $J$ in the space of $T$ \citep{keener2010theoretical}. Note that the spaces $\Bspace$ and all $\BKspace$'s are not bounded, so care must be taken to show that $(\hat{T}_{K, n})_{K, n}$ will eventually reside in a bounded subset and then uniform convergence is attained in that subset. The following auxiliary result presents that idea.


\begin{customlemma}{\ref{lem-unifconv}}
\begin{enumerate}
\item For any fixed $C_0>0$, 
\begin{equation}
    \sup_{\|T\|_{HS} \leq C_0} |\hat{J}_n(T) - J(T)| \xrightarrow{P} 0 \quad (n\to \infty).
\end{equation}
\item Let $\hat{T}_{K, n}$ be the unique minimizer of $\hat{J}_n$ over $\BKspace$. There exists a constant $D$ such that $P(\sup_{K} \|\hat{T}_{K, n} \|_{HS} < D) \to 1$ as $n\to \infty$.    
\end{enumerate}
\end{customlemma}
\begin{proof}
    \begin{enumerate}
        \item The proof proceeds in a few small steps.
        
        \emph{Step 1.}  By triangle inequality of Wasserstein distances,
    \begin{align}\label{eqn-boundphi}
        |W_2(T\#\mu, \nu) - W_2(T\#\hat{\mu}_{n_1}, \hat{\nu}_{n_2})| & \leq W_2(T\# \hat{\mu}_{n_1}, T\#\mu) +  W_2( \hat{\nu}_{n_2}, \nu)\nonumber\\
        & \leq \|T\|_{op}  W_2( \hat{\mu}_{n_1}, \mu) + W_2( \hat{\nu}_{n_2}, \nu)\nonumber\\
        & \leq \|T\|_{HS}  W_2( \hat{\mu}_{n_1}, \mu) + W_2( \hat{\nu}_{n_2}, \nu).
    \end{align}
    Therefore,
    \begin{equation}
    \label{eqn-uniboundphi}
        \sup_{\|T\|_{HS} \leq C_0} | \hat{\Phi}_n(T) - \Phi(T) | \leq C_0 W_2( \hat{\mu}_{n_1}, \mu) + W_2( \hat{\nu}_{n_2}, \nu)
    \end{equation}
    By Proposition 2.2.6. of \cite{panaretos2020invitation} and with our assumption of bounded second moments of $\mu$ and $\nu$, we have $W_2( \hat{\mu}_{n_1}, \mu)$ and $W_2( \hat{\nu}_{n_2}, \nu)$ converge almost surely to 0 as $n\to \infty$.
As almost surely convergence implies convergence in probability, we have
    \begin{equation}
        \sup_{\|T\|_{HS} \leq C_0} | \hat{\Phi}_n(T) - \Phi(T) | \xrightarrow{P} 0,
    \end{equation}
    which means for all $\epsilon > 0$, 
    \begin{equation}
        P\left(\sup_{\|T\|_{HS} \leq C_0} | \hat{\Phi}_n(T) - \Phi(T)| < \epsilon \right) \to 1,
    \end{equation}
    \emph{Step 2.} Combining $\sup_{\|T\|_{HS} \leq C_0} | \hat{\Phi}_n(T) - \Phi(T)| < \epsilon$ with the fact that $\Phi^2(T) \leq C_1\|T\|_{HS} +C_2$ implies that for all $T$ such that $\|T\|_{HS} \leq C_0$, we have $\Phi^2(T) \leq C_1C_0 +C_2 =: C$
    \begin{align*}
        |\hat{J}_n(T) - J(T)| & = |\hat{\Phi}^2_n(T) - \Phi^2(T)| \\
        & = | \hat{\Phi}_n(T) - \Phi(T)| | \hat{\Phi}_n(T) + \Phi(T)| \\
        & \leq \epsilon (2\sqrt{C} + \epsilon).
    \end{align*}
    Hence
    \begin{equation}
        P\left(\sup_{\|T\|_{HS} \leq C_0} | \hat{J}_n(T) - J(T)| < \epsilon (2\sqrt{C} + \epsilon) \right) \geq P\left(\sup_{\|T\|_{HS} \leq C_0} | \hat{\Phi}_n(T) - \Phi(T)| < \epsilon \right) \to 1.
    \end{equation}
    Noticing that for all $\delta > 0$, there exists an $\epsilon > 0$ such that $\epsilon (2\sqrt{C} + \epsilon) = \delta$, we arrive at the convergence in probability to 0 of $\sup_{\|T\|_{HS} \leq C_0} | \hat{J}_n(T) - J(T)|$.
    
    \item We also organize the proof in a few steps. 
    
    \emph{Step 1.} Denote $\hat{\Phi}_n(T) = W_2(T\# \hat{\mu}_{n_1}, \hat{\nu}_{n_2})$. We first show that for any fixed $C_0$, 
\begin{equation}
    \sup_{\|T\|_{HS} \geq C_0} \dfrac{|\hat{\Phi}_n(T) - \Phi(T)|}{\|T\|_{HS}} \xrightarrow{P} 0\quad (n\to \infty).
\end{equation}
    
    Indeed, from~\eqref{eqn-boundphi},
    \begin{equation}
        \sup_{\|T\|_{HS} \geq C_0} \dfrac{| \hat{\Phi}_n(T) - \Phi(T) |}{\|T\|_{HS}} \leq W_2( \hat{\mu}_{n_1}, \mu) + \dfrac{W_2( \hat{\nu}_{n_2}, \nu)}{C_0}.
    \end{equation}
    As above, we have $W_2( \hat{\mu}_{n_1}, \mu)$ and $W_2( \hat{\nu}_{n_2}, \nu)$ converge to 0 almost surely as $n\to \infty$.
    Hence, $\sup_{\|T\|_{HS} \geq C_0} \dfrac{| \hat{\Phi}_n(T) - \Phi(T) |}{\|T\|_{HS}} \to 0$ almost surely, and therefore in probability.
    
    \emph{Step 2.} For any fixed $C_0$ and $\delta$,
    \begin{equation}
        P\left(\sup_{\|T\|_{HS} \geq C_0}  \dfrac{|\hat{\Phi}_n(T) - \Phi(T)|}{\|T\|_{HS}} <\delta \right) \to 1 \quad (n\to \infty).
    \end{equation}
    The event $\sup_{\|T\|_{HS} \geq C_0}  \dfrac{|\hat{\Phi}_n(T) - \Phi(T)|}{\|T\|} <\delta$ implies that for all $T$ such that $\|T\|_{HS} \geq C_0$, from Lemma~\ref{lem:propsofJ} we have
    \begin{equation*}
        \hat{J}_n(T) \leq (\Phi(T) + \delta \|T\|_{HS})^2 + \eta \|T\|_{HS}^2 \leq (\sqrt{C_1 \|T\|_{HS}^2 + C_2} + \delta \|T\|_{HS})^2 + \eta \|T\|_{HS}^2.
    \end{equation*}
    
    Now for each $K$, we can choose a $\tilde{T}_K\in \BKspace$ such that $\| \tilde{T}_K\|_{HS} = C_0$. Thus,
    \begin{align*}
        \inf_{T\in \BKspace} \hat{J}_n(T) & \leq \hat{J}_n(\tilde{T}_K) \leq  (\sqrt{C_1 \|\tilde{T}_K\|_{HS}^2 + C_2} + \delta \|\tilde{T}_K\|_{HS})^2 + \eta \|\tilde{T}_K\|_{HS}^2 \\
        & = (\sqrt{C_1 C_0^2 + C_2} + \delta C_0)^2 + \eta C_0^2 =: C,
    \end{align*}
    which is a constant. 
    
    On the other hand, choose $D = \sqrt{C / \eta}$, we have for all $T$ such that $\|T\|_{HS} > D$
    \begin{equation}
         \hat{J}_n(T) \geq \eta \|T \|_{HS}^2 > C,
    \end{equation}
    which means $\inf_{T\in \BKspace: \|T \|_{HS} > D} \hat{\Phi}_n(T) > C$ for all $K$. 
    
    Combining two facts above, we have $\|\hat{T}_{K, n}\|_{HS} \leq D$ for all $K$. 
    
    \emph{Step 3.} It follows from the previous step that 
    \begin{equation}
        P\left(\sup_{K} \|\hat{T}_{K, n}\|_{HS} \leq D\right) \geq P\left(\sup_{\|T\|_{HS} \geq C_0}  \dfrac{|\hat{\Phi}_n(T) - \Phi(T)|}{\|T\|_{HS}} <\delta \right),
    \end{equation}
    which means this probability also goes to 1 as $n\to \infty$.
    \end{enumerate}
\end{proof}

We are ready to tackle the consistency of our estimation procedure.

\begin{customthm}{\ref{thm:convergence}}
There exists a unique minimizer $\hat{T}_{K, n}$ of $\hat{J}_{n}$ over $\BKspace$ for all $n$ and $K$. Moreover, $\hat{T}_{K, n} \xrightarrow{P} T_0$ in $\|\cdot \|_{HS}$ as $K_1, K_2, n_1, n_2 \to \infty$.
\end{customthm}

\begin{proof}[Proof of Theorem \ref{thm:convergence}]
  The proof proceeds in several smaller steps.
  
  \emph{Step 1.} Take any $\epsilon > 0$. As $T_K \to T_0$ when $K\to \infty$, there exist $\kappa = (\kappa_1, \kappa_2)$ such  that $\|T_K - T_0\|_{HS} \leq \epsilon$ for all $K_1 > \kappa_1, K_2> \kappa_2$. Let 
    \begin{equation}
        L_{\epsilon} = \inf_{T \in \Bspace\setminus B(T_0, \epsilon)} J(T),
    \end{equation}
    where $B(T, \epsilon)$ is the Hilbert-Schmidt open ball centered at $T$ having radius $\epsilon$. It can be seen that $L_{\epsilon} > J(T_0)$, as otherwise, there exists a sequence $(T_p)_p \not \in B(T, \epsilon)$ such that $J(T_p) \to J(T_0)$, which implies $T_p \to T_0$, a contradiction.
    
    \emph{Step 2.} Let $\delta = L_{\epsilon} - J(T_0) > 0$. By Lemma~\ref{lem:convergeTK}, we can choose $\kappa$ large enough so that we also have $|J(T_K) - J(T_0)| < \delta/2\,\,\forall \, K_1>\kappa_1, K_2>\kappa_2$. Let 
    \[L_{K, \epsilon} = \inf_{\BKspace \setminus B(T_K, 2\epsilon)} J(T).\] 
    As $B(T_0, \epsilon) \subset B(T_K, 2\epsilon)$ and $\BKspace \subset \Bspace$, we have
    \begin{equation}
        L_{K, \epsilon} = \inf_{\BKspace \setminus B(T_K, 2\epsilon)} J(T) \geq  \inf_{T \in \Bspace\setminus B(T_0, \epsilon)} J(T) = L_{\epsilon}.
    \end{equation}
    Therefore,
    \begin{equation}
        L_{K, \epsilon} - J(T_K) \geq L_{\epsilon} - J(T_0) - \delta/2 = \delta/2.
    \end{equation}
    for all $K > \kappa$. 
    
    \emph{Step 3.} Now, if we have
    \begin{equation}
        \sup_{\|T \|\leq D} |\hat{J}_n(T) - J(T)| \leq \delta /4 ,\quad \sup_{K} |\hat{T}_{K, n}| \leq D,
    \end{equation}
    where $D$ is a constant as in  Lemma~\ref{lem-unifconv}, then 
    \begin{equation}
        \hat{J}_n(T_K) \leq J(T_K) + \delta/4,
    \end{equation}
    and
    \begin{equation}
        \hat{J}_n(T) \geq J(T) - \delta/4 \geq J(T_K) + \delta/4,
    \end{equation}
    for all $\| T\|_{HS}\leq D$ \emph{and} $T \in \BKspace \setminus B(T_K, 2\epsilon)$, where the last inequality is due to Step 2. 
    
    Combining with $|\hat{T}_{K, n}| \leq D$, we have $\hat{T}_{K, n}$ must lie inside $B(T_K, 2\epsilon)\cap \BKspace$ because it is the minimizer of $\hat{J}_n$ over $\BKspace$. Hence $\|\hat{T}_{K, n} - T_K \|_{HS} \leq 2 \epsilon$, which deduces that $\|\hat{T}_{K, n} - T_0 \|_{HS} \leq \|\hat{T}_{k,n} - T_K\|_{HS} + \|T_k - T_0\|_{HS} \leq 2\epsilon + \epsilon = 3 \epsilon$. 
    
    \emph{Step 4.} Continuing from the previous step, for all $\kappa$ large enough, we have the following inclusive relation of events
    \begin{equation}
        \{\sup_{\|T \|\leq D} |\hat{J}_n(T) - J(T)| \leq \delta /4\} \cap \{ \sup_{K} |\hat{T}_{K, n}| \leq D\} \subset \{ \sup_{K > \kappa} \|\hat{T}_{K, n} - T_0 \|_{HS} \leq 3 \epsilon\} 
    \end{equation}
    Using the inequality that for any event $A, B$, $P(A \cap B) \geq P(A) + P(B) - 1$, we obtain
    \begin{equation}
        P( \sup_{K > \kappa} \|\hat{T}_{K, n} - T_K \|_{HS} \leq 3 \epsilon) \geq P(\sup_{\|T \|_{HS}\leq D} |\hat{J}_n(T) - J(T)| \leq \delta /4) + P(\sup_{K} |\hat{T}_{K, n}| \leq D) - 1,
    \end{equation}
    which goes to 1 as $n\to \infty$ due to Lemma~\ref{lem-unifconv}. Because this is true for all $\epsilon>0$, we have
    \begin{equation}
        \hat{T}_{K, n} \xrightarrow{P} T_0,
    \end{equation}
    as $K, n\to \infty$.
\end{proof}

\paragraph{Consistency when the functional data are observed only at design points}
Finally, we are ready to prove Theorem~\ref{thm:consistency-nkd}, which is re-stated herein.

\begin{customthm}{\ref{thm:consistency-nkd}}
(i)        For every $n_1, n_2, K_1, K_2$ and sequences of design points in source and target domains, the cost function
\begin{eqnarray}\label{eq:optimize_connection_new_d2}
    \hat{J}_{n, K, d}(\mathbf{\Lambda}) = \min_{\pi \in \hat{\Pi}} \sum_{l,k=1}^{n_1,n_2} \pi_{lk} D_{lkd}(\mathbf{\Lambda}) + \eta \|\mathbf{\Lambda}\|_F^2,
\end{eqnarray}
where 
\begin{equation*}
     D_{lkd}(\mathbf{\Lambda}) = \|\mathbf{\Lambda} a_{ld} - b_{kd}\|_2^2,
\end{equation*}
in which $a_{ld} =  (\langle f_{1,l}, U_i  \rangle_{d})_{i=1}^{K_1}$ and  $b_{kd} =  (\langle f_{2,k}, V_j  \rangle_{d})_{i=1}^{K_2}\, \forall l, k$, has unique minimizer $\mathbf{\Lambda}_{n, K, d} \in \mathbb{R}^{K_2\times K_1}$ that corresponds to operator $T_{n, K, d}$. 

(ii) Suppose that for any natural index pair $(i, j)$, there holds
\begin{equation}
    \langle f, U_i \rangle_d \to \langle f, U_i \rangle_{H_1},  \langle g, V_j \rangle_d \to \langle g, V_j \rangle_{H_2}, 
\end{equation}
almost surely as $d\to \infty$, where $f \sim \mu$ and $g \sim \nu$. Then for any sequence $n_1, n_2, K_1, K_2 \to \infty$ and $d\to \infty$ with a rate depends on $n_1, n_2, K_1, K_2$, we have $T_{n, K, d} \xrightarrow{P} T_0$ in $\|\cdot \|_{HS}$. Here $T_0$ denotes the minimizer of the population version of FOT given in Eq.~\eqref{eq:shrinkage}.
\end{customthm}

\begin{proof}
For each $n_1, n_2, K_1, K_2$ and sequences of design points, the existence and uniqueness of $\mathbf{\Lambda}_{n, K, d}$ follows from Theorem \ref{thm:existuniq}. Thus, part (i) is immediate. To establish part (ii), we rewrite the objective function \eqref{eq:optimize_connection_new_d2} as
\begin{equation}\label{eq:optimize_new_d_Wasserstein}
    W_2^2(T_{\#} \mu_{n, K, d}, \nu_{n, K, d}) + \eta \|T\|_{HS}^2,
\end{equation}
where
\begin{equation*}
    \mu_{n, K, d} = \dfrac{1}{n_1}\sum_{l=1}^{n_1} \delta_{f_{l, K, d}}, \quad f_{l, K, d} = \sum_{i=1}^{K_1} \langle f_{l}, U_i  \rangle_d U_i, \quad (f_{l})_{l=1}^{n_1} \overset{iid}{\sim}\mu,
\end{equation*}
and 
\begin{equation*}
    \nu_{n, K, d} = \dfrac{1}{n_2}\sum_{k=1}^{n_2} \delta_{g_{k, K, d}}, \quad g_{k, K, d} = \sum_{i=1}^{K_2} \langle g_k, V_j  \rangle_d V_j, \quad (g_{k})_{k=1}^{n_2} \overset{iid}{\sim}\nu. 
\end{equation*}
As a consequence of Lemma \ref{lem-unifconv} and Theorem \ref{thm:convergence}, the conclusion of the theorem can be achieved if we can show that
\begin{equation}
    W_2(\mu_{n, K, d}, \mu) \xrightarrow{P} 0, \quad W_2(\nu_{n, K, d}, \nu) \xrightarrow{P} 0.
\end{equation}
It suffices to establish the convergence for $\mu$, as the work for $\nu$ can be done in the same way. Note that
\begin{equation}
    W_2(\mu_n, \mu) \xrightarrow{a.s.} 0, 
\end{equation}
as $n_1 \to \infty$, where $\mu_n = \frac{1}{n_1} \sum_{l=1}^{n_1}\delta_{f_{l}}$, so that we only need to show 
\begin{equation}
\label{eq:next-to-last}
    W_2(\mu_{n, K, d}, \mu_n) \xrightarrow{P} 0\quad (\text{as } n\to \infty).
\end{equation}

Consider the coupling that places mass $\dfrac{1}{n_1}$ on $(f_l, f_{l. K, d})$ for all $l=1, \dots, d_1$, then we have
\begin{align}\label{eq:bound-W2-ndk}
     W_2^2(\mu_{n, K, d}, \mu_n)& \leq  \dfrac{1}{n_1} \sum_{l=1}^{n_1} \|f_l - f_{l, K, d}\|_{H_1}^2 \nonumber \\
     & = \dfrac{1}{n_1} \sum_{l=1}^{n_1} \left(\sum_{i=1}^{K_1} (\langle f_{l}, U_i \rangle_d - \langle f_l, U_i \rangle_{H_1})^2 + \sum_{i=K_1+1}^{\infty} \langle f_l, U_i \rangle_{H_1}^2\right).
\end{align}
As $n_1, K_1\to \infty$, by an application of Fubini's theorem, we have
\begin{equation}
    \lim_{n_1, K_1\to \infty} \sum_{l=1}^{n_1}\sum_{i=1}^{K_1} \dfrac{1}{n_1}\langle f_l, U_i \rangle_{H_1}^2 = \lim_{n_1\to \infty} \sum_{l=1}^{n_1}\dfrac{1}{n_1}\|f_{l}\|_{H_1}^2 = E_{f\sim \mu} \|f\|_{H_1}^2,
\end{equation}
almost surely. Hence, 
\begin{equation}
    \lim_{n_1, K_1\to \infty} \sum_{l=1}^{n_1}\sum_{i=K_1+1}^{\infty} \dfrac{1}{n_1}\langle f_l, U_i \rangle_{H_1}^2 = 0,
\end{equation}
almost surely. Now consider the first sum of the right-hand side of \eqref{eq:bound-W2-ndk}, for each $l=1, \dots, n_1$ and $i=1,\dots, K_1$, we have from the assumption of the theorem that
\begin{equation}
    (\langle f_{l}, U_i \rangle_d - \langle f_l, U_i \rangle_{H_1})^2\to 0,
\end{equation}
almost surely for $f_l \sim \mu$ as $d\to \infty$, and almost surely convergence implies convergence in probability. So, for every $\delta, \epsilon >0$, we can choose $D = D(\delta, \epsilon, K_1, n_1)$ such that
\begin{equation}
    P\left(\sum_{i=1}^{K_1}(\langle f_{l}, U_i \rangle_d - \langle f_l, U_i \rangle_{H_1})^2 > \epsilon\right) \leq \dfrac{\delta}{n_1}, \quad \forall d > D.
\end{equation}
Hence,
\begin{align*}
    P\left(\dfrac{1}{n_1} \sum_{l=1}^{n_1}\sum_{i=1}^{K_1}(\langle f_{l}, U_i \rangle_d - \langle f_l, U_i \rangle_{H_1})^2 > \epsilon\right) & \leq P\left(\bigcup_{l=1}^{n_1} \left\{\sum_{i=1}^{K_1}(\langle f_{l}, U_i \rangle_d - \langle f_l, U_i \rangle_{H_1})^2 > \epsilon \right\}\right)\\
    &\leq \sum_{l=1}^{n_1}P\left(\sum_{i=1}^{K_1}(\langle f_{l}, U_i \rangle_d - \langle f_l, U_i \rangle_{H_1})^2 > \epsilon\right)\\
    & \leq \delta,
\end{align*}
     for all $d > D(\delta, \epsilon, K_1, n_1)$. It means that $W_2(\mu_{n, K, d}, \mu_n)$ converges to 0 in probability as $n_1, K_1\to \infty$ and the numbers of design points grow to infinity with a rate depending on $n_1$ and $K$. Thus, the RHS of \eqref{eq:bound-W2-ndk} vanishes in probability, so Eq.~\eqref{eq:next-to-last} is established, and the rest of the proof follows similarly to the proofs of Lemma \ref{lem-unifconv} and Theorem \ref{thm:convergence}.
\end{proof}

\section{Discussions and Future Work}
\label{section:conclusion}

We proposed a formulation of optimal transport for probability distributions on domains of functions, where the stochastic map between functional domains can be represented by an infinite dimensional Hilbert-Schmidt operator mapping a Hilbert space of functions to another. 
We proposed a learning method for transport maps based on subspace approximations of Hilbert-Schmidt operators, and implemented an efficient algorithm that involves joint optimization of such operators and a suitable space of stochastic couplings.  Theoretical guarantees on the existence, uniqueness, and consistency of our estimator were achieved. Through simulation studies, we validated our theory and demonstrated the effectiveness of our method of approximation, and that of the learning algorithm, by taking into account the functional nature of the data domains. The effectiveness of our approach was further demonstrated in a couple of real-world domain adaptation applications involving complex and realistic robot arm movements.

\begin{figure}[ht!] 
 \resizebox{\textwidth}{!}{
\begin{tabular}{ |p{4.6cm}|p{4.4cm}|p{1.6cm}|p{4.2cm} | p{4.4cm} | }
 \hline
 Work  &  measure  & functional  &  transport map \\ 
 \hline
 \hline
\citep{genevay2016stoch_LSOT}    & discrete, semi-discrete &  no    &  n/a \\ 
 \hline
\citep{seguy2017LSOT}     & continuous &  no    &  neural network \\ 
 \hline
 \citep{alvarez2019towards_joint, grave2019unsupervised}, \citep{meng2019large_map_est}    & discrete &  no   &  rigid transformation \\ 
\hline
 \citep{perrot2016mapping_est_discre_OT}    & discrete &  no   &  linear \& kernel  \\ 
 \hline 
\citep{xie2019scalable_LSOT} & empirical &  no   &  GAN  \\ 
 \hline 
\citep{makkuva2019ICNN}    & empirical &  no   &  ICNN  \\ 
 \hline 
\citep{mallasto2017learning_wGP}    & single Gaussian process &  \textbf{yes}   &  n/a$^*$ \\ 
\hline 
\textbf{This work}    & measures on Hilbert spaces &  \textbf{yes}   &  Hilbert-Schmidt operator \\ 
\hline 
\end{tabular}
}
\captionof{table}{Related works on optimal transport map estimation.}
\label{tab:summary}
\end{figure}

{ \subsection{Related work }}

{
Optimal transport-based applications have made significant strides in the field of machine learning  ~\citep{arjovsky2017wgan,ho2017multilevel,li2019learning_IOT,chen2019imp_seq2seq_OT,alvarez2019towards_joint,fan2020scalable_wbc_icnn,alvarez2020geometric_otdd,fan2021variational_w_gradflow,korotin2021neural_ot_benchmark,fatras2021unbalanced_minibatch,nguyen2022hierarchical_sliced,bunne2022proximal_ot_dynamics}.
}
{
A suite of OT based approaches involve solving the Kantorovich problem through linear programming or the Sinkhorn algorithm~\citep{Cuturi-Sinkhorn-13,courty2016optimal_OTDA,pooladian2021entropic_map}. Typically, they look for an optimal coupling between empirical measures while the objectives can be extended depending on the problems at hand, such as the ones that gave rise to the Gromov-Wasserstein distance~\citep{memoli2011gromov}, Sliced Wasserstein~\citep{nguyen2022hierarchical_sliced}, partial minibatch OT~\citep{fatras2021unbalanced_minibatch}, outlier robustness criterion~\citep{mukherjee2020outlier_OT}, Orlicz-Wasserstein distance~\citep{Guha-Growing-19}, and so on.
}
{
The results obtained can be extended to vector spaces with high dimensions, but scaling the algorithms for larger sample sizes is typically challenging. 
Also, managing continuous measures within vector spaces is complicated and the most significant difficulty lies in generalizing to out-of-sample data that has not been seen before.}

{
From an applied viewpoint, the problem of (Monge) map estimation has always been of interest, because the Monge map represents the simplest form of the optimal coupling distribution \emph{if} such a map exists. The question of existence and almost sure uniqueness of the Monge map was settled by \cite{brenier1987} in what is now known as the celebrated Brenier theorem, a result which has been extended and generalized by many other authors (see Chapter 6 of \cite{ambrosio2005gradient} and Chapters 9--11 of \cite{villani2008optimal}). Brenier-type theorems helped to rejuvenate the development of optimal transport, in theory and subsequently in practice. In practice, early attempts at \emph{learning} a transport map from data include approaches that represent transport maps as linear transformation or kernel based function classes ~\citep{perrot2016mapping_est_discre_OT}.
A more ambitious approach was made in the work of ~\cite{seguy2017LSOT}, who proposed to find the optimal transport map parameterized by a rich neural network model, following the regularized Kantorovich dual formulation. Exploiting the characterization of the Monge map as the gradient of a convex potential, various authors have proposed to exploit computational convexity~\citep{makkuva2019ICNN,korotin2021neural_ot_benchmark,fan2020scalable_wbc_icnn,korotin2022neural,bunne2022supervised} by leveraging input convex neural networks (ICNNs)~\citep{amos2017_ICNN}. A summary of map estimation work is given in Table \ref{tab:summary}. 
While most of the aforementioned literature still centers at the situation where the support of the distributions are subsets of finite dimensional vector spaces, the growing implementation of machine learning algorithms in various real-world applications, such as time series, have motivated the formulation of optimal transport problem in the domains of functions.}

{ As mentioned in the Introduction, most known results and techniques on optimal transport between distributions on function spaces are related to Gaussian processes and Gaussian measures on normed spaces 
~\citep{mallasto2017learning_wGP,masarotto2019procrustes_GP,knott1984optimal,pigoli2014distances}. These results are natural generalization from those of the multivariate Gaussian distributions~\citep{dowson1982frechet,givens1984class_W_prob}. Specifically, the 2-Wasserstein distance between Gaussian processes with certain covariance coincides with the Procrustes distance between the two covariance operators, which can be approximated arbitrarily well via finite-dimensional approximation~\citep{masarotto2019procrustes_GP}. Additional advances on optimal transport for Gaussian measures have been made by other authors, e.g., \citep{takatsu2011wasserstein_geo_Gaus_Meas,agueh2011barycenters,alvarez2016fixed_barycenter}. In particular, for centered Gaussian measures supported by Hilbert spaces, there exists a linear \emph{subspace} of the (source) Hilbert space where the optimal map is explicit and well-defined as a linear operator. Unfortunately, such a linear map is unbounded so it cannot be extended to the whole (source and target) domains. Such theoretical advances notwithstanding, in practice, the Gaussian distribution assumption is clearly too restrictive in many domains (recall our comparison to
GPOT of \cite{mallasto2017learning_wGP} in Section \ref{section:Experiments}). Our work may be viewed as a first step at addressing optimal transport in the domains of functions that go beyond the Gaussian assumption, and with a particular focus on learning the explicit transport map for sampled functional data.}

{ The formulation presented in this paper can be viewed as a regularization approach to the optimal coupling problem with respect to the source and target distributions on Hilbert spaces of functions. The overall optimal coupling is therefore represented jointly by a compact linear operator $T$ which transports functions in the source domain to functions in the target domain, and by a stochastic coupling distribution $\Pi$. The compactness of $T$ and the sparseness of $\Pi$ induced by suitable penalty techniques offer several theoretical advantages (i.e., existence, uniqueness and consistency of the estimates), as well as practical advantages (i.e., efficient computation and interpretability). 

The primary limitation of our approach, as discussed in the Introduction, is in the situation where the deterministic optimal transport map exists but is either an unbounded linear operator or a nonlinear operator. Developing a theory and methods to accommodate unbounded or nonlinear optimal transport map in the infinite dimensional setting will be  challenging. From a practical and data-driven perspective, one needs to balance the cost of estimating a nonlinear map with the possibility that a single nonlinear map may not offer a very accurate pushforward probability measure for the target domain. That is because in infinite dimensional settings, it is more likely than not that such a deterministic map does not exist, unless certain restrictions are in place. 
}
{}
{}

{ \subsection{Future work}}
Within our formulation of linear operator regularized optimal transport in the functional domains, there are a number of interesting questions and approaches that are worthy of further investigation.
\begin{itemize}
    \item It is of interest to characterize the convergence behavior of the FOT estimator with respect to the number of function samples, basis functions, and design points in the infinite-dimensional setting. 
    Very recently, several groups of researchers have started to study rates of convergence of optimal transport map estimators. However, their work focuses on finite-dimensional domain settings (e.g., \citep{hutter2021minimax_ot_map, tudor_manole2021plugin_map_rate, gunsilius_2021_map_rate, deb2021rates_ot_map}). New ideas and more powerful techniques will probably be required to address analogous questions in the functional domains.
    
    \item There is a vast opportunity to expand the scope of real-world applications by bridging functional data analysis techniques with the optimal transport formalism, where both functional data analysis and optimal transport viewpoints play complementary roles toward achieving effective solutions. For example, in domain adaptation tasks in robotics, healthcare, and autonomous driving, data are intrinsically associated with physical processes and functions. One may also be interested in learning a transport map that pushes forward the samples across diverse yet related domains, such as assembling in manufacturing. For these tasks, it would be critical to investigate the choice of basis functions for a specific machine learning problem, which is further related to functional PCA or representational learning. 
    
    
    \item The proposed FOT framework should be useful toward 
    tackling the domain generalization problem. A principled way is to pushforward the predictive function from a source domain towards a target domain by directly working on the predictive function's parameters. This idea can be generalized to prevalent deep learning methods by considering the basis functions as deep feature extractors. Thus, an FOT based approach may provide a more \textit{interpretable} solution for domain generalization with high-dimensional data, such as those that arise in natural language processing and computer vision.

\end{itemize}

\section*{Acknowledgement}

We thank Rayleigh Lei and Vincenzo Loffredo for valuable discussions and advice on an early version of this paper. Jiacheng Zhu is supported in part by Carnegie Mellon University’s College of Engineering Moonshot Award for Autonomous Technologies for Livability and Sustainability (ATLAS) and by CMU’s Mobility21 National University Transportation Center, which is sponsored by the US Department of Transportation. Mengdi Xu and Ding Zhao are partially supported by the NSF CAREER CNS-2047454. Long Nguyen is partially supported by NSF grant DMS-2015361 and a research gift from Wells Fargo.

\vskip 0.2in

\bibliography{main,Aritra}

\end{document}